\documentclass[11pt]{article}

\usepackage{jmlr2eLin}

\usepackage{amsfonts,mathrsfs,amsmath}
\usepackage[page,title,titletoc,header]{appendix}
\usepackage{caption,subcaption}
\usepackage{url}
\usepackage{enumerate,anysize,color}

\newtheorem{thm}{Theorem}
\newtheorem{pro}{Proposition}
\newtheorem{rem}{Remark}
\newtheorem{Exa}{Example}
\newtheorem{as}{Assumption}
\newtheorem{alg}{Algorithm}

\newcommand{\be}{\begin{equation}}
\newcommand{\ee}{\end{equation}}
\newcommand{\bea}{\begin{eqnarray*}}
\newcommand{\eea}{\end{eqnarray*}}
\newcommand{\mR}{\mathbb{R}}
\newcommand{\mN}{\mathbb{N}}
\newcommand{\mE}{\mathbb{E}}

\newcommand{\mcE}{\mathcal{E}}

\newcommand{\TK}{\mathcal{T}_{\rho}}
\newcommand{\TKL}{\mathcal{T}_{\rho,\lambda}}
\newcommand{\TXL}{\mathcal{T}_{{\bf x},\lambda}}
\newcommand{\LK}{\mathcal{L}_{\rho}}
\newcommand{\IK}{\mathcal{S}_{\rho}}
\newcommand{\TX}{\mathcal{T}_{\bf x}}
\newcommand{\SX}{\mathcal{S}_{\bf x}}

\newcommand{\HK}{H}
\newcommand{\HR}{H_{\rho}}
\newcommand{\LR}{L^2(H,\rho_X)}
\newcommand{\LRH}{L_{\hat{\rho}}^2}

\newcommand{\J}{{\bf J}}

\newcommand{\FH}{f_{\mathcal{H}}}

\newcommand{\tr}{\operatorname{tr}}
\newcommand{\la}{\langle}
\newcommand{\ra}{\rangle}
\newcommand{\eref}[1] {(\ref{#1})}

\newcommand{\revision}{\textcolor{black}}

\begin{document}

\title{Optimal Rates for Multi-pass Stochastic Gradient Methods}

\author{\name Junhong Lin
	 \email jhlin5@hotmail.com \\
       \addr
       Laboratory for Computational and Statistical Learning\\
       Istituto Italiano di Tecnologia and Massachusetts Institute of Technology\\
       Bldg. 46-5155, 77 Massachusetts Avenue, Cambridge, MA 02139, USA
              \AND
Lorenzo Rosasco \email lrosasco@mit.edu \\
       \addr
       DIBRIS, Universit\`a di Genova\\
       Via Dodecaneso, 35 --- 16146 Genova, Italy\\
       Laboratory for Computational and Statistical Learning\\
        Istituto Italiano di Tecnologia and Massachusetts Institute of Technology\\
       Bldg. 46-5155, 77 Massachusetts Avenue, Cambridge, MA 02139, USA
}

\maketitle



\begin{abstract}
We analyze the  {learning  properties} of the stochastic gradient method when multiple passes over the data and mini-batches are allowed. We study how  regularization properties are controlled by the step-size,  the  number of passes and the mini-batch size.  In particular, we consider the square loss and show that    for  a universal step-size choice, the number of passes acts as a regularization parameter, and optimal finite sample bounds  can be achieved by early-stopping. Moreover, we show that larger step-sizes are allowed when considering mini-batches.
Our analysis is based on  a unifying approach, encompassing both batch and stochastic gradient methods as special cases.  {As a byproduct, we derive optimal convergence results for batch gradient methods (even in the non-attainable cases).}
\end{abstract}

\section{Introduction}

Modern machine learning applications require computational approaches that are at the same time statistically accurate
and numerically efficient \citep{bousquet2008tradeoffs}.  This has motivated a recent interest in stochastic gradient methods (SGM),  since on the  one hand they  enjoy good practical performances, especially in large scale scenarios, and on the other hand  they are amenable to theoretical studies. In particular, unlike other learning approaches, such as empirical risk minimization or Tikhonov regularization, theoretical results on  SGM naturally integrate statistical and computational aspects.

Most generalization studies on  SGM consider the case where only one pass over the data is allowed and the step-size is appropriately chosen, see \citep{cesa-bianchi2004,nemirovski2009robust,ying2008online,tarres2014online,dieuleveut2014non,orab14} and references therein, possibly considering averaging \citep{poljak1987introduction}. In particular, recent works show how the step-size can be seen to play the role of a regularization parameter whose choice controls the bias and variance properties of the obtained solution \citep{ying2008online,tarres2014online,dieuleveut2014non,lin2016generalization}. These latter works show that  balancing these contributions, it is possible to derive a step-size choice  leading to optimal   learning bounds.  Such a choice typically depends on some unknown properties of the data generating distributions and it can be chosen by cross-validation in practice.

While processing  each data point only once is natural in streaming/online scenarios,  in practice SGM is often used to process large data-sets and  multiple passes over the data are typically considered. In this case, the number of passes over the data, as well as the step-size, need then to be determined. While the role of multiple passes is well understood if the goal is empirical risk minimization \citep[see e.g.,][]{boyd2007stochastic}, its effect with respect to generalization is less clear. A few recent works have recently started to tackle this question. In particular, results in this direction have been derived in \citep{hardt2015train} and \citep{lin2016generalization}. The former work considers a general stochastic optimization setting and studies  stability properties of SGM allowing to derive  convergence results as well as finite sample bounds. The latter work, restricted to supervised learning,  further develops these results to compare the respective roles of step-size and number of passes, and  show how different parameter settings can lead to optimal error bounds. In particular, it shows that there are  two extreme cases: while one between the  step-size or the number of passes is fixed a priori, while the other one acts as a regularization parameter and needs to be chosen adaptively.  The main shortcoming of these latter results is that they are for the  worst case, in the sense that they do not consider the possible effect of benign assumptions on the problem \citep{zhang2005learning,caponnetto2007optimal} that can lead to faster rates for other learning approaches such as Tikhonov regularization.
Further, these results do not consider the possible effect on generalization  of mini-batches, rather than a single point in each gradient step \citep{shalev2011pegasos,dekel2012optimal,sra2012optimization,ng2016machine}. This latter strategy is often considered especially for  parallel implementation of SGM.

The study in this paper fills in these gaps in the case where the loss function is  the least squares loss.
We  consider a variant of SGM for least squares,  where gradients are sampled uniformly at random and mini-batches are allowed. The number of passes, the step-size and the mini-batch size are then parameters to be determined.  Our main results highlight the respective roles of these parameters and show how can they be chosen so that the corresponding solutions achieve optimal learning errors in a variety of settings. In particular, we show for the first time that multi-pass SGM with early stopping and a universal step-size choice can achieve optimal convergence rates, matching those of ridge regression \citep{smale2007learning,caponnetto2007optimal}. Further, our analysis shows how the mini-batch size and the step-size choice are tightly related. Indeed, larger mini-batch sizes allow considering larger step-sizes while keeping the optimal learning bounds. This result gives insights on how to exploit mini-batches for  parallel computations while preserving optimal statistical accuracy. Finally, we note that a recent work \citep{rosasco2015learning} is related to the analysis in the paper.  The generalization properties of a  multi-pass incremental gradient are analyzed in  \citep{rosasco2015learning},  for a  cyclic,  rather than a stochastic,  choice  of the gradients  and with no mini-batches. The analysis in this latter case appears to be harder and results in  \citep{rosasco2015learning} give good learning bounds only in restricted setting and considering iterates rather than the excess risk. Compared to  \citep{rosasco2015learning}  our results show how stochasticity can be exploited to get fast rates and analyze the role of mini-batches. The basic idea of our proof is to approximate the SGM learning sequence in terms of the batch gradient descent sequence, see Subsection \ref{subsec:proSke} for further details. This allows to study batch and stochastic gradient methods simultaneously, and may be also useful for analyzing other learning algorithms.

This paper is an extended version of a prior conference paper \citep{lin2016optimal}. In \citep{lin2016optimal}, we give convergence results with optimal rates for the attainable case (i.e., assuming the existence of at least one minimizer of the expected risk over the hypothesis space) in a fixed step-size setting. In this new version, we give convergence results with optimal rates, for both the attainable and non-attainable cases, and consider
more general step-size choices. The extension from the attainable case to the non-attainable case is non-trivial. As will be seen from the proof, in contrast to the attainable case, a different and refined estimation is needed for the non-attainable case.
Interestingly, as a byproduct of this paper, we also derived optimal rates for the batch gradient descent methods in the non-attainable case. To the best of our knowledge, such a result may be the first kind for batch gradient methods, without requiring any extra unlabeled data as that in \citep{caponnetto2010cross}. Finally, we also add novel convergence results for the iterates showing that they converge to the minimal norm solution of the expected risk with optimal rates.

The rest of this paper is organized as follows. Section \ref{sec:learning} introduces the learning setting and the SGM algorithm. Main results with discussions and proof sketches are presented in Section \ref{sec:main}.
Preliminary lemmas necessary for the proofs will be given in Section \ref{sec:estimates} while detailed proofs will be conducted in Sections \ref{sec:biasSam} to \ref{sec:hnorm}.
Finally, simple numerical simulations are given in Section \ref{sec:numerical} to complement our theoretical results.

\paragraph{Notation} For any $a,b \in \mR$, $a \vee b$ denotes the maximum of $a$ and $b$.
$\mN$ is the set of all positive integers. For any $T \in \mN,$ $[T]$ denotes the
set $\{1,\cdots,T\}.$
For any two positive sequences $\{a_t\}_{t\in [T]}$ and $\{b_t\}_{t\in [T]},$
the notation $a_{t} \lesssim b_{t}$ for all $t\in [T]$ means that there exists a positive constant
$C \geq 0$ such that $C$ is independent of $t$ and that $a_{t}\leq C b_{t}$ for all $t \in [T].$

 \section{Learning with SGM}\label{sec:learning}
We begin by introducing the learning setting we consider, and then
describe the SGM learning algorithm.
Following \citep{rosasco2015learning}, the formulation
we consider is close to the setting of functional regression, and covers the reproducing kernel Hilbert space (RKHS) setting as a special case, see Appendix \ref{app:learning}. In particular, it reduces to standard linear
regression for finite dimensions.
\subsection{Learning Problems}
Let $\HK$ be a separable Hilbert space, with inner product and induced norm denoted by
$\la \cdot, \cdot \ra_{\HK}$ and $\| \cdot \|_{\HK}$, respectively.
Let the input space $X \subseteq \HK$ and the output space $Y \subseteq \mR$.
Let $\rho$ be an unknown probability measure on $Z=X\times Y,$ $\rho_X(\cdot)$ the induced marginal measure on $X$, and  $\rho(\cdot | x)$ the conditional probability measure on $Y$ with respect to $x \in X$ and $\rho$.

Considering the square loss function, the problem under study is the minimization of the {\it risk},
\be\label{expectedRisk}
\inf_{\omega \in \HK} \mcE(\omega), \quad \mcE(\omega) = \int_{X\times Y} ( \la \omega, x \ra_{\HK} - y)^2 d\rho(x,y),
\ee
when the measure  $\rho$ is known only through
 a sample $\mathbf z=\{z_i=(x_i, y_i)\}_{i=1}^m$ of size $m\in\mN$, independently and identically distributed (i.i.d.) according to $\rho$.
 In the following, we measure the quality of an approximate solution  $\hat{\omega} \in \HK$ (an estimator) considering {\it the excess risk}, i.e.,
 \be\label{excessrisk}
 \mcE(\hat{\omega})  - \inf_{\omega \in \HK} \mcE(\omega).
 \ee
 Throughout this paper, we assume that there exists a constant $\kappa \in [1,\infty[$, such that
 \be\label{boundedKernel} \la x,x' \ra_{\HK} \leq \kappa^2, \quad \forall x,x'\in X.
 \ee
\subsection{Stochastic Gradient Method}
We study the following variant of SGM, possibly with mini-batches.
Unlike some of the variants studied in the literature, the algorithm we consider in this paper does not involve any explicit penalty term or any projection step, in which case one does not need to tune the penalty/projection parameter.
\begin{alg}\label{alg:1}
Let $b \in [m].$ Given any sample $\bf z$, the $b$-minibatch stochastic gradient method is defined by $\omega_1 =0$ and
\be\label{Alg}
\omega_{t+1}=\omega_t - \eta_t {1 \over b} \sum_{i= b(t-1)+1}^{bt} (\la \omega_t, x_{j_i}\ra_{\HK} - y_{j_i}) x_{j_i} , \qquad t=1, \ldots, T, \ee
where $\{\eta_{t}>0\}$ is a step-size sequence.  Here, $j_1,j_2,\cdots,j_{bT}$ are i.i.d. random variables from the uniform distribution on $[m]$ \footnote{Note that, the random variables $j_1,\cdots, j_{bT}$ are conditionally independent given the sample $\bf z$.}.
\end{alg}

We add some comments on the above algorithm. First, different choices for the mini-batch size $b$ can lead to different algorithms. In particular, for $b=1$, the above algorithm corresponds to a simple SGM, while for $b=m,$ it is a stochastic version of the batch gradient descent. In this paper, we are particularly interested in the cases of $b=1$ and $b = \sqrt{m}.$
Second, other choices on the initial value, rather than $\omega_1 =0$, is possible. In fact, following from our proofs in this paper, the interested readers can see that the convergence results stated in the next subsections still hold for other choices of initial values.
Finally, the number of total iterations $T$ can be bigger than the number of sample points $m$. This indicates that we can use the sample more than once, or in another words, we can run the algorithm with multiple passes over the data. Here and in what follows, the number of `passes' over the data is referred to $\lceil {bt \over m} \rceil$ at $t$ iterations of the algorithm.

 The aim of this paper is to derive excess risk bounds for Algorithm \ref{alg:1}.
Throughout this paper, we assume that $\{\eta_t\}_t$ is non-increasing, and $T \in \mN$ with $T \geq 3$. We denote by $\J_t$ the set $\{j_l: l=b(t-1)+1,\cdots,bt\}$ and by $\J$ the set $\{j_l: l=1,\cdots,bT\}$.

\section{Main Results with Discussions }\label{sec:main}
In this section, we first state some basic assumptions. Then, we
 present and discuss our main results.

\subsection{Assumptions}
The following assumption is related to a moment assumption on $|y|^2$. It is
weaker than the often considered bounded output assumption, such as the binary classification problems where $Y=\{-1,1\}.$

\begin{as}\label{as:noiseExp}
  There exists constants $M \in ]0,\infty[$ and $v \in ]1,\infty[$ such that
  \be\label{noiseExp}
  \int_{Y} y^{2l} d\rho(y|x) \leq l! M^l v, \quad \forall l \in \mN,
  \ee
  $\rho_{ X}$-almost surely.
\end{as}

To present our next assumption, we introduce the operator $\LK : \LR \to \LR$, defined by $\LK(f) = \int_{X}  \la x, \cdot\ra_{\HK} f(x) \rho_{X}(x).$ {
Here, $\LR$ is the Hilbert space of square integral functions from $\HK$ to $\mR$ with respect to $\rho_X$, with  norm,
$$\|f\|_{\rho} = \left (\int_{X} |f(x)|^2 d \rho_X(x) \right)^{1/2}.$$}
Under Assumption \eref{boundedKernel}, $\LK$ can be proved to be positive trace class operators \citep{cucker2007learning}, and hence $\LK^{\zeta}$ with $\zeta\in \mR$ can be defined by using the spectral theory.

It is well known \citep[see e.g.,][]{cucker2007learning} that the function minimizing $\int_{Z} (f(x) - y)^2 d\rho(z)$
 over all measurable functions $f: \HK \to \mR$ is the regression function, given by
\be\label{regressionfunc}
f_{\rho}(x) = \int_Y y d \rho(y | x),\qquad x \in X.
\ee
Define another Hilbert space $\HR = \{f: X \to  \mR| \exists \omega \in \HK \mbox{ with } f(x) = \la \omega, x \ra_{\HK}, \rho_X \mbox{-almost surely}\}.$
Under Assumption \eref{boundedKernel},  it is easy to see that $\HR$ is a subspace of $\LR.$
Let $\FH$ be the projection of the regression function $f_{\rho}$ onto the closure of $\HR$ in $\LR.$
It is easy to see that the search for a solution of Problem \eref{expectedRisk} is equivalent to the search of a linear function in $\HR$ to approximate $\FH$.
From this point of view, bounds on the excess risk of a learning algorithm on $\HR$ or $\HK$, naturally depend on the following assumption, which quantifies
 how well, the target function $\FH$ can be approximated by $\HR$.

\begin{as}\label{as:regularity}
  There exist $\zeta> 0$ and $R>0$, such that $\|\LK^{-\zeta} \FH \|_{\rho} \leq R.$
\end{as}
The above assumption is fairly standard in non-parametric regression \citep{cucker2007learning,rosasco2015learning}.
The bigger $\zeta$ is, the more stringent the assumption is, since $$\LK^{\zeta_1}(\LR) \subseteq \LK^{\zeta_2}(\LR)\quad \mbox{when }\zeta_1 \geq \zeta_2.$$
In particular, for $\zeta =0,$ we are making no assumption, while for $\zeta = 1/2,$ we are requiring $\FH \in \HR$, since \citep{rosasco2015learning}
\be\label{eq:isometry}
\HR = \LK^{1/2}(\LR)  .
\ee
In the case of $\zeta \geq 1/2$,
$\FH \in \HR$, which implies Problem \eref{expectedRisk} has at least one solution in the space $\HK$.
In this case, we denote  $\omega^{\dag}$ as the solution with the minimal $\HK$-norm.

Finally, the last assumption relates to the capacity of the hypothesis space.
\begin{as}\label{as:eigenvalues}
  For some $\gamma \in ]0,1]$ and $c_{\gamma}>0$, $\LK$ satisfies
\be\label{eigenvalue_decay}
 \tr(\LK(\LK+\lambda I)^{-1})\leq c_{\gamma} \lambda^{-\gamma}, \quad \mbox{for all } \lambda>0.
\ee
\end{as}
The left hand-side of of \eref{eigenvalue_decay} is called as the effective
dimension \citep{caponnetto2007optimal}, or the degrees of freedom \citep{zhang2005learning}.
It can be related to covering/entropy number conditions, see \citep{steinwart2008support} for further details.
Assumption \ref{as:eigenvalues} is always true for $\gamma=1$ and $c_{\gamma} =\kappa^2$, since
 $\LK$ is a trace class operator which implies the eigenvalues of $\LK$, denoted as $\sigma_i$, satisfy
 $\tr(\LK) = \sum_{i} \sigma_i \leq \kappa^2.$
  This is referred to as the capacity independent setting.
  Assumption \ref{as:eigenvalues} with $\gamma \in]0,1]$ allows to derive better error rates. It is satisfied, e.g.,
   if the eigenvalues of $\LK$ satisfy a polynomial decaying condition $\sigma_i \sim i^{-1/\gamma}$, or with $\gamma=0$ if $\LK$ is finite rank.

\subsection{Optimal Rates for SGM and Batch GM: Simplified Versions}

We start with the following corollaries, which are the simplified versions of our main results stated in the next subsections. 
 \begin{corollary}[Optimal Rate for SGM] \label{cor:simplfied}
   Under Assumptions \ref{as:regularity} and \ref{as:eigenvalues}, let $|y| \leq M$ almost surely for some $M>0.$ Let $p_* = \lceil m^{1 \over 2\zeta+\gamma} \rceil$ if $2\zeta+\gamma>1$, or $p_* = \lceil m^{1 -\epsilon} \rceil$ with $\epsilon \in ]0,1[$ otherwise.
   Consider the SGM with \\
1) $b=1$, $\eta_t \simeq {1 \over m}$ for all $t \in [(p_*m)],$ and $\tilde{\omega}_{p_*} = \omega_{p_*m+1}.$  \\
{
If $\delta \in]0,1]$ and $m \geq m_{\delta}$, then
 with probability\footnote{Here, `high probability' refers to the sample ${\bf z}$.}at least $1-\delta$, it holds}
\be\label{eq:optBounds}
\mE_{\J}[\mcE(\tilde{\omega}_{p_*})] - \inf_{ \HK} \mcE \leq C
\begin{cases}
  m^{-{2\zeta \over 2\zeta+\gamma}}& \mbox{ when } 2\zeta+\gamma>1;\\
  m^{-2\zeta(1-\epsilon)}&            \mbox{ otherwise}.
\end{cases}
\ee
Furthermore, the above also holds for the SGM with\footnote{Here, we assume that $\sqrt{m}$ is an integer.}\\
2) $b = \sqrt{m},$ $\eta_t \simeq {1 \over \sqrt{m} }$ for all $t \in [(p_*\sqrt{m})],$ and  $\tilde{\omega}_{p_*} = \omega_{p_* \sqrt{m} + 1}.$\\
{In the above, $m_{\delta}$ and $C$ are positive constants depending on $\kappa^2, \|\TK\|, M, \zeta, R,c_{\gamma},\gamma$, a polynomial of $\log m$ and $\log (1/\delta)$, and $m_{\delta}$ also on $\delta$  (and also on $\|\FH\|_{\infty}$ in the case that $\zeta <1/2$).}
\end{corollary}

We add some comments on the above result. First, the above result asserts that, at $p_*$ passes over the data, the SGM with two different fixed step-size and fixed mini-batch size choices, achieves optimal learning error bounds, matching (or improving) those of ridge regression \citep{smale2007learning,caponnetto2007optimal}.
 Second, according to the above result, using mini-batch allows to use a larger step-size while achieving the same optimal error bounds.
Finally, the above result can be further simplified in some special cases. For example, if we consider the capacity independent case, i.e., $\gamma=1$, and assuming that $\FH \in \HR$, which is equivalent to making Assumption \ref{as:regularity} with $\zeta = 1/2$ as mentioned before, the error bound is
 $O(m^{-1/2})$, while the number of passes $p_* = \lceil\sqrt{m} \rceil.$
\begin{rem}
  [Finite Dimensional Case]  With a simple modification of our proofs, we can derive similar results for the finite dimensional case, i.e., $\HK = \mR^d$, where in this case, $\gamma=0$.
  In particular, letting $\zeta=1/2,$ under the same assumptions of Corollary \ref{cor:simplfied}, if one considers the SGM with $b=1$ and $\eta_t \simeq {1 \over m}$ for all $t \in [m^2],$ then with high probability,
$
\mE_{\J}[\mcE(\omega_{{m}^2+1})] - \inf_{ \HK} \mcE \lesssim {d / m},
$
provided that $m \gtrsim d \log d.$
\end{rem}

\begin{rem}
	{
		From the proofs, one can easily see that if $\FH$ and $\mcE(\tilde{\omega}_{p_*}) - \inf_{\HK}\mcE$   are replaced respectively by $f_*\in \LR$ and $\|\la \cdot,\tilde{\omega}_{p_*} \ra_{\HK} - f_*\|_{\rho}^2$, in both  the assumptions and the error bounds, then all theorems and their corollaries of this paper are still true, as long as  $f_*$ satisfies $\int_{X} (f_* - f_{\rho})(x)K_x d\rho_{X} = 0$.  As a result, if we assume that
		$f_{\rho}$ satisfies Assumption \ref{as:regularity} (with $\FH$ replaced by $f_{\rho}$), as typically done in  \citep{smale2007learning,caponnetto2007optimal,steinwart2009optimal,caponnetto2010cross} for the RKHS setting, we have that with high probability,
		\bea
		\mE_{\J}\|\la \cdot,\tilde{\omega}_{p_*} \ra_{\HK} - f_{\rho}\|_{\rho}^2 \leq  C\begin{cases}
			m^{-{2\zeta \over 2\zeta+\gamma}}& \mbox{ when } 2\zeta+\gamma>1;\\
			m^{-2\zeta(1-\epsilon)}&            \mbox{ otherwise}.
		\end{cases}
		\eea
		In this case, the factor $\|\FH\|_{\infty}$ from the upper bounds for the case $\zeta < 1/2$  is exactly  $\|f_{\rho}\|_{\infty}$ and  can be controlled by the condition $|y| \leq M$ (and more generally, by Assumption \ref{as:noiseExp}). Since many common RKHSs are universally consistent \citep{steinwart2008support}, 
		making Assumption \ref{as:regularity} on $f_{\rho}$ is natural and moreover, deriving error bounds with respect to $f_{\rho}$ seems to be more interesting in this case. 
	}
\end{rem}


As a byproduct of our proofs in this paper, we derive the following optimal results for batch gradient methods (GM), defined by $\nu_1 =0$ and
\be\label{Alg2B}
\nu_{t+1}=\nu_t - \eta_t {1 \over m} \sum_{i=1}^m (\la \nu_t, x_i\ra_{\HK} - y_i) x_i , \qquad t=1, \ldots, T. \ee

\begin{corollary}[Optimal Rate for Batch GM]\label{cor:simplfiedGM}
  Under the assumptions and notations of Corollary \ref{cor:simplfied}, consider
 batch GM \eref{Alg2B} with  $\eta_t \simeq 1 $.
 If $m$ is large enough, then with high probability, \eref{eq:optBounds} holds for
 $\tilde{\omega}_{p_*} = \nu_{p_*+1}.$
\end{corollary}

In the above corollary, the convergence rates are optimal for $2\zeta+\gamma >1$. To the best of our knowledge, these results are the first ones with minimax rates \citep{caponnetto2007optimal,blanchard2016optimal}
for the batch GM in the non-attainable case. Particularly, they improve the results in the previous literature, see Subsection \ref{subsec:discussion} for more discussions.

 Corollaries \ref{cor:simplfied} and \ref{cor:simplfiedGM}  cover the main contributions of this paper. In the following subsections, we will present the main theorems of this paper,  following with several corollaries and simple discussions, from which one can derive the simplified versions stated in this subsection.
In the next subsection, we present results for SGM in the attainable case while results in the non-attainable case will be given in Subsection \ref{subsec:main_non}, as the bounds for these two cases are different and particularly their proofs require different estimations. At last, results with more specific convergence rates for batch GM will be presented in Subsection \ref{subsec:main_bgm}.

\subsection{Main Results for SGM: Attainable Case}
In this subsection, we present convergence results in the attainable case, i.e., $\zeta \geq 1/2$, following with simple discussions.
One of our main theorems in the attainable case is stated next, and provides error bounds for the studied algorithm. For the sake of readability, we only present results in a fixed step-size setting in this section. Results in a general setting ($\eta_t = \eta_1 t^{-\theta}$ with $0\leq \theta<1$ can be found in Section \ref{sec:deriveing}.
\begin{thm}\label{thm:main}
  Under Assumptions \ref{as:noiseExp}, \ref{as:regularity} and \ref{as:eigenvalues}, let $\zeta \geq 1/2$, $\delta \in]0,1[$, $\eta_t = \eta \kappa^{-2} $ for all $t \in [T],$ with $\eta \leq {1 \over 8(\log T + 1)}.$
  If $m \geq m_{\delta}$,
     then the following holds with probability at least $1-\delta$: for all $t \in [T],$
        \be\label{mainTotalErr}
        \begin{split}
        \mE_{\J} [\mcE(\omega_{t+1})] - \inf_{ \HK}\mcE \leq q_1 (\eta t)^{-2\zeta} + q_2 m^{-{2\zeta \over 2\zeta+\gamma}} (1 +  m^{-{1 \over 2\zeta + \gamma}} \eta t )^2 \log^2 T \log^2 {1 \over \delta} \\ + q_3 \eta b^{-1} ( 1 \vee m^{-{1 \over 2\zeta+\gamma }}\eta t ) \log T.
        \end{split}
        \ee
        Here, $m_{\delta},q_1,q_2$ and $q_3$ are positive constants depending on $\kappa^2, \|\TK\|, M, v, \zeta, R,c_{\gamma},\gamma$, and $m_{\delta}$ also on $\delta$ (which will be given explicitly in the proof).
\end{thm}

  There are three terms in the upper bounds of \eref{mainTotalErr}. The first term depends on the regularity of the target function and it arises from bounding the bias, while the last two terms result from estimating the sample variance and the computational variance (due to the random choices of the points), respectively.
  To derive optimal rates, it is necessary to balance these three terms. Solving this trade-off problem leads to different choices on $\eta$, $T$, and $b$, corresponding to different regularization strategies, as shown in subsequent corollaries.

The first corollary gives generalization error bounds for simple SGM, with a universal step-size depending on the number of sample points.
\begin{corollary}\label{cor:MPSGMB}
 Under Assumptions \ref{as:noiseExp}, \ref{as:regularity} and \ref{as:eigenvalues}, let $\zeta \geq 1/2$ , $\delta\in ]0,1[$,
  $b=1$ and $\eta_t \simeq {1\over m}$ for all $t \in [m^2]$.
 If $m \geq m_{\delta},$ then with probability at least $1-\delta$, there holds
 \be\label{corTotalErr}
       \mE_{\J} [\mcE(\omega_{t+1})] - \inf_{ \HK}\mcE  \lesssim \left\{ \Big({m \over t}\Big)^{2\zeta} + m^{-{2\zeta +2 \over 2\zeta+\gamma}}   \Big({t\over m}\Big)^2\right\} \cdot \log^2 m \log^2{1\over \delta}, \quad \forall t\in [m^2],
 \ee
 and in particular,
  \be\label{minimaxBound}
  \mE_{\J} [\mcE(\omega_{T^*+1})] - \inf_{\HK}\mcE  \lesssim  m^{-{2\zeta \over 2\zeta+\gamma}} \log^2 m \log^2{1\over \delta},
  \ee
  where $T^* = \lceil m^{2\zeta+ \gamma+ 1 \over 2\zeta+\gamma} \rceil.$
  Here, $m_\delta$ is exactly the same as in Theorem \ref{thm:main}.
\end{corollary}

\begin{rem}
  Ignoring the logarithmic term and letting $t = pm$, Eq. \eref{corTotalErr} becomes
  \bea
       \mE_{\J} [\mcE(\omega_{pm+1})] - \inf_{ \HK}\mcE \lesssim  p^{-2\zeta} +  m^{-{2\zeta + 2 \over 2\zeta+\gamma}} p^2.
  \eea
  A smaller $p$ may lead to a larger bias, while a larger $p$ may lead to a larger sample error.
  From this point of view, $p$
 has a regularization effect.
\end{rem}
The second corollary provides error bounds for SGM with a fixed mini-batch size and a fixed step-size (which depend on the number of sample points).
\begin{corollary}\label{cor:MbSGMB}
Under Assumptions  \ref{as:noiseExp}, \ref{as:regularity} and \ref{as:eigenvalues}, let $\zeta \geq 1/2$, $\delta\in ]0,1[$,
 $b=\lceil {\sqrt{m}} \rceil$ and $\eta_t \simeq {1\over \sqrt{m}}$ for all $t \in [m^2]$.
 If $m \geq m_\delta,$ then with probability at least $ 1- \delta$, there holds
 \be\label{corTotalErrB}
        \begin{split}
       \mE_{\J} [\mcE(\omega_{t+1})] - \inf_{\HK}\mcE  \lesssim \left\{ \Big({\sqrt{m} \over t} \Big)^{2\zeta} + m^{-{2\zeta +2 \over 2\zeta+\gamma}} \Big({t \over \sqrt{m}} \Big)^2 \right\} \log^2 m \log^2{1\over \delta}, \quad \forall t\in [m^2],
        \end{split}
 \ee
 and particularly,
  \be\label{minimaxBoundB}
  \mE_{\J} [\mcE(\omega_{T^*+1})] - \inf_{\HK}\mcE \lesssim  m^{-{2\zeta \over 2\zeta+\gamma}} \log^2 m \log^2{1\over \delta} ,
  \ee
  where $T^* = \lceil m^{{1 \over 2\zeta+\gamma} + {1\over 2} }\rceil.$
 \end{corollary}
The above two corollaries follow from Theorem \ref{thm:main} with the simple observation that the dominating terms in \eref{mainTotalErr} are the terms related to the bias and the sample variance, when a small step-size is chosen.
 The only free parameter in \eref{corTotalErr} and \eref{corTotalErrB} is the number of iterations/passes. The ideal stopping rule is achieved by balancing the two terms related to the bias and the sample variance, showing the regularization effect of the number of passes. Since the ideal stopping rule depends on the unknown parameters $\zeta$ and $\gamma$, a hold-out cross-validation procedure is often used to tune the stopping rule in practice.
  Using an argument similar to that in Chapter 6 from \citep{steinwart2008support}, it is possible to show that this procedure can achieve the same convergence rate.

We give some further remarks. First,
the upper bound in \eref{minimaxBound} is optimal up to a logarithmic factor, in the sense that it matches the minimax lower rate in \citep{caponnetto2007optimal,blanchard2016optimal}.
Second, according to Corollaries \ref{cor:MPSGMB} and \ref{cor:MbSGMB}, $ {b T^*\over m}  \simeq m^{1 \over 2\zeta+ \gamma} $ passes over the data are needed to obtain optimal rates in both cases. Finally, in comparing the simple SGM and the mini-batch SGM, Corollaries \ref{cor:MPSGMB} and \ref{cor:MbSGMB} show that a larger step-size is allowed to use for the latter.

In the next result, both the step-size and the stopping rule are tuned to obtain optimal rates for simple SGM with multiple passes. In this case, the step-size and the number of iterations are the regularization parameters.

\begin{corollary}\label{cor:MPSGMA}
Under Assumptions \ref{as:noiseExp}, \ref{as:regularity} and \ref{as:eigenvalues}, let $\zeta \geq 1/2$,  $\delta\in ]0,1[$,
  $b=1$ and $\eta_t \simeq m^{-{2\zeta \over 2\zeta+ \gamma}}$ for all $t \in [m^2].$
 If $m \geq m_{\delta},$ and $T^* = \lceil m^{2\zeta+1 \over 2\zeta+\gamma} \rceil, $
  then  \eref{minimaxBound} holds with probability at least $1- \delta.$
\end{corollary}

The next  corollary shows that for some suitable mini-batch sizes, optimal rates can be achieved with a constant step-size (which is nearly independent of the number of sample points) by early stopping.

\begin{corollary}\label{cor:MbSGMA}
Under Assumptions \ref{as:noiseExp}, \ref{as:regularity} and \ref{as:eigenvalues}, let $\zeta \geq1/2,$  $\delta\in ]0,1[$,
  $b=\lceil m^{2\zeta \over 2\zeta+\gamma} \rceil$ and $\eta_t \simeq {1 \over \log m}$ for all $t \in [m]$.
 If $m \geq m_{\delta},$ and $T^* = \lceil m^{ 1 \over 2\zeta+\gamma} \rceil, $
  then  \eref{minimaxBound} holds with probability at least $1- \delta.$
\end{corollary}

According to Corollaries \ref{cor:MPSGMA} and \ref{cor:MbSGMA}, around $m^{1 -\gamma  \over 2\zeta+\gamma}$ passes over the data are needed to achieve the best performance in the above two strategies.
In comparisons with Corollaries \ref{cor:MPSGMB} and \ref{cor:MbSGMB} where around $m^{ \zeta + 1 \over 2\zeta+\gamma}$ passes are required, the latter seems to require fewer passes over the data.
However, in this case, one might have to run the algorithms multiple times to tune the step-size, or the mini-batch size.

\begin{rem}
  1) If we make no assumption on the capacity, i.e., $\gamma=1$, Corollary \ref{cor:MPSGMA} recovers the result in \citep{ying2008online} for one pass SGM.\\
  2) If we make no assumption on the capacity and assume that $\FH \in \HR$,  from Corollaries \ref{cor:MPSGMA} and \ref{cor:MbSGMA}, we see that the optimal convergence rate $O(m^{-1/2})$ can be achieved  after one pass over the data in both of these two strategies. In this special case, Corollaries \ref{cor:MPSGMA} and \ref{cor:MbSGMA} recover the results for one pass SGM in, e.g., \citep{shamir2013stochastic,dekel2012optimal}.
\end{rem}

The next result gives generalization error bounds for `batch' SGM with a constant step-size (nearly independent of the number of sample points).
\begin{corollary}\label{cor:BSGM}
Under Assumptions \ref{as:noiseExp}, \ref{as:regularity} and \ref{as:eigenvalues}, let $\zeta \geq 1/2,$  $\delta\in ]0,1[$,
  $b=m$ and $\eta_t \simeq {1 \over \log m}$ for all $t \in [m].$
 If $m \geq m_{\delta},$ and $T^* = \lceil m^{ 1 \over 2\zeta+\gamma} \rceil, $
  then  \eref{minimaxBound} holds with probability at least $1- \delta.$
\end{corollary}

Theorem \ref{thm:main} and its corollaries give convergence results with respect to the target function values. In the next theorem and corollary,
we will present convergence results in $\HK$-norm.

\begin{thm}\label{thm:hnorm}
  Under the assumptions of Theorem \ref{thm:main}, the following holds with probability at least $1-\delta:$ for all $t\in [T]$
\be\label{eq:genBHnorm}
\mE_{\bf J}[\|\omega_{t} - \omega^{\dag}\|_{\HK}^2] \leq q_1 (\eta t)^{1 - 2\zeta} + q_2 m^{-{2\zeta-1 \over 2\zeta+\gamma}}  (1 +  m^{-{1\over 2\zeta+\gamma}}\eta t )^2\log^2 T\log^2 {1 \over \delta} + q_3 \eta^2 t b^{-1}.
\ee
 Here, $q_1,q_2$ and $q_3$ are positive constants depending on $\kappa^2, \|\TK\|, M, v, \zeta, R,c_{\gamma}$, and $\gamma$  (which can be given explicitly in the proof).
\end{thm}
The proof of the above theorem is similar as that for Theorem \ref{thm:main}, and will be given in Subsection \ref{sec:hnorm}.
Again, the upper bound in \eref{eq:genBHnorm} is composed of three terms related to bias, sample variance, and computational variance.
Balancing these three terms leads to different choices on $\eta$, $T$, and $b$, as shown in the following corollary.
\begin{corollary}\label{cor:hnorm}
  With the same assumptions and notations from any one of Corollaries \ref{cor:MPSGMB} to \ref{cor:BSGM}, the following holds with probability at least $1-\delta:$
  \bea
  \mE_{\bf J}[\|\omega_{T^*+1} - \omega^{\dag}\|_{\HK}^2] \lesssim m^{-{2\zeta-1 \over 2\zeta+\gamma}} \log^2 m\log^2 {1 \over \delta}.
  \eea
\end{corollary}
The convergence rate in the above corollary is optimal up to a logarithmic factor, as it matches the minimax rate shown in \citep{blanchard2016optimal}.


In the next subsection, we will present convergence results in the non-attainable case, i.e., $\zeta < 1/2$.

\subsection{Main Results for SGM: Non-attainable Case}\label{subsec:main_non}
Our main theorem in the non-attainable case is stated next, and provides error bounds for the studied algorithm.
 Here, we present results with a fixed step-size, whereas general results with a decaying step-size will be given in Section \ref{sec:deriveing}.

\begin{thm}\label{thm:generalRateNonFix}
  Under Assumptions  \ref{as:noiseExp}, \ref{as:regularity} and \ref{as:eigenvalues}, let $\zeta \leq 1/2$, $\delta \in]0,1[$, $\eta_t = \eta \kappa^{-2}$ for all $t \in [T],$ with $0<\eta \leq {1 \over 8(\log T+1)} $.
     Then the following holds  for all $t \in [T]$ with probability at least $1-\delta$:
1) if $2\zeta+\gamma>1$ and $m\geq m_{\delta},$ then
        \be\label{totalErrGenNonFixA}
        \begin{split}
        \mE_{\J}[\mcE(\omega_{t+1})] - \inf_{ \HK} \mcE
        \leq \left( q_1 (\eta t)^{-2\zeta}
        +  q_2 m^{-{2\zeta \over 2\zeta+\gamma}}  \right) (1 \vee   m^{-{1\over 2\zeta+\gamma}} \eta t )^3 \log^4 T  \log^2 {1 \over \delta} \\
        + q_3\eta b^{-1} ( 1 \vee m^{-{1\over 2\zeta+\gamma}}\eta t ) \log T ;
        \end{split}
        \ee
 2) if $2\zeta+\gamma\leq 1$ and for some $\epsilon\in]0,1[$, $m\geq m_{\delta,\epsilon}$, then
 \bea
        \begin{split}
        \mE_{\J}[\mcE(\omega_{t+1})] - \inf_{\HK} \mcE
        \leq \left( q_1(\eta t)^{-2\zeta}
        +  q_2 m^{\gamma(1-\epsilon) -1}  \right) (1 \vee  \eta m^{\epsilon-1} t )^3 \log^4 T  \log^2 {1 \over \delta} \\
        + q_3 \eta b^{-1} ( 1 \vee m^{\epsilon-1}\eta t ) \log T .
        \end{split}
        \eea
         Here, $m_{\delta}$ (or $m_{\delta,\epsilon}$), $q_1,q_2$ and $q_3$ are positive constants depending only on $\kappa^2, \|\TK\|, M, v, \zeta, R,c_{\gamma},\gamma$, $\|\FH\|_{\infty}$, and $m_{\delta}$ (or  $m_{\delta,\epsilon}$) also on $\delta$ (and $\epsilon$).
\end{thm}

The upper bounds in \eref{mainTotalErr} (for the attainable case) and \eref{totalErrGenNonFixA} (for the non-attainable case) are similar, whereas the latter has an extra logarithmic factor. Consequently, in the subsequent corollaries, we derive  $O(m^{-{2\zeta \over 2\zeta+\gamma}} \log^4 m)$ for the non-attainable case. In comparison with that for the attainable case, the convergence rate for the non-attainable case has an extra $\log^2 m$ factor.

Similar to Corollaries \ref{cor:MPSGMB} and \ref{cor:MbSGMB}, and as direct consequences of the above theorem, we have the following generalization error bounds for the studied algorithm with different choices of parameters in the non-attainable case.

\begin{corollary}\label{cor:MPSGMBNon}
 Under Assumptions \ref{as:noiseExp}, \ref{as:regularity} and \ref{as:eigenvalues}, let $\zeta \leq 1/2$ , $\delta\in ]0,1[$,
  $b=1$ and $\eta_t \simeq {1\over m}$ for all $t \in [m^2]$. With probability at least $1-\delta$, the following holds:\\
1) if $2\zeta+\gamma>1$, $m \geq m_{\delta}$ and $T^* = \lceil m^{ 1+2\zeta+\gamma \over 2\zeta+\gamma} \rceil$, then
  \be\label{minimaxBoundNonA}
  \mE_{\J} [\mcE(\omega_{T^*+1})] - \inf_{\HK}\mcE  \lesssim  m^{-{2\zeta \over 2\zeta+\gamma}} \log^4 m \log^2{1\over \delta};
  \ee
2) if $2\zeta+\gamma \leq 1$, and for some $\epsilon \in ]0,1[$, $m \geq m_{\delta, \epsilon}$, and
 $T^* = \lceil m^{2 - \epsilon}\rceil,$
 then
  \be\label{minimaxBoundNonB}
  \mE_{\J} [\mcE(\omega_{T^*+1})] - \inf_{\HK}\mcE  \lesssim  m^{-{2\zeta}(1-\epsilon)} \log^4 m \log^2{1\over \delta}.
  \ee
 Here, $m_\delta$ and $m_{\delta,\epsilon}$ are  given by Theorem \ref{thm:generalRateNonFix}.
\end{corollary}

\begin{corollary}
 Under Assumptions \ref{as:noiseExp}, \ref{as:regularity} and \ref{as:eigenvalues}, let $\zeta \leq 1/2$ , $\delta\in ]0,1[$,
  $b \simeq \sqrt{m}$ and $\eta_t \simeq {1\over \sqrt{m}}$ for all $t \in [m^2]$. With probability at least $1-\delta$, there holds\\
1) if $2\zeta+\gamma>1$, $m \geq m_{\delta}$ and $T^* = \lceil m^{{1\over 2\zeta+\gamma} + {1 \over 2}} \rceil$, then \eref{minimaxBoundNonA} holds;\\
2) if $2\zeta+\gamma \leq 1$, for some $\epsilon \in ]0,1[$, $m \geq m_{\delta, \epsilon}$, and
 $T^* = \lceil m^{{3\over 2} - \epsilon}\rceil,$
 then \eref{minimaxBoundNonB} holds.
\end{corollary}

The convergence rates in the above corollaries, i.e., $m^{-{2\zeta \over 2\zeta+\gamma}}$ if $2\zeta+\gamma>1$ or $m^{-{2\zeta(1-\epsilon)}}$ otherwise, match those in \citep{dieuleveut2014non} for  one pass SGM with averaging, up to a logarithmic factor. Also, in the capacity independent case, i.e., $\gamma=1$, the convergence rates in the above corollary read as $m^{-{2\zeta \over 2\zeta+1}}$ (since $2\zeta+\gamma$ is always bigger than $1$), which are exactly the same as those  in  \citep{ying2008online} for one pass SGM.

Similar results to Corollaries \ref{cor:MPSGMA}--\ref{cor:BSGM} can be also derived
for the non-attainable case by applying Theorem \ref{thm:generalRateNonFix}.  Refer to Appendix \ref{app:further} for more details.

\subsection{Main Results for Batch GM}\label{subsec:main_bgm}
In this subsection, we present convergence results for batch GM.
As a byproduct of our proofs in this paper, we have the following convergence rates for batch GM.
\begin{thm}\label{thm:bgmopt}
  Under Assumptions  \ref{as:noiseExp}, \ref{as:regularity} and \ref{as:eigenvalues}, set $\eta_t \simeq 1 $, for all $t \in [m].$ Let $T^* = \lceil m^{1 \over 2\zeta+\gamma} \rceil$ if $2\zeta+\gamma>1$, or $T^* = \lceil m^{1 -\epsilon} \rceil$ with $\epsilon\in]0,1[$ otherwise. Then with probability at least $1-\delta$ ($0< \delta < 1$), the following holds for the learning sequence generated by \eref{Alg2B}: \\
     1) if $\zeta>1/2$ and $m \geq m_{\delta}$, then
      \bea
\mcE(\nu_{T_*+1}) - \inf_{\HK} \mcE \lesssim m^{-{2\zeta \over 2\zeta+\gamma}} \log^2 m \log^2 {1\over \delta};
\eea
2) if $\zeta\leq 1/2,$ $2\zeta+\gamma >1$ and $m \geq m_{\delta}$, then
      \bea
\mcE(\nu_{T_*+1}) - \inf_{\HK} \mcE \lesssim m^{-{2\zeta \over 2\zeta+\gamma}} \log^4 m \log^2 {1\over \delta};
\eea
3) if $2\zeta+\gamma\leq 1$ and $m\geq m_{\delta, \epsilon}$, then
\bea
\mcE(\nu_{T_*+1})
 - \inf_{\HK} \mcE \lesssim m^{-{2\zeta(1-\epsilon)}} \log^4 m \log^2 {1\over \delta}.
\eea
{Here, $m_{\delta}$ (or $m_{\delta,\epsilon}$), and all the  constants in the upper bounds are positive and  depend only on $\kappa^2, \|\TK\|, M, v, \zeta, R,c_{\gamma},\gamma$, $\|\FH\|_{\infty}$, and $m_{\delta}$ (or  $m_{\delta,\epsilon}$) also on $\delta$ (and $\epsilon$)}.
\end{thm}

\subsection{Discussions}\label{subsec:discussion}
 We must compare our results with previous works.
For non-parametric regression with the square loss, one pass SGM has been studied in, e.g., \citep{ying2008online,shamir2013stochastic,tarres2014online,dieuleveut2014non}. In particular, \citet{ying2008online} proved capacity independent rate of order $O(m^{-{2\zeta \over 2\zeta + 1}} \log m)$ with a fixed step-size $\eta \simeq m^{-{2\zeta \over 2\zeta+1}}$,
 and \cite{dieuleveut2014non} derived capacity dependent error bounds of order $O(m^{-{2\min(\zeta,1) \over 2\min(\zeta,1) + \gamma}})$ (when $2\zeta + \gamma >1$) for the average. Note also that a regularized  version of SGM has been studied in \citep{tarres2014online}, where the derived convergence rate is of order $O(m^{-{2\zeta \over 2\zeta + 1}})$ assuming that $\zeta \in [{1\over 2},1].$ In comparison with these existing convergence rates, our rates from \eref{minimaxBound} are comparable, either involving the capacity condition, or allowing a broader regularity parameter $\zeta$ (which thus improves the rates).
For finite dimensional cases, it has been shown in \citep{bach2013non} that one pass SGM with averaging with a constant step-size achieves the optimal
convergence rate of $O(d/m).$ In comparisons, our results for multi-pass SGM with a smaller step-size seems to be suboptimal in the computational complexity, as we need $m$ passes over the data to achieve the same rate.  \revision{The reason for this may arise from  ``the computational error" that will be introduced later, or the fact that we do not consider an averaging step as done in \citep{bach2013non}.}
 We hope that in the future by considering a larger step-size and averaging, one can reduce the computational complexity of multi-pass SGM while achieving the same rate.

   More recently, \citet{rosasco2015learning} studied multiple passes SGM with a fixed ordering at each pass, also called incremental gradient method.
     Making no assumption on the capacity, rates of order $O(m^{-{\zeta \over \zeta +1}})$ (in $\LR$-norm) with a universal step-size $\eta \simeq {1 / m}$ are derived. In comparisons,  Corollary \ref{cor:MPSGMB} achieves better rates, while considering the capacity assumption.
      Note also that \citet{rosasco2015learning} proved sharp rate in $\HK$-norm for $\zeta \geq 1/2$ in the capacity independent case.
   In comparisons, we derive optimal capacity-dependent rate, considering mini-batches.

The idea of using mini-batches (and parallel implements) to speed up SGM in a general stochastic optimization setting can be found, e.g., in \citep{shalev2011pegasos,dekel2012optimal,sra2012optimization,ng2016machine}.
Our theoretical findings, especially the interplay between the mini-batch size and the step-size, can give further insights on parallelization learning.
Besides, it has been shown in \citep{cotter2011better,dekel2012optimal} that for one pass mini-batch SGM with a fixed step-size $\eta \simeq b/\sqrt{m}$ and a smooth loss function, assuming the existence of at least one solution in the hypothesis space for the expected risk minimization, the convergence rate is of order $O(\sqrt{1/m} + b/m)$ by considering an averaging scheme. When adapting to the learning setting we consider, this reads as that if $\FH \in \HR$, i.e., $\zeta =1/2,$
the convergence rate for the average is  $O(\sqrt{1/m} + b/m)$. Note that, $\FH$ does not necessarily belong to $\HR$ in general. Also, our derived convergence rate from Corollary \ref{cor:MbSGMB} is better, when the regularity parameter $\zeta$ is greater than $1/2,$ or $\gamma$ is smaller than $1$.

For batch GM in the attainable case, convergent results with optimal rates  have been derived in, e.g, \citep{bauer2007regularization,caponnetto2010cross,blanchard2016optimal,dicker2016kernel}. In particular, \cite{bauer2007regularization} proved convergence rates $O(m^{-{2\zeta \over 2\zeta+1}})$ without considering Assumption \ref{as:eigenvalues}, and \cite{caponnetto2010cross} derived convergence rates $O(m^{-{2\zeta \over 2\zeta+\gamma}}).$ For the non-attainable case, convergent results with suboptimal rates $O(m^{-2\zeta \over 2\zeta+2 })$ can be found in \citep{yao2007early}, and to the best of our knowledge, the only result with optimal rate $O(m^{-2\zeta \over 2\zeta+\gamma })$
is the one derived by \cite{caponnetto2010cross}, but the result requires extra unlabeled data. In contrast, Theorem \ref{thm:bgmopt} of this paper does not require any extra unlabeled data, while achieving the same optimal rates (up to a logarithmic factor). To the best of our knowledge, Theorem \ref{thm:bgmopt} may be the first optimal result in the non-attainable case for batch GM.

We end this discussion with some further comments on batch GM and simple SGM. First, according to Corollaries \ref{cor:simplfied} and \ref{cor:simplfiedGM}, it seems that both simple SGM (with step-size $\eta_t \simeq m^{-1}$) and batch GM (with step-size $\eta_t \simeq 1$) have the same computational complexities (which are related to the number of passes) and the same orders of upper bounds. However, there is a subtle difference between these two algorithms. As we see from \eref{unbias_index} in the coming subsection, every $m$ iterations of simple SGM (with step-size $\eta_t \simeq m^{-1}$) corresponds to one iteration of batch GM (with step-size $\eta_t \simeq 1$). In this sense, SGM discretizes and refines the regularization path of batch GM, which thus may lead to smaller generalization errors. This phenomenon can be further understood by comparing our derived bounds, \eref{mainTotalErr} and \eref{eq:bgmBound}, for these two algorithms. Indeed, if one can ignore the computational error, one can easily show that the minimization (over $t$) of right hand-side of \eref{mainTotalErr} with $\eta \simeq m^{-1}$  is always smaller than that of \eref{eq:bgmBound} with $\eta \simeq 1$.
At last, by Corollary \ref{cor:MbSGMA}, using a larger step-size for SGM allows one to stop earlier (while sharing the same optimal rates), which thus reduces the computational complexity. This suggests that SGM may have some computational advantage over batch GM.



\subsection{Proof Sketch (Error Decomposition)}\label{subsec:proSke}
{The key to our proof is a novel error decomposition, which may be also used in analysing other learning algorithms. One may also use the approach in \citep{bousquet2008tradeoffs,lin2015iterative,lin2016generalization} which is based on the following error decomposition,
\begin{eqnarray*}
 {\mathbb E}
\mcE(\omega_t)-\inf_{\HK}\mcE =
 [{\mathbb E} (\mcE(\omega_t)- \mcE_{\mathbf{z}}(\omega_t)) +{\mathbb E}\mcE_{\mathbf{z}}(\tilde \omega) - \mcE(\tilde \omega)]
+{\mathbb E}(\mcE_{\mathbf{z}}(\omega_t)-\mcE_{\mathbf{z}}(\tilde \omega))
+ \mcE(\tilde \omega) -\inf_{\HK} \mcE,
\end{eqnarray*}
where $\tilde{\omega}\in \HK$ is some suitably intermediate element and
$\mcE_{\mathbf{z}}$ denotes the empirical risk over $\bf z$, i.e.,
\be\label{eq:emprisk}
\mcE_{\bf z}(\cdot) = {1\over m} \sum_{i=1}^m \left(\la \cdot, x_i \ra - y_i \right)^2.
\ee
 However, one can only derive a sub-optimal convergence rate, since the proof procedure  involves upper bounding the learning sequence to estimate the sample error (the first term of right-hand side). Also, in this case, the `regularity' of the regression function can not be fully utilized for estimating the bias (the last term). Thanks to the  property of squares loss, we can exploit a different  error decomposition leading to better results.}

To describe the decomposition,
 we need to introduce two sequences. The \emph{population iteration} is defined by $\mu_1 =0$ and
\be\label{Alg3B}
\mu_{t+1}=\mu_{t} - \eta_t \int_{X}(\la \mu_{t}, x \ra_{\HK} - f_{\rho}(x))  x d\rho_{X}(x), \qquad t=1, \ldots, T. \ee
The above iterated procedure is ideal and can not be implemented in practice, since the distribution $\rho_{X}$ is unknown in general.
Replacing $\rho_{X}$ by the empirical measure and  $f_{\rho}(x_i)$ by $y_i$, we derive the \emph{sample iteration} (associated with the sample $\bf z$), i.e., \eref{Alg2B}.
Clearly, $\mu_{t}$ is deterministic and $\nu_{t}$ is a $\HK$-valued random variable depending on $\bf z.$
Given the sample $\bf z$, the sequence $\{\nu_{t}\}_{t}$ has a natural relationship with the learning sequence $\{\omega_t\}_{t}$, since
\be\label{unbias_index}
\mE_{\J}[\omega_t] = \nu_t.
\ee
Indeed, taking the expectation with respect to $\J_t$ on both sides of \eref{Alg}, and noting that
$\omega_{t}$ depends only on $\J_1,\cdots,\J_{t-1}$ (given any $\bf z$), one has
$$
\mE_{\J_t}[\omega_{t+1}] = \omega_t - \eta_t {1 \over m}\sum_{i=1}^m( \la \omega_t, x_i \ra_{\HK} - y_{i}) x_i, $$
and thus,
$$
\mE_{\J}[\omega_{t+1}] = \mE_{\J}[\omega_t] - \eta_t {1 \over m}\sum_{i=1}^m(\la \mE_{\J}[\omega_t], x_i \ra_{\HK} - y_{i}) x_i , \quad t=1, \ldots, T, $$
which satisfies the iterative relationship given in \eref{Alg2B}. By an induction argument, \eref{unbias_index} can then be proved.

Let $\IK: \HK \to \LR$ be the linear map defined by
$(\IK \omega)(x) = \la \omega,x \ra_{\HK}, \forall \omega,x \in \HK.$
We have the following error decomposition.
\begin{pro}
  We have
  \be\label{errorDecompos}
  \mE_{\J}[\mcE(\omega_t)] - \inf_{ \HK}\mcE\leq 2\|\IK \mu_{t} - \FH\|_{\rho}^2 + 2 \|\IK \nu_t - \IK \mu_{t}\|_{\rho}^2 + \mE_{\J}[\|\IK \omega_t - \IK \nu_{t}\|_{\rho}^2].
  \ee
\end{pro}
\begin{proof}
For any $\omega \in \HK$, we have \citep{rosasco2015learning}
$$
\mcE(\omega) - \inf_{\omega \in \HK}\mcE(\omega) = \|\IK \omega - \FH\|_{\rho}^2.
$$
Thus, $
\mcE(\omega_t) - \inf_{ \HK}\mcE  = \|\IK \omega_t - \FH\|_{\rho}^2,
$ and
 \bea
 &&\mE_{\J}[\|\IK \omega_t - \FH\|_{\rho}^2] = \mE_{\J}[\|\IK \omega_t - \IK \nu_t + \IK \nu_t - \FH\|_{\rho}^2] \\
 && = \mE_{\J}[\|\IK \omega_t - \IK \nu_t\|_{\rho}^2 + \|\IK \nu_t - \FH\|_{\rho}^2] + 2 \mE_{\J}\la \IK \omega_t - \IK \nu_t, \IK \nu_t - \FH \ra_{\rho}.
 \eea
 Using \eref{unbias_index} in the above equality,
 we get,
 $$
 \mE_{\J}[\|\IK \omega_{t} - \FH\|_{\rho}^2]
 = \mE_{\J}[\|\IK \omega_t - \IK \nu_t\|_{\rho}^2 + \|\IK \nu_t - \FH\|_{\rho}^2].
 $$
 The proof is finished by considering,
 \bea
 \|\IK \nu_t - \FH\|_{\rho}^2 = \|\IK \nu_t - \IK \mu_{t} + \IK \mu_{t} -  \FH\|_{\rho}^2 \leq 2\|\IK \nu_t - \IK \mu_{t}\|_{\rho}^2 + 2\|\IK \mu_{t} -  \IK \FH\|_{\rho}^2.
 \eea
\end{proof}
There are three terms in the upper bound of the error decomposition \eref{errorDecompos}.
We refer to
the deterministic term  $\|\IK \mu_{t} - \FH\|_{\rho}^2$ as the {\it bias}, the term $\|\IK \nu_t - \IK \mu_{t}\|_{\rho}^2$ depending on $\bf z$ as the \emph{sample variance}, and $\mE_{\J}[\|\IK \omega_t - \IK \nu_t\|_{\rho}^2]$ as the \emph{computational variance}.  {The bias term, which is deterministic, has been well studied in the literature, see e.g., \citep{yao2007early} and also \citep{rosasco2015learning}. The main novelties of this paper are the estimate of the sample and
computational variances and the difficult part is the estimate of the computational variances.
 The  proof of these results  is quite lengthy and makes use of some ideas from \citep{yao2007early,smale2007learning,bauer2007regularization,ying2008online,tarres2014online,rudi2015less}.
These three error terms will be estimated in Sections \ref{sec:biasSam} and \ref{sec:comp}.
The bounds in Theorems \ref{thm:main} and \ref{thm:generalRateNonFix} thus follow plugging these estimations in the error decomposition, see Section \ref{sec:deriveing} for more details.}
The proof for Theorem \ref{thm:hnorm} is similar, see Section \ref{sec:hnorm} for the details.

\section{Preliminary Analysis}\label{sec:estimates}
In this section, we introduce some notation and preliminary lemmas that are necessary to our proofs.
\subsection{Notation}
We first introduce some notations.
For $t \in \mN,$
$\Pi_{t+1}^T(L) = \prod_{k=t+1}^T (I - \eta_k L)$ for $t \in [T-1]$ and $\Pi_{T+1}^T(L) = I,$ for any operator $L: \mathcal{H} \to \mathcal{H},$
where $\mathcal{H}$ is a Hilbert space and $I$ denotes the identity operator on $\mathcal{H}$.
$\mE[\xi]$ denotes the expectation of a random variable $\xi.$
For a given bounded operator $L: \LR \to \HK, $ $\|L\|$ denotes the operator norm of $L$, i.e., $\|L\| = \sup_{f\in \LR, \|f\|_{\rho}=1} \|Lf\|_{\HK}$.
We will use the conventional notations on summation and production: $\prod_{i=t+1}^t = 1$ and $\sum_{i=t+1}^t = 0.$

We next introduce some auxiliary operators. Let $\IK: \HK \to \LR$ be the linear map $\omega \to \la \omega, \cdot \ra_{\HK}$, which is bounded by $\kappa$ under Assumption \eref{boundedKernel}. Furthermore, we consider the adjoint operator $\IK^*: \LR \to \HK$, the covariance operator $\TK: \HK \to \HK$ given by $\TK = \IK^* \IK$, and the operator $\LK : \LR \to \LR$ given by $\IK \IK^*.$ It can be easily proved that $ \IK^*g = \int_X x g(x) d\rho_X(x)$
and $\TK = \int_X \la \cdot , x \ra_{\HK} x d \rho_X(x).$
The operators $\TK$ and $\LK$ can be proved to be positive trace class operators (and hence compact).
 For any $\omega \in \HK$,
it is easy to prove the following isometry property \citep{steinwart2008support}
\be\label{isometry}
\|\IK \omega \|_{\rho} = \|\sqrt{\TK} \omega\|_{\HK}.
\ee
We define the sampling operator $\SX: \HK \to \mR^m$ by $(\SX \omega)_i = \la \omega, x_i \ra_{\HK},$ $i \in [m]$, where the norm $\|\cdot\|_{\mR^m}$ in $\mR^m$ is the Euclidean norm times $1/\sqrt{m}$.
Its adjoint operator $\SX^*: \mR^m \to \HK,$ defined by $\la \SX^*{\bf y}, \omega \ra_{\HK} = \la {\bf y}, \SX \omega\ra_{\mR^m}$ for ${\bf y} \in \mR^m$ is thus given by $\SX^*{\bf y} = {1 \over m} \sum_{i=1}^m y_i x_i.$ Moreover, we can define the empirical covariance operator $\TX: \HK \to \HK$ such that $\TX = \SX^* \SX$. Obviously,
\bea
\TX = {1 \over m} \sum_{i=1}^m \la \cdot, x_i \ra_{\HK} x_i.
\eea
With these notations, \eref{Alg3B} and \eref{Alg2B} can be rewritten as
\be\label{Alg3}
\mu_{t+1}=\mu_{t} - \eta_t (\TK \mu_{t} - \IK^* f_{\rho})  , \qquad t=1, \ldots, T, \ee
and
\be\label{Alg2}
\nu_{t+1}=\nu_t - \eta_t (\TX \nu_t - \SX^* {\bf y}), \qquad t=1, \ldots, T, \ee
respectively.

Using the projection theorem, one can prove that
\be\label{frFH}
\IK^* f_{\rho} = \IK^* \FH.
\ee
Indeed, since $\FH$ is the projection of the regression function $f_{\rho}$ onto the closure of $\HR$ in $\LR,$ according to  the projection theorem, one has
\bea
\la \FH - f_{\rho}, \IK \omega \ra_{\rho} =0, \qquad \forall \omega \in \HK,
\eea
which can be written as
\bea
\la \IK^* \FH - \IK^* f_{\rho}, \omega \ra_{\HK} =0, \qquad \forall \omega \in \HK,
\eea
and thus leading to \eref{frFH}.

\subsection{Concentration Inequality}
We need the following concentration result for Hilbert space valued random variable
used in \citep{caponnetto2007optimal} and based on the results in \citep{pinelis1986remarks}.

\begin{lemma}
  \label{lem:Bernstein}
  Let $w_1,\cdots,w_m$ be i.i.d random variables in a Hilbert space with norm $\|\cdot\|$. Suppose that
   there are two positive constants $B$ and $\sigma^2$ such that
   \be\label{bernsteinCondition}
   \mE [\|w_1 - \mE[w_1]\|^l] \leq {1 \over 2} l! B^{l-2} \sigma^2, \quad \forall l \geq 2.
   \ee
   Then for any $0< \delta <1$, the following holds with probability at least $1-\delta$,
  $$ \left\| {1 \over m} \sum_{k=1}^m w_m - \mE[w_1] \right\| \leq 2\left( {B \over m} + {\sigma \over \sqrt{ m }} \right) \log {2 \over \delta} .$$
In particular, \eref{bernsteinCondition} holds if
\be\label{bernsteinConditionB}
\|w_1\| \leq B/2 \ \mbox{ a.s.}, \quad \mbox{and } \quad \mE [\|w_1\|^2] \leq \sigma^2.
\ee
\end{lemma}
\subsection{Basic Estimates}

Finally, we introduce the following three basic estimates, whose proofs can be found in Appendix \ref{sec:prov_estimates}.
\begin{lemma}\label{lem:estimate1}
  Let $\theta\in [0,1[$, and $t\in\mN$. Then
 $$ {t^{1-\theta} \over 2} \leq \sum_{k=1}^{t} k^{-\theta} \leq
   {t^{1 - \theta} \over 1-\theta}.
 $$
\end{lemma}

\begin{lemma}
  \label{lem:estimate1a}
  Let $\theta \in \mR$ and $t \in \mN $.
  Then $$ \sum_{k=1}^{t} k^{-\theta} \leq t^{\max(1-\theta,0)} (1+\log t).
 $$
\end{lemma}

\begin{lemma}\label{lem:estimate2}
Let  $q \in \mR$ and $t\in\mN$ with $t\geq 3$. Then
  \bea  \sum_{k=1}^{t-1} {1 \over t-k} k^{-q}
 \leq 2 t^{-\min(q,1)} (1+\log t).
  \eea
\end{lemma}
{
In the next sections, we begin  proving  the main results. The proofs are quite lengthy and they are divided into several steps.  For the ease of readability, we list some of the notations and definitions in Appendix \ref{sec:notations}. We also remark that
we are particularly interested in developing error bounds in terms of the stepsize $\eta_t$ ($=\eta_1 t^{-\theta}$), the number of iterations $t$ or $T$, the `regularization' parameter $\lambda>0$, the sample size $m$, the minibatch size $b$,  and the failing profitability $\delta$. 
Other parameters such as $\kappa^2, \|\TK\|, M, v, R,c_{\gamma}$ and $\|\FH\|_{\infty}$ can be always viewed as some constants, which are less important in our error bounds.
}
\section{Estimating Bias and Sample Variance}\label{sec:biasSam}
In this section, we estimate the bias and the sample variance.

\subsection{Bias}\label{sec:initial}
In this subsection, we develop upper bounds for the bias, i.e., $\|\IK \mu_{t} - \FH\|_{\rho}^2$. Towards this end, we introduce the following lemma, whose proof borrows idea from \citep{ying2008online,tarres2014online}.
\begin{lemma}
  \label{lemma:initialerror}
  Let $L$ be a compact, positive operator on a separable Hilbert space $H$. Assume that $\eta_1 \|L\| \leq 1$. Then for $t\in \mN$ and any non-negative integer $k  \leq t - 1,$
  \be\label{initialerror_interm}
  \| \Pi_{k+1}^t(L) L^{\zeta}\| \leq \left( \zeta \over \mathrm{e} \sum_{j=k+1}^t \eta_j \right)^{\zeta}.
  \ee
\end{lemma}
\begin{proof}
   Let $\{\sigma_i\}$ be the sequence of eigenvalues of $L.$
   We have
  \bea
  \| \Pi_{k+1}^t(L) L^{\zeta}\| = \sup_{i} \prod_{l=k+1}^t (1 - \eta_l \sigma_i)\sigma_i^{\zeta}.
  \eea
  Using the basic inequality
  \be\label{expx}
  1 + x \leq \mathrm{e}^{x} \qquad \mbox{for all } x \geq -1,
  \ee
  with $\eta_l\|L\| \leq 1$, we get
  \bea
  \| \Pi_{k+1}^t(L) L^{\zeta}\| &\leq& \sup_i \exp\left\{ - \sigma_i \sum_{l=k+1}^t \eta_l\right\} \sigma_i^{\zeta} \\
  &\leq &\sup_{x \geq 0} \exp\left\{ - x  \sum_{l=k+1}^t \eta_l\right\} x^{\zeta}.
  \eea
  The maximum of the function $g(x) = \mathrm{e}^{-cx}x^{\zeta}$( with $c>0$) over $ \mR_+ $ is achieved at $x_{\max}= \zeta/c,$ and thus
  \be\label{exppoly}
  \sup_{x \geq 0} \mathrm{e}^{-cx} x^{\zeta} =  \left({\zeta \over \mathrm{e}c} \right)^{\zeta}.
  \ee
  Using this inequality, one can get the desired result \eref{initialerror_interm}.
\end{proof}

With the above lemma and Lemma \ref{lem:estimate1}, we can derive the following result for the bias.
\begin{pro}\label{pro:initialErr}
  Under Assumption \ref{as:regularity}, let $\eta_1 \kappa^2 \leq 1$. Then, for any $t\in \mN,$
  \be\label{initialErrA}
  \|\IK \mu_{t+1} - \FH\|_{\rho} \leq R \left( \zeta \over 2 \sum_{j=1}^t \eta_j \right)^{\zeta}.
  \ee
  In particular, if $\eta_t = \eta t^{-\theta}$ for all $t\in \mN$, with $\eta \in ]0,\kappa^{-2}]$ and $\theta \in [0,1[,$ then
\be\label{initialErrB}
  \|\IK \mu_{t+1} - \FH\|_{\rho} \leq R \zeta^{\zeta} \eta^{-\zeta} t^{(\theta-1)\zeta} .
  \ee
\end{pro}
The above result is essentially proved in \citep{yao2007early}, see also \citep{rosasco2015learning} when step-size is fixed. For the sake of completeness, we provide a proof in Appendix \ref{sec:prov_estimates}.
The following lemma gives upper bounds for the sequence $\{ \mu_{t}\}_{t\in \mN}$ in $\HK$-norm. It will be used for the estimation on the sample variance in the next section.
\begin{lemma}\label{lemma:htInfty}
Under Assumption \ref{as:regularity}, let $\eta_1 \kappa^2 \leq 1$. The following holds for all  $t\in \mN$:\\
1) If $\zeta \geq 1/2,$
  \be\label{htInfty}
  \| \mu_{t+1} \|_{\HK} \leq R \kappa^{2\zeta -1}.
  \ee
2) If $\zeta \in ]0,1/2],$
\be\label{htInftyNon}
  \| \mu_{t+1} \|_{\HK} \leq R\left\{ \kappa^{2\zeta-1 }\vee \left(\sum_{k=1}^t \eta_k\right)^{{1\over 2} - \zeta} \right\}.
  \ee
\end{lemma}
\begin{proof}
The proof can be found in Appendix \ref{sec:prov_estimates}.
The proof for a fixed step-size (i.e., $\eta_t = \eta$ for all $t$) can be also found in \citep{rosasco2015learning}. For a general step-size, the proof is similar. Note also that our proof for the non-attainable case is simpler than that in \citep{rosasco2015learning}.
\end{proof}

\subsection{Sample Variance}\label{subsec:sam}
In this subsection, we estimate  the sample variance, i.e., $\mE[\|\IK \mu_t - \IK \nu_t \|_{\rho}^2].$
Towards this end, we need some preliminary analysis. We first introduce the following key inequality, which also provides the basic idea on estimating $\mE[\|\IK \mu_t - \IK \nu_t\|_{\rho}^2].$
\begin{lemma}\label{lemma:ideaSample}
  For all $t \in [T],$ we have
  \be\label{ideaSample}
  \|\IK \nu_{t+1} - \IK \mu_{t+1}\|_{\rho} \leq  \sum_{k=1}^t \eta_k \left\| \TK^{1\over 2}\Pi_{k+1}^t (\TX) N_k\right\|_{\HK},
  \ee
  where
   \be\label{Nt}
  N_k =  (\TK \mu_k - \IK^* f_{\rho}) - (\TX \mu_k - \SX^*{\bf y}), \qquad \forall k\in[T].
  \ee
\end{lemma}
\begin{proof}
  Since $\nu_{t+1}$ and $\mu_{t+1}$ are given by \eref{Alg2} and \eref{Alg3}, respectively,
  \bea
  \nu_{t+1} - \mu_{t+1} &=& \nu_t - \mu_{t}  + \eta_t \left\{ (\TK \mu_{t} - \IK^* f_{\rho}) - (\TX \nu_t - \SX^*{\bf y}) \right\} \\
                    &=&(I - \eta_t \TX) (\nu_t - \mu_{t}) + \eta_t \left\{ (\TK \mu_{t} - \IK^* f_{\rho}) - (\TX \mu_{t} - \SX^*{\bf y})\right\},
  \eea
  which is exactly
   \bea
  \nu_{t+1} - \mu_{t+1}= (I - \eta_t \TX) (\nu_t - \mu_{t}) + \eta_t N_t.
  \eea
  Applying this relationship iteratively, with $\nu_1 = \mu_1 =0,$
   \be\label{eq:interm2}
  \nu_{t+1} - \mu_{t+1} = \Pi_{1}^t (\TX) (\nu_1 - \mu_1) + \sum_{k=1}^t \eta_k \Pi_{k+1}^t(\TX)  N_k
  = \sum_{k=1}^t \eta_k \Pi_{k+1}^t(\TX)  N_k.
  \ee
   By \eref{isometry}, we have
  \bea
  \|\IK \nu_{t+1} - \IK \mu_{t+1}\|_{\rho} = \left\| \sum_{k=1}^t \eta_k \TK^{1 \over 2}\Pi_{k+1}^t (\TX) N_k\right\|_{\HK},
  \eea
  which leads to the desired result \eref{ideaSample}.
\end{proof}
The above lemma shows that in order to upper bound $\mE[\|\IK \mu_{t} - \IK \nu_{t}\|_{\rho}^2],$ one may only need to bound $\left\| \TK^{1\over 2}\Pi_{k+1}^t (\TX) N_k\right\|_{\HK}.$
A detailed look at this latter term indicates that one may  analyze the terms $\TK^{1\over 2}\Pi_{k+1}^t (\TX)$ and  $N_k$ separately, since $\mE_{\bf z}[N_k] = 0$ and the properties of the deterministic sequence $\{\mu _k\}_k$ have been derived in Section \ref{sec:initial}. Moreover, to exploit the capacity condition from Assumption \ref{as:eigenvalues}, we estimate $\|(\TK+\lambda)^{-{1\over 2}} N_k\|_{\HK}$ (with $\lambda>0$ properly chosen later), rather than
$\|N_k\|_{\HK}$, as follows.
\begin{lemma}\label{lem:lambdaNk}
 Under Assumptions \ref{as:noiseExp}, \ref{as:regularity} and \ref{as:eigenvalues}, let $\{N_t\}_t$ be as in \eref{Nt}. Then for any fixed $\lambda>0,$ and $T\geq 2,$ \\
1) if $\zeta \geq 1/2,$ with probability at least $1-\delta_1,$
the following holds for all $k \in \mN:$
 \be\label{lambdaNk}
 \|(\TK+\lambda)^{-{1\over 2}} N_k\|_{\HK} \leq
 4 (R \kappa^{2\zeta} + \sqrt{M} ) \left( { \kappa   \over m \sqrt{\lambda}} + {\sqrt{ 2\sqrt{v} c_{\gamma} } \over \sqrt{m \lambda^{\gamma}}} \right) \log{4 \over \delta_1}.
 \ee
2) if $\zeta \in ]0,1/2[,$ with probability at least $1-\delta_1,$ the following holds for all $k\in [T]:$
\begin{multline}
 \|(\TK+\lambda)^{-{1\over 2}} N_k\|_{\HK} \leq
  2\left( 3\|\FH\|_{\infty}  + 2\sqrt{M}+ \kappa R\right) \left( { \kappa  \over m \sqrt{\lambda}} + {\sqrt{ 2\sqrt{v} c_{\gamma} } \over \sqrt{m \lambda^{\gamma}}} \right) \log{3T \over \delta_1} \\
 + {2 \kappa^2 R \left(\sum_{i=1}^{k} \eta_i\right)^{{1\over 2} - \zeta } \over m\sqrt{\lambda}} \log {3T \over \delta_1} +  {2{\kappa} R \over \sqrt{m\lambda}} \left( 1 \over  \sum_{i=1}^{k} \eta_i \right)^{\zeta} \log{3T \over \delta_1}.
 \label{lambdaNkNon}\end{multline}
\end{lemma}

\begin{proof}
We will apply Bernstein inequality from Lemma \ref{lem:Bernstein} to prove the result.\\
{\bf Attainable Case: $\zeta \geq 1/2$.} See Appendix \ref{sec:prov_estimates} for the proof.\\
 {\bf Non-attainable case: $0<\zeta <1/2$.} \\
 Let $w_i = (\FH(x_i)-y_i)(\TK + \lambda)^{-1/2}x_i,$ for all $i\in[m].$  Noting that by \eref{frFH}, and taking the expectation with respect to the random variable $(x,y)$ (from the distribution $\rho$),
  $$\mE[\omega] = \mE[(\FH(x)-f_{\rho}(x))(\TK + \lambda)^{-1/2}x] = 0.$$
 Applying H\"{o}lder's inequality, for any $l\geq 2,$
 \bea
 &&\mE[\|w - \mE[w]\|_{\HK}^l] = \mE[\|w\|_{\HK}^l]
 \leq 2^{l-1} \mE[(|\FH(x)|^l + |y|^l)\|(\TK+\lambda)^{-1/2}x\|_{\HK}^l]\\
 && \leq 2^{l-1} \int_X(\|\FH\|_{\infty}^l + \int_Y|y|^l d\rho(y|x))\|(\TK+\lambda)^{-1/2}x\|_{\HK}^ld\rho_X(x).
 \eea
 Using Cauchy-Schwarz's inequality and Assumption \ref{as:noiseExp} which implies,
\be\label{eq:ymoment}
\int_{Y} y^l d\rho(y|x)  \leq  \left(\int_{Y} |y|^{2l} d\rho(y|x)\right)^{1\over 2} \leq \sqrt{l! M^{l} v} \leq l! (\sqrt{M})^{l} \sqrt{  v},
\ee
we get
 \be\label{eq:interm1}
 \mE[\|w - \mE[w]\|_{\HK}^l] \leq 2^{l-1} (\|\FH\|_{\infty}^l + l!(\sqrt{M})^l \sqrt{v}) \int_{X} \|(\TK+\lambda)^{-1/2}x\|_{\HK}^l d\rho_X(x).
 \ee
By Assumption \eref{boundedKernel},
\be\label{eq:int1}\|(\TK+ \lambda I)^{-{1\over 2}} x\|_{\HK} \leq {\|x\|_{\HK} \over \sqrt{\lambda}} \leq {\kappa \over \sqrt{\lambda}}.\ee
Besides, using the fact that $\mE[\|\xi\|_{\HK}^2] = \mE[\tr(\xi \otimes \xi)] = \tr(\mE[\xi \otimes \xi])$
and $\mE[x \otimes x] = \TK,$ we know that
\bea
\int_{X} \|(\TK + \lambda I)^{-{1\over 2}} x\|_{\HK}^2 d\rho_{X}(x) = \tr((\TK+ \lambda I)^{-{1\over 2}} \TK (\TK+ \lambda I)^{-{1\over 2}}) = \tr((\TK+ \lambda I)^{-1} \TK),
\eea
and as a result of the above and Assumption \ref{as:eigenvalues},
\bea
\int_{X} \|(\TK + \lambda I)^{-{1\over 2}} x\|_{\HK}^2 d\rho_{X}(x) \leq c_{\gamma} \lambda^{-\gamma}.
\eea
It thus follows that
\be\label{eq:xmoment}
\int_{X} \|(\TK+\lambda)^{-1/2}x\|_{\HK}^l d\rho_X(x) \leq \left( {\kappa \over \sqrt{\lambda}}\right)^{l-2} \int_{X} \|(\TK+\lambda)^{-1/2}x\|_{\HK}^2 d\rho_X(x) \leq \left( {\kappa \over \sqrt{\lambda}}\right)^{l-2} c_{\gamma} \lambda^{-\gamma}.
\ee
Plugging the above inequality into \eref{eq:interm1},
 \bea
 \mE[\|w - \mE[w]\|_{\HK}^l] &\leq& 2^{l-1} (\|\FH\|_{\infty}^l + l!(\sqrt{M})^l \sqrt{v}) \left({\kappa \over \sqrt{\lambda}}\right)^{l-2}  c_{\gamma} \lambda^{-\gamma}\\
 &\leq& {1 \over 2} l! \left({2\kappa(\|\FH\|_{\infty}+\sqrt{M}) \over \sqrt{\lambda}}\right)^{l-2} 4c_{\gamma}\sqrt{v}(\|\FH\|_{\infty}+\sqrt{M})^2\lambda^{-\gamma}.
 \eea
Therefore, using Lemma \ref{lem:Bernstein}, we get that with probability at least $1-\delta,$
 \begin{align}\label{eq:int3}
\left\| (\TK + \lambda)^{-{1\over 2}}{1\over m} \sum_{i=1}^{m}\left(\FH(x_i) - y_i\right)x_i \right\|_{\HK}
=& \left\| {1 \over m} \sum_{i=1}^m (\mE[w_i] - w_i) \right\|_{\HK}\nonumber \\
\leq& 4 (\sqrt{M} + \|\FH\|_{\infty}) \left( { \kappa   \over m \sqrt{\lambda}} + {\sqrt{ \sqrt{v} c_{\gamma} } \over \sqrt{m \lambda^{\gamma}}} \right) \log{2 \over \delta}.
\end{align}
We next let $\xi_i = (\TK+\lambda)^{-1/2}(\la \mu_k,x_i\ra - \FH(x_i))x_i,$ for all $i\in [m]$. We assume that $k \geq 2$. (The proof for the case $k=1$ is simpler as $\mu_1=0$.)
It is easy to see that the expectation of each $\xi_i$ with respect to the random variable $(x_i,y_i)$ is 
\bea
\mE[\xi]=(\TK+\lambda)^{-1/2}(\TK\mu_k - \IK^*\FH) = (\TK+\lambda)^{-1/2}(\TK\mu_k - \IK^*f_{\rho}),
\eea
and
\bea
\|\xi\|_{\HK} \leq  (\|\IK \mu_k\|_{\infty} + \|\FH\|_{\infty}) \|(\TK+\lambda)^{-1/2}x\|_{\HK}.
\eea
 By Assumption \eref{boundedKernel}, $\|\IK \mu_k\|_{\infty} \leq \kappa \|\mu_k\|_{\HK}$.
It thus follows from the above and \eref{eq:int1} that
\bea
\|\xi\|_{\HK} \leq  (\kappa\|\mu_k\|_{\HK} + \|\FH\|_{\infty}) {\kappa \over \sqrt{\lambda}}.
\eea
Besides,
\bea
\mE\|\xi\|_{\HK}^2 \leq {\kappa^2 \over \lambda} \mE(\mu_k(x) - \FH(x))^2 = {\kappa^2 \over \lambda} \|\IK\mu_{k} - \FH\|_{\rho}^2 \leq {\kappa^2R^2 \over \lambda} \left( \zeta \over 2 \sum_{i=1}^{k-1} \eta_i \right)^{2\zeta} \leq {\kappa^2R^2 \over \lambda} \left( 1 \over \sum_{i=1}^{k} \eta_i \right)^{2\zeta},
\eea
where for the last inequality, we used \eref{initialErrA}. Applying Lemma \ref{lem:Bernstein} and \eref{htInftyNon}, we get that with probability at least $1-\delta,$
\bea
&&\left\|(\TK+\lambda)^{-1/2}[{1\over m}\sum_{i=1}^m(\mu_k(x_i) - \FH(x_i))x_i - (\TK\mu_k - \IK^*f_{\rho})]\right\|_{\HK} \\
&\leq& 2{\kappa} \left( {\kappa\|\mu_{k}\|_{\HK} + \|\FH\|_{\infty} \over m\sqrt{\lambda}} + {R \over \sqrt{m\lambda}} \left( 1 \over  \sum_{i=1}^{k} \eta_i \right)^{\zeta} \right) \log{2 \over \delta}\\
&\leq& 2{\kappa} \left( {\kappa R + \|\FH\|_{\infty} \over m\sqrt{\lambda}}+ {\kappa R \left( \sum_{i=1}^{k} \eta_i\right)^{{1\over 2} - \zeta} \over m\sqrt{\lambda}} + {R \over \sqrt{m\lambda}} \left( 1 \over  \sum_{i=1}^{k} \eta_i \right)^{\zeta} \right) \log{2 \over \delta}.
\eea
Introducing the above estimate and \eref{eq:int3} into the following inequality
\begin{multline*}
\|(\TK+\lambda)^{-1/2} N_k\|_{\HK} \leq \left\| (\TK + \lambda)^{-{1\over 2}}{1\over m} \sum_{i=1}^{m}\left(\FH(x_i) - y_i\right)x_i \right\|_{\HK} \\
+ \left\|(\TK+\lambda)^{-1/2}[{1\over m}\sum_{i=1}^m(\mu_k(x_i) - \FH(x_i))x_i - (\TK\mu_k - \IK^*f_{\rho})]\right\|_{\HK},
\end{multline*}
and then substituting with \eref{htInftyNon}, by a simple calculation, one can prove the desired result by scaling $\delta.$
\end{proof}

The next lemma is from \cite{rudi2015less}, and is derived applying a recent Bernstein inequality from \citep{tropp2012user,minsker2011some} for a sum of random operators.
\begin{lemma}\label{lem:differOperator}
  Let $\delta_2 \in (0,1)$ and ${9\kappa^2 \over m} \log {m \over \delta_2} \leq \lambda \leq \|\TK\|.$ Then
  the following holds with probability at least $1 - \delta_2,$
  \be\label{differOperator}
  \|(\TX + \lambda I)^{-{1\over 2}} \TK^{1\over 2} \| \leq \|(\TX + \lambda I)^{-{1\over 2}} (\TK + \lambda I)^{1\over 2} \| \leq 2.
  \ee
\end{lemma}
Now we are in a position to estimate the sample variance.
\begin{pro}\label{pro:sampleErrA}
Under Assumptions \ref{as:noiseExp}, \ref{as:regularity} and \ref{as:eigenvalues}, let $\eta_1 \kappa^2 \leq 1$ and $0<\lambda \leq \|\TK\|$. Assume that \eref{differOperator} holds. Then for all $t\in  [T]:$\\
1) if $\zeta \geq 1/2,$ and \eref{lambdaNk} hold, then for $t\in \mN,$
\be\label{gSampleErr}
\begin{split}
&\|\IK \nu_{t+1} - \IK \mu_{t+1}\|_{\rho} \\
\leq& 4 (R \kappa^{2\zeta} + \sqrt{M} ) \left( { \kappa   \over m \sqrt{\lambda}} + {\sqrt{ 2\sqrt{v} c_{\gamma} } \over \sqrt{m \lambda^{\gamma}}} \right) \left( \sum_{k=1}^{t-1} { 2\eta_k  \over \sum_{i=k+1}^t \eta_i}  + 4 \lambda \sum_{k=1}^{t-1} \eta_k + \sqrt{2} \kappa^2 \eta_t\right) \log {4 \over \delta_1}.
\end{split}\ee
2) if $\zeta < 1/2,$ and \eref{lambdaNkNon} hold for any $t\in [T]$, then for $t\in [T]:$
\begin{multline}\label{gSampleErrNon}
\|\IK \nu_{t+1} - \IK \mu_{t+1}\|_{\rho}  \leq \left( \sum_{k=1}^{t-1} { 2\eta_k  \over \sum_{i=k+1}^t \eta_i}  + 4\lambda \sum_{k=1}^{t-1} \eta_k + \sqrt{2} \kappa^2 \eta_t\right) \\
 \times \left(2\left( 3\|\FH\|_{\infty}  + 3\sqrt{M}+ \kappa R \right)  \left( { \kappa   \over m \sqrt{\lambda}} + {\sqrt{ 2\sqrt{v} c_{\gamma} } \over \sqrt{m \lambda^{\gamma}}} \right) + {2\kappa^2 R \left(\sum_{i=1}^t \eta_i\right)^{{1\over 2} - \zeta }   \over m \sqrt{\lambda}} \right)   \log {3T \over \delta_1} \\
+ {2{\kappa} R \over \sqrt{m\lambda}} \log{3T \over \delta_1} \left(\sum_{k=1}^{t-1} { 2 \eta_k  \over \left(\sum_{i=1}^k \eta_i \right)^{\zeta} \sum_{i=k+1}^t \eta_i}  + 4 \lambda \sum_{k=1}^{t-1} {\eta_k \over \left(\sum_{i=1}^k \eta_i \right)^{\zeta}} + {\sqrt{2} \kappa^2 \eta_t \over \left(\sum_{i=1}^t \eta_i \right)^{\zeta}}\right).
\end{multline}
\end{pro}
\begin{proof}
  For notational simplicity, we let $\TKL = \TK + \lambda I$ and $\TXL = \TX + \lambda I.$
   Note that by Lemma \ref{lemma:ideaSample}, we have \eref{ideaSample}.
  When $k\in [t-1]$, by rewriting $\TK^{1 \over 2}\Pi_{k+1}^t (\TX) N_k$ as
  \bea
  \TK^{1\over 2} \TXL^{-{1\over 2} } \TXL^{1\over 2} \Pi_{k+1}^t (\TX) \TXL^{1\over 2} \TXL^{- {1\over 2}}   \TKL^{1\over 2} \TKL^{-{1\over 2}} N_k,
  \eea
  we can upper bound $\|\TK^{1 \over 2}\Pi_{k+1}^t (\TX) N_k\|_{\HK}$ as
  \bea
  \|\TK^{1 \over 2} \Pi_{k+1}^t(\TX) N_k\|_{\HK} \leq \|\TK^{1\over 2} \TXL^{-{1\over 2} } \| \|\TXL^{1\over 2} \Pi_{k+1}^t (\TX) \TXL^{1\over 2}\| \| \TXL^{- {1\over 2}}  \TKL^{1\over 2}\| \| \TKL^{-{1\over 2}} N_k\|_{\HK}.
  \eea
  Applying \eref{differOperator}, the above can be relaxed as
  \bea
  \|\TK^{1 \over 2} \Pi_{k+1}^t(\TX) N_k\|_{\HK} \leq 4 \|\TXL^{1\over 2} \Pi_{k+1}^t(\TX) \TXL^{1\over 2}\| \| \TKL^{-{1\over 2}} N_k\|_{\HK},
  \eea
  which is equivalent to
  \bea
  \|\TKL^{1 \over 2}\Pi_{k+1}^t(\TX) N_k\|_{\HK} \leq 4 \|\TXL \Pi_{k+1}^t(\TX)\| \| \TKL^{-{1\over 2}} N_k\|_{\HK}.
  \eea
Thus, following from $\eta_k \kappa^2 \leq 1$ which implies $\eta_k \|\TX\| \leq 1,$
  \bea
  \|\TXL \Pi_{k+1}^t (\TX)\|   &\leq& \|\TX \Pi_{k+1}^t (\TX)\| + \|\lambda \Pi_{k+1}^t (\TX)\|
  \\ &\leq& \|\TX \Pi_{k+1}^t (\TX)\| + \lambda.
  \eea
Applying Lemma \ref{lemma:initialerror} with $\zeta=1$ to bound $\|\TX \Pi_{k+1}^t (\TX)\|$, we get
  \bea
  \|\TXL \Pi_{k+1}^t (\TX)\|  \leq   { 1 \over \mathrm{e} \sum_{j=k+1}^t \eta_j}  + \lambda.
  \eea
When $k=t$,
  \bea
  \|\TK^{1 \over 2} \Pi_{k+1}^t (\TX) N_k\|_{\HK} = \|\TK^{1 \over 2} N_t\|_{\HK} \leq \|\TK^{1 \over 2}\| \| \TKL^{1\over 2}\| \| \TKL^{-{1\over 2}} N_t\|_{\HK} \\
  \leq \|\TK\|^{1 \over 2} ( \|\TK\| + \lambda)^{1\over 2} \| \TKL^{-{1\over 2}} N_t\|_{\HK}.
  \eea
  Since $\lambda \leq \|\TK\| \leq \tr(\TK) \leq \kappa^2,$ we derive
  \bea
  \|\TK^{1 \over 2} \Pi_{k+1}^t (\TX) N_t\|_{\HK} \leq \sqrt{2} \kappa^2 \| \TKL^{-{1\over 2}} N_t\|_{\HK}.
  \eea
From the above analysis, we see that $\sum_{k=1}^t \eta_k \left\| \TK^{1\over 2}\Pi_{k+1}^t (\TX) N_k\right\|_{\HK}$ can be upper bounded by
  \bea
   \leq  \left( \sum_{k=1}^{t-1} { \eta_k/2 \| \TKL^{-{1\over 2}} N_k\|_{\HK}  \over \sum_{i=k+1}^t \eta_i}  + \lambda \sum_{k=1}^{t-1} \eta_k\| \TKL^{-{1\over 2}} N_k\|_{\HK} + \sqrt{2} \kappa^2 \eta_t\| \TKL^{-{1\over 2}} N_t\|_{\HK} \right).
  \eea
Plugging \eref{lambdaNk} (or \eref{lambdaNkNon}) into the above, and then combining with \eref{ideaSample}, we get the desired bound \eref{gSampleErr} (or \eref{gSampleErrNon}).
The proof is complete.
\end{proof}
Setting $\eta_t = \eta_1 t^{-\theta}$ in the above proposition, with the basic estimates from Section \ref{sec:estimates}, we get the following explicit bounds for the sample variance.
\begin{pro}
  \label{pro:sampleErrB}
 Under Assumptions \ref{as:noiseExp}, \ref{as:regularity} and \ref{as:eigenvalues}, let $\eta_t = \eta_1 t^{-\theta}$   with $\eta_1 \in ]0,\kappa^{-2}]$ and $\theta \in [0,1[.$
   Assume that \eref{differOperator} holds.
Then the following holds for all $t\in  [T]$ and any $0<\lambda \leq \|\TK\|$:\\
1) If $\zeta \geq 1/2,$ and \eref{lambdaNk} holds for all $t \in [T],$
\be\label{gSampleErrB}
\begin{split}
&\|\IK \nu_{t+1} - \IK \mu_{t+1}\|_{\rho} \\
\leq&  4 (R \kappa^{2\zeta} + \sqrt{M} ) \left( {8\lambda\eta_1 t^{1-\theta}\over 1-\theta} + 4\log t + 4+ \sqrt{2}\eta_1 \kappa^2 \right) \left( { \kappa   \over m \sqrt{\lambda}} + {\sqrt{ 2\sqrt{v} c_{\gamma} } \over \sqrt{m \lambda^{\gamma}}} \right) \log {4 \over \delta_1}.
\end{split}\ee
2) If $\zeta < 1/2,$ and \eref{lambdaNkNon} holds for all $t \in [T],$
\begin{multline}\label{gSampleErrBNon}
\|\IK \nu_{t+1} - \IK \mu_{t+1}\|_{\rho}  \leq \left( {8\lambda\eta_1 t^{1-\theta}\over 1-\theta} + 4\log t + 4+ \sqrt{2}\eta_1 \kappa^2 \right) \log {3T \over \delta_1}\\
 \times \left(2\left( 3\|\FH\|_{\infty}  + 3\sqrt{M}+ \kappa R\right)  \left( { \kappa \over m \sqrt{\lambda}} + {\sqrt{ 2\sqrt{v} c_{\gamma} } \over \sqrt{m \lambda^{\gamma}}} \right) + \left( {\kappa \over 1-\theta} \sqrt{ {\eta_1 t^{1-\theta} \over m}}
+ 1 \right){4{\kappa} R \over \sqrt{m\lambda}} {1\over (\eta_1t^{1-\theta})^{\zeta}} \right).
\end{multline}
\end{pro}
\begin{proof}
  By Proposition \ref{pro:sampleErrA}, we have \eref{gSampleErr} or \eref{gSampleErrNon}.
  Note that
  \bea
  \sum_{k=1}^{t-1} { \eta_k  \over \sum_{i=k+1}^t \eta_i} = \sum_{k=1}^{t-1} { k^{-\theta} \over \sum_{i=k+1}^t i^{-\theta}} \leq \sum_{k=1}^{t-1} { k^{-\theta} \over (t-k) t^{-\theta}}.
  \eea
  Applying Lemma \ref{lem:estimate2}, we get
  \bea
  \sum_{k=1}^{t-1} { \eta_k  \over \sum_{i=k+1}^t \eta_i} \leq 2+ 2\log t,
  \eea
  and by Lemma \ref{lem:estimate1},
  \bea
  \sum_{k=1}^{t-1}\eta_k = \eta_1 \sum_{k=1}^{t-1} k^{-\theta} \leq {2 \eta_1 t^{1-\theta} \over 1-\theta}.
  \eea
  Introducing the last two estimates into \eref{gSampleErr} and \eref{gSampleErrNon}, one can get \eref{gSampleErrB} and that
  \begin{multline*}
\|\IK \nu_{t+1} - \IK \mu_{t+1}\|_{\rho}  \leq \left( {8\lambda\eta_1 t^{1-\theta}\over 1-\theta} + 4 \log t + 4 + \sqrt{2}\eta_1 \kappa^2 \right) \\
 \times \left(2\left( 3\|\FH\|_{\infty}  + 3\sqrt{M}+ \kappa R\right)  \left( { \kappa \over m \sqrt{\lambda}} + {\sqrt{ 2\sqrt{v} c_{\gamma} } \over \sqrt{m \lambda^{\gamma}}} \right) +  {4\kappa^2 R\left(\eta_1 t^{1-\theta}\right)^{{1\over 2} - \zeta} \over (1-\theta) m \sqrt{\lambda}} \right) \log {3T \over \delta_1} \\
+ {2{\kappa} R \over \sqrt{m\lambda}} \log{3T \over \delta_1} \left(\sum_{k=1}^{t-1} { 2\eta_k  \over \left(\sum_{i=1}^k \eta_i \right)^{\zeta} \sum_{i=k+1}^t \eta_i}  + 4\lambda \sum_{k=1}^{t-1} {\eta_k \over \left(\sum_{i=1}^k \eta_i \right)^{\zeta}} + {\sqrt{2} \kappa^2 \eta_t \over \left(\sum_{i=1}^t \eta_i \right)^{\zeta}}\right).
\end{multline*}
To prove \eref{gSampleErrBNon}, it remains to estimate the last term of the above. Again, using Lemmas \ref{lem:estimate1}, \ref{lem:estimate1a} and \ref{lem:estimate2}, we get
\bea
&&\sum_{k=1}^{t-1} { \eta_k  \over \left(\sum_{i=1}^k \eta_i \right)^{\zeta} \sum_{i=k+1}^t \eta_i}  =  {1 \over \eta_1^{\zeta}}\sum_{k=1}^{t-1} {  k^{-\theta}  \over \left(\sum_{i=1}^k i^{-\theta} \right)^{\zeta} \sum_{i=k+1}^t i^{-\theta}} \\
&\leq& {1 \over \eta_1^{\zeta}} \sum_{k=1}^{t-1} {k^{-\theta}  \over (k^{1-\theta}/2)^\zeta (t-k)  t^{-\theta}} = {2^{\zeta} \over \eta_1^{\zeta}}t^{\theta} \sum_{k=1}^{t-1} {k^{-(\theta + \zeta(1-\theta))}  \over t-k }\\
&\leq& {2^{\zeta} \over \eta_1^{\zeta}} t^{\theta} 2 t^{-(\theta+\zeta(1-\theta))} (1 + \log t) \leq {4 (1 + \log t) \over (\eta_1t^{1-\theta})^{\zeta}},
\eea
\bea
\sum_{k=1}^{t-1} {\eta_k \over \left(\sum_{i=1}^k \eta_i \right)^{\zeta}} = \eta_1^{1-\zeta} \sum_{k=1}^{t-1} {k^{-\theta} \over \left(\sum_{i=1}^k i^{-\theta} \right)^{\zeta}} \leq 2^{\zeta} \eta_1^{1-\zeta} \sum_{k=1}^{t-1} k^{-(\theta+\zeta(1-\theta))} \leq {2(\eta_1t^{1-\theta})^{1-\zeta} \over (1-\theta)}, \quad \mbox{and}
\eea
\bea
{ \eta_t \over \left(\sum_{i=1}^t \eta_i \right)^{\zeta}} =  { \eta_1 t^{-\theta} \over \left(\sum_{i=1}^t \eta_1 i^{-\theta} \right)^{\zeta}} \leq 2^{\zeta} {\eta_1 t^{-\theta} \over (\eta_1 t^{1-\theta})^{\zeta}} \leq {\sqrt{2}\eta_1 \over (\eta_1 t^{1-\theta})^{\zeta}}.
\eea
Therefore,
  \begin{multline*}
\|\IK \nu_{t+1} - \IK \mu_{t+1}\|_{\rho}  \leq \left( {8\lambda\eta_1 t^{1-\theta}\over 1-\theta} + 4 \log t + 4+ \sqrt{2}\eta_1 \kappa^2 \right) \\
 \times \left(2\left( 3\|\FH\|_{\infty}  + 3\sqrt{M}+ \kappa R\right)  \left( { \kappa \over m \sqrt{\lambda}} + {\sqrt{ 2\sqrt{v} c_{\gamma} } \over \sqrt{m \lambda^{\gamma}}} \right) +  {4\kappa^2R\left(\eta_1 t^{1-\theta}\right)^{{1\over 2} - \zeta} \over (1-\theta) m \sqrt{\lambda}} \right) \log {3T \over \delta_1} \\
+ {2{\kappa} R \over \sqrt{m\lambda}} \log{3T \over \delta_1} \left(8 + 8\log t + {8\lambda\eta_1t^{1-\theta} \over 1-\theta} + 2\kappa^2 \eta_1\right) {1\over (\eta_1t^{1-\theta})^{\zeta}}.
\end{multline*}
Rearranging terms, we can prove the second part.
\end{proof}
In conclusion, we get the following result for the sample variance.
\begin{thm}\label{thm:sampleErr}
  Under Assumptions \ref{as:noiseExp}, \ref{as:regularity} and \ref{as:eigenvalues}, let $\delta_1, \delta_2 \in]0,1[$ and ${9\kappa^2 \over m} \log {m \over \delta_2} \leq \lambda \leq \|\TK\|.$ Let $\eta_t = \eta_1 t^{-\theta}$ for all $t \in [T],$ with $\eta_1 \in ]0,\kappa^{-2}]$ and $\theta \in [0,1[.$
  Then with probability at least $1 - \delta_1 - \delta_2,$ the following holds for all $t\in  [T]:$\\
 1) if $\zeta \geq 1/2$, we have \eref{gSampleErrB}. \\
 2) if $\zeta <1/2,$ we have \eref{gSampleErrBNon}.
\end{thm}

\section{Estimating Computational Variance}\label{sec:comp}
In this section, we estimate the computational variance, $\mE[\|\IK \omega_t - \IK \nu_t\|_{\rho}^2]$.
For this, a series of lemmas is introduced.

\subsection{Cumulative Error}
We have the following lemma, which shows that the computational variance can be controlled by a sum of weighted empirical risks.
\begin{lemma}\label{lem:cul_err} We have
\be\label{eq:cul_err}
\mE_{\bf J}\|\IK \omega_{t+1} - \IK \nu_{t+1}\|_{\rho}^2
\leq {\kappa^2 \over b} \sum_{k=1}^t \eta_{k}^2 \left\|\TK^{1\over 2}\Pi^t_{k+1}(\TX)\right\|^2  \mE_{\J}[\mcE_{\bf z}(\omega_k)].
\ee
\end{lemma}
\begin{proof}
  Since $\omega_{t+1}$ and $\nu_{t+1}$ are given by \eref{Alg} and \eref{Alg2}, respectively,
\bea
 \omega_{t+1} - \nu_{t+1} &=& (\omega_t - \nu_t)  + \eta_t \left\{ (\TX \nu_t - \SX^*{\mathbf y} ) - {1\over b} \sum_{i=b(t-1)+1}^{bt} (\la \omega_t, x_{j_i} \ra_{\HK} - y_{j_i}) x_{j_i}\right\} \\
 &=& (I - \eta_t \TX) (\omega_t - \nu_t) +  {\eta_t \over b} \sum_{i=b(t-1)+1}^{bt} \left\{ (\TX \omega_t - \SX^*{\mathbf y} ) - (\la \omega_t, x_{j_i} \ra_{\HK} - y_{j_i}) x_{j_i}\right\}.
\eea
Applying this relationship iteratively,
\bea
\omega_{t+1} - \nu_{t+1} = \Pi^t_1(\TX) (\omega_1 - \nu_1) + {1 \over b} \sum_{k=1}^t \sum_{i=b(k-1)+1}^{bk} \eta_k \Pi^t_{k+1}(\TX) M_{k,i},
\eea
where we denote
\be\label{Mk}
M_{k,i} =  (\TX \omega_k - \SX^*{\mathbf y} ) - ( \la \omega_{k}, x_{j_i} \ra_{\HK} - y_{j_i}) x_{j_i}.
\ee
Since $\omega_1 = \nu_1=0,$ then
\bea
\omega_{t+1} - \nu_{t+1} = {1 \over b} \sum_{k=1}^t \sum_{i=b(k-1)+1}^{bk} \eta_k \Pi^t_{k+1}(\TX) M_{k,i}.
\eea
Therefore,
\begin{eqnarray}
\mE_{\J}\|\IK \omega_{t+1} - \IK\nu_{t+1}\|_{\rho}^2 &=& {1 \over b^2} \mE_{\J} \left\| \sum_{k=1}^t \sum_{i=b(k-1)+1}^{bk} \eta_k\IK \Pi^t_{k+1}(\TX) M_{k,i} \right\|_{\rho}^2 \nonumber \\
&=&  {1 \over b^2}\sum_{k=1}^t \sum_{i=b(k-1)+1}^{bk} \eta_{k}^2 \mE_{\J} \left\|\IK\Pi^t_{k+1}(\TX) M_{k,i} \right\|_{\rho}^2, \label{eq1}
\end{eqnarray}
where for the last equality, we use the fact that if $k\neq k',$ or $k=k'$ but $i\neq i'$\footnote{This is possible only when $b \geq 2$.}, then
\bea
\mE_{\J} \la \IK \Pi^t_{k+1}(\TX) M_{k,i}, \IK\Pi^t_{k'+1}(\TX) M_{k',i'} \ra_{\rho} = 0.
\eea
Indeed, if $k\neq k',$ without loss of generality, we consider the case $k< k'.$ Recalling that $M_{k,i}$ is given by \eref{Mk} and that given any $\bf z$, $\omega_{k}$ is depending only on $\J_1,\cdots,\J_{k-1},$
we thus have
\bea
&&\mE_{\J} \la \IK \Pi^t_{k+1}(\TX) M_{k,i} , \IK \Pi^t_{k'+1}(\TX) M_{k',i'} \ra_{\rho} \\
 &&= \mE_{\J_1,\cdots,\J_{k'-1}} \la \IK\Pi^t_{k+1}(\TX) M_{k,i} , \IK\Pi^t_{k'+1}(\TX) \mE_{\J_{k'}}[M_{k',i'}] \ra_{\rho} = 0.
\eea
If $k=k'$ but $i\neq i',$
 without loss of generality, we assume $i<i'.$
By noting that $\omega_{k}$ is depending only on $\J_1,\cdots,\J_{k-1}$ and $M_{k,i}$ is depending only on $\omega_{k}$ and $z_{j_i}$ (given any sample ${\bf z}$),
\bea
&&\mE_{\J} \la \IK\Pi^t_{k+1}(\TX) M_{k,i} , \IK\Pi^t_{k+1}(\TX) M_{k,i'} \ra_{\rho} \\
 &&= \mE_{\J_1,\cdots,\J_{k-1}} \la \IK\Pi^t_{k+1}(\TX) \mE_{j_i}[ M_{k,i}] , \IK\Pi^t_{k'+1}(\TX) \mE_{j_{i'}}[M_{k,i'}] \ra_{\rho} = 0.
\eea
Using the isometry property \eref{isometry} to \eref{eq1},
\bea
 \mE_{\J}\left\|\IK\Pi^t_{k+1}(\TX) M_{k,i} \right\|_{\rho}^2 =  \mE_{\J} \left\|\TK^{1\over 2}\Pi^t_{k+1}(\TX) M_{k,i} \right\|_{\HK}^2 \leq   \left\|\TK^{1 \over 2}\Pi^t_{k+1}(\TX)\right\|^2 \mE_{\J} \left\|M_{k,i} \right\|_{\HK}^2,
\eea
and by applying the inequality $\mE[\|\xi - \mE[\xi]\|_{\HK}^2] \leq \mE[\|\xi\|_{\HK}^2]$,
\bea
\mE_{\J} \left\|M_{k,i} \right\|_{\HK}^2 \leq \mE_{\J} \left \| (\la \omega_k, x_{j_i} \ra_{\HK} - y_{j_i}) x_{j_i}\right\|_{\HK}^2
\leq  \kappa^2 \mE_{\J}[(\la \omega_k, x_{j_i} \ra_{\HK} - y_{j_i})^2 ] = \kappa^2 \mE_{\J}[\mcE_{\bf z}(\omega_k)],
\eea
where for the last inequality we use \eref{boundedKernel}. Therefore, we can get the desired result.
\end{proof}

To estimate the computational variance from \eref{eq:cul_err}, we need to further develop upper bounds for the empirical risks and the weighted factors, which will be given in the following two subsections.

\subsection{Bounding the Empirical Risk}
This subsection is devoted to upper bounding $\mE_{\bf J}[\mcE_{\bf z}(\omega_l)]$. The process relies on some tools from convex analysis and a decomposition related to the weighted averages and the last iterates from \citep{shamir2013stochastic,lin2015iterative}. We begin by introducing the following lemma, a fact based on the square loss' special properties.
\begin{lemma}\label{lemma:mfejer} Given any sample $\bf z,$ and $l\in \mN$, let $\omega\in \HK$ be independent from $\J_l$, then
\be\label{mfejer}
\eta_l \left( \mcE_{\bf z}(\omega_l) - \mcE_{\bf z}(\omega) \right)
\leq  \|\omega_{l}- \omega\|_{\HK}^2 - \mE_{\J_l}\|\omega_{l+1}- \omega\|_{\HK}^2 + \eta_l^2 \kappa^2 \mcE_{\bf z}(\omega_l).
\ee
\end{lemma}
\begin{proof}
Since $\omega_{t+1}$ is given be \eref{Alg}, subtracting both sides of \eref{Alg} by $\omega$, taking the square $\HK$-norm, and expanding the inner product,
\bea
\|\omega_{l+1} - \omega\|_{\HK}^2 = \|\omega_{l} - \omega\|_{\HK}^2  + {\eta_l^2 \over b^2}  \left\| \sum_{i=b(l-1)+1}^{bl} (\la \omega_l, x_{j_i} \ra_{\HK} - y_{j_i}) x_{j_i} \right\|_{\HK}^2  \\
+ {2\eta_l \over b} \sum_{i=b(l-1)+1}^{bl} (\la \omega_l,  x_{j_i}\ra_{\HK} - y_{j_i}) \la \omega - \omega_l, x_{j_i} \ra_{\HK}.
\eea
By Assumption \eref{boundedKernel}, $\|x_{j_i}\|_{\HK} \leq \kappa$, and thus
\bea
\left\| \sum_{i=b(l-1)+1}^{bl} (\la \omega_l, x_{j_i} \ra_{\HK} - y_{j_i}) x_{j_i} \right\|_{\HK}^2
&\leq& \left(\sum_{i=b(l-1)+1}^{bl} |\la \omega_l,  x_{j_i} \ra_{\HK} - y_{j_i}| \kappa\right)^2 \\
&\leq& \kappa^2 b \sum_{i=b(l-1)+1}^{bl} (\la \omega_l, x_{j_i} \ra_{\HK} - y_{j_i})^2,
\eea
where for the last inequality, we used Cauchy-Schwarz inequality.
Thus,
\bea
\|\omega_{l+1}- \omega\|_{\HK}^2
 \leq \|\omega_{l} - \omega \|_{\HK}^2  +  { \eta_l^2 \kappa^2 \over b} \sum_{i=b(l-1)+1}^{bl} (\la \omega_l, x_{j_i}\ra_{\HK} - y_{j_i})^2  \\
 + {2\eta_l \over b} \sum_{i=b(l-1)+1}^{bl} (\la \omega_l, x_{j_i} \ra_{\HK} - y_{j_i}) (\la \omega, x_{j_i} \ra_{\HK}- \la \omega_l, x_{j_i} \ra_{\HK}).
\eea
Using the basic inequality $a(b-a) \leq (b^2 - a^2)/2,\forall a,b \in \mR,$
\bea
\|\omega_{l+1}- \omega\|_{\HK}^2
\leq  \|\omega_{l}- \omega\|_{\HK}^2  + { \eta_l^2\kappa^2 \over b} \sum_{i=b(l-1)+1}^{bl} (\la \omega_l, x_{j_i} \ra_{\HK} - y_{j_i})^2 \\
  + {\eta_l \over b} \sum_{i=b(l-1)+1}^{bl} \left( (\la \omega, x_{j_i} \ra_{\HK} - y_{j_i})^2 - (\la \omega_l, x_{j_i} \ra_{\HK} - y_{j_i})^2 \right).
\eea
Noting that $\omega_{l}$ and $\omega$ are independent from $\J_l$, and taking the expectation on both sides with respect to $\J_l,$
\bea
\mE_{\J_l}\|\omega_{l+1}- \omega\|_{\HK}^2
\leq  \|\omega_{l}- \omega\|_{\HK}^2  + \eta_l^2 \kappa^2 \mcE_{\bf z}(\omega_l)  + \eta_l \left( \mcE_{\bf z}(\omega) - \mcE_{\bf z}(\omega_l) \right),
\eea
which leads to the desired result by rearranging terms. The proof is complete.
\end{proof}
Using the above lemma and a decomposition related to the weighted averages and the last iterates from \citep{shamir2013stochastic,lin2015iterative}, we can prove the following relationship.
\begin{lemma}\label{lemma:empiricalRelat}
  Let $\eta_1 \kappa^2 \leq 1/2$ for all $t \in \mN.$ Then
  \be\label{empiricalRelat}
   \eta_t \mE_{\J} [\mcE_{\bf z}(\omega_t)]  \leq 4\mcE_{\bf z}(0) {1 \over t} \sum_{l=1}^t \eta_l + 2\kappa^2  \sum_{k=1}^{t-1}{1 \over k(k+1)}  \sum_{i=t-k}^{t-1} \eta_i^2\mE_{\J} [\mcE_{\bf z}(\omega_i)].
\ee
\end{lemma}
\begin{proof}
 For $k=1, \cdots, t-1$,
  \bea
  && {1 \over k} \sum_{i=t-k+1}^{t}  \eta_i \mE_{\J}[\mcE_{\bf z}(\omega_i)] - {1 \over k+1} \sum_{i=t-k}^t \eta_i \mE_{\J}[\mcE_{\bf z}(\omega_i)] \\
   &=& {1 \over k(k+1)} \left\{ (k+1)\sum_{i=t-k+1}^{t} \eta_i \mE_{\J}[\mcE_{\bf z}(\omega_i)] - k \sum_{i=t-k}^t \eta_i \mE_{\J}[\mcE_{\bf z}(\omega_i)] \right\}\\
 & =& {1 \over k(k+1)} \sum_{i=t-k+1}^{t} (\eta_i \mE_{\J}[\mcE_{\bf z}(\omega_i)] -  \eta_{t-k} \mE_{\J}[\mcE_{\bf z}(\omega_{t-k})]) .
  \eea
  Summing over $k=1, \cdots, t-1$, and rearranging terms, we get \citep{lin2015iterative}
\bea
   \eta_t \mE_{\J}[\mcE_{\bf z}(\omega_t)]  = {1 \over t} \sum_{i=1}^t \eta_i \mE_{\J} [\mcE_{\bf z}(\omega_i)] + \sum_{k=1}^{t-1} {1 \over k(k+1)} \sum_{i=t-k+1}^{t} (\eta_i \mE_{\J} [\mcE_{\bf z}(\omega_i)] -  \eta_{t-k} \mE_{\J}[\mcE_{\bf z}(\omega_{t-k})]) .
\eea
Since $\{\eta_t\}_t$ is decreasing and $\mE_{\J}[\mcE_{\bf z}(\omega_{t-k})]$ is non-negative, the above can be relaxed as
\be\label{decomposition}
   \eta_t \mE_{\J} [\mcE_{\bf z}(\omega_t)]  \leq {1 \over t} \sum_{i=1}^t \eta_i \mE_{\J}[\mcE_{\bf z}(\omega_i)] + \sum_{k=1}^{t-1} {1 \over k(k+1)} \sum_{i=t-k+1}^{t} \eta_i  \mE_{\J}[\mcE_{\bf z}(\omega_i) -  \mcE_{\bf z}(\omega_{t-k})] .
\ee
In the rest of the proof, we will upper bound the last two terms of the above.

To bound the first term of the right side of \eref{decomposition}, we apply Lemma \ref{lemma:mfejer} with $\omega=0$ to get
\bea
\eta_l \mE_{\J} \left( \mcE_{\bf z}(\omega_l) - \mcE_{\bf z}(0) \right)
\leq  \mE_{\J} [\|\omega_{l}\|_{\HK}^2 - \|\omega_{l+1}\|_{\HK}^2] + \eta_l^2 \kappa^2 \mE_{\J}[\mcE_{\bf z}(\omega_l)].
\eea
Rearranging terms,
\bea
\eta_l (1 - \eta_l \kappa^2) \mE_{\J}[ \mcE_{\bf z}(\omega_l)]
\leq  \mE_{\J} [\|\omega_{l}\|_{\HK}^2 - \|\omega_{l+1}\|_{\HK}^2] + \eta_l \mcE_{\bf z}(0).
\eea
 It thus follows from the above and  $\eta_l \kappa^2 \leq 1/2$ that
\bea
\eta_l \mE_{\J}[\mcE_{\bf z}(\omega_l)]/2  \leq \mE_{\J}[\|\omega_{l}\|_{\HK}^2  - \|\omega_{l+1}\|_{\HK}^2]    + \eta_l \mcE_{\bf z}(0).
\eea
Summing up over $l=1,\cdots,t,$
\bea
\sum_{l=1}^t \eta_l \mE_{\J}[\mcE_{\bf z}(\omega_l)]/2  \leq  \mE_{\J}[\|w_1\|_{\HK}^2  - \|\omega_{t+1}\|_{\HK}^2]    + \mcE_{\bf z}(0) \sum_{l=1}^t \eta_l .
\eea
Introducing with $\omega_1=0, \|\omega_{t+1}\|_{\HK}^2\geq 0$, and then multiplying both sides by $2/t,$ we get
\be\label{averageBound}
{1 \over t} \sum_{l=1}^t \eta_l \mE_{\J}[\mcE_{\bf z}(\omega_l)]  \leq  2 \mcE_{\bf z}(0) {1 \over t} \sum_{l=1}^t \eta_l .
\ee

It remains to bound the last term of \eref{decomposition}. Let $k \in [t-1]$ and $i\in \{t-k,\cdots, t\}.$ Note that given the sample $\bf z,$ $\omega_i$ is depending only on $\J_1,\cdots, \J_{i-1}$ when $i>1$ and $\omega_1=0.$
Thus, we can apply Lemma \ref{lemma:mfejer} with $\omega=\omega_{t-k}$ to derive
\bea
\eta_i \left( \mcE_{\bf z}(\omega_i) - \mcE_{\bf z}(\omega_{t-k}) \right)
\leq   \|\omega_{i}-\omega_{t-k}\|_{\HK}^2 - \mE_{\J_i}\|\omega_{i+1}-\omega_{t-k}\|_{\HK}^2 + \eta_i^2 \kappa^2 \mcE_{\bf z}(\omega_i).
\eea
Therefore,
\bea
\eta_i \mE_{\J}\left[ \mcE_{\bf z}(\omega_i) - \mcE_{\bf z}(\omega_{t-k}) \right]
\leq   \mE_{\J} [\|\omega_{i}-\omega_{t-k}\|_{\HK}^2 -\|\omega_{i+1}-\omega_{t-k}\|_{\HK}^2] + \eta_i^2 \kappa^2 \mE_{\J} [\mcE_{\bf z}(\omega_i)].
\eea
Summing up over $i=t-k,\cdots, t,$
\bea
\sum_{i=t-k}^t \eta_i \mE_{\J} \left [ \mcE_{\bf z}(\omega_i) - \mcE_{\bf z}(\omega_{t-k}) \right]
\leq   \kappa^2 \sum_{i=t-k}^t \eta_i^2 \mE_{\J} [\mcE_{\bf z}(\omega_i)].
\eea
Note that the left hand side is exactly $\sum_{i=t-k+1}^t \eta_i \mE_{\J}  \left[ \mcE_{\bf z}(\omega_i) - \mcE_{\bf z}(\omega_{t-k}) \right]$.
We thus know that the last term of \eref{decomposition} can be upper bounded by
\bea
&&\kappa^2  \sum_{k=1}^{t-1}{1 \over k(k+1)}  \sum_{i=t-k}^t \eta_i^2\mE_{\J} [\mcE_{\bf z}(\omega_i)] \\
&=& \kappa^2  \sum_{k=1}^{t-1}{1 \over k(k+1)}  \sum_{i=t-k}^{t-1} \eta_i^2\mE_{\J} [\mcE_{\bf z}(\omega_i)] + \kappa^2 \eta_t^2\mE_{\J} [\mcE_{\bf z}(\omega_t)] \sum_{k=1}^{t-1}{1 \over k(k+1)}.
\eea
Using the fact that
\bea
\sum_{k=1}^{t-1}{1 \over k(k+1)} = \sum_{k=1}^{t-1} \left({1 \over k} - {1 \over k+1}\right) = 1 - {1\over t} \leq 1,
\eea
and $\kappa^2 \eta_t \leq 1/2,$
we get that the last term of \eref{decomposition} can be bounded as
\bea
&&\sum_{k=1}^{t-1} {1 \over k(k+1)} \sum_{i=t-k+1}^{t} \eta_i ( \mE_{\J}[\mcE_{\bf z}(\omega_i)] -  \mE_{\J}[\mcE_{\bf z}(\omega_{t-k})]) \\
&\leq &\kappa^2  \sum_{k=1}^{t-1}{1 \over k(k+1)}  \sum_{i=t-k}^{t-1} \eta_i^2\mE_{\J} [\mcE_{\bf z}(\omega_i)] +  \eta_t\mE_{\J} [\mcE_{\bf z}(\omega_t)]/2.
\eea
Plugging the above and \eref{averageBound} into the decomposition \eref{decomposition}, and rearranging terms
\bea
   \eta_t \mE_{\J}[\mcE_{\bf z}(\omega_t)]/2  \leq 2\mcE_{\bf z}(0) {1 \over t} \sum_{l=1}^t \eta_l + \kappa^2  \sum_{k=1}^{t-1}{1 \over k(k+1)}  \sum_{i=t-k}^{t-1} \eta_i^2\mE_{\J} [\mcE_{\bf z}(\omega_i)],
\eea
which leads to the desired result by multiplying both sides by $2$. The proof is complete.
\end{proof}

We also need the following lemma, whose proof can be done using an induction argument.
\begin{lemma}
  \label{lemma:induction}
  Let $\{u_t\}_{t=1}^T$, $\{A_t\}_{t=1}^T$ and  $\{B_t\}_{t=1}^T$ be three sequences of non-negative numbers such that
  $u_1 \leq A_1$ and
  \be\label{inductionAssu}
  u_t \leq A_t  + B_t \sup_{i \in [t-1]} u_i,\qquad \forall t\in\{2,3,\cdots, T\}.
  \ee
  Let $\sup_{t \in [T]} B_t \leq B < 1.$
  Then for all $t \in [T],$
  \be\label{inductionConse}
  \sup_{k \in [t]} u_k \leq   {1 \over 1 - B} \sup_{k \in [t]} A_k.
  \ee
\end{lemma}
\begin{proof}
  When $t=1,$ \eref{inductionConse} holds trivially since $u_1 \leq A_1$ and $B< 1$. Now assume for some $t \in \mN$ with $ 2\leq  t \leq T,$
  \bea
  \sup_{i \in [t-1]} u_i \leq {1 \over 1 - B} \sup_{i \in [t-1]} A_i.
  \eea
  Then, by \eref{inductionAssu}, the above hypothesis,  and $B_t \leq B$, we have
  \bea
  u_t \leq A_t  + B_t \sup_{i \in [t-1]} u_i \leq A_t + {B_t \over 1 - B} \sup_{i \in [t-1]} A_i \leq \sup_{i \in [t]} A_i \left(1 + {B_t \over 1 - B} \right) \leq \sup_{i \in [t]} A_i {1 \over 1 - B}.
  \eea
  Consequently,
  \bea
  \sup_{k \in [t]} u_k \leq   {1 \over 1 - B} \sup_{k \in [t]} A_k,
  \eea
  thereby showing that indeed \eref{inductionConse} holds for $t$.
  By mathematical induction, \eref{inductionConse} holds for every $t\in [T].$
  The proof is complete.
\end{proof}
Now we can bound $\mE_{\J}[\mcE_{\bf z}(\omega_k)]$ as follows.
\begin{lemma}\label{lemma:empriskB}
 Let $\eta_1 \kappa^2 \leq 1/2$ and for all $t \in [T]$ with $t \geq 2,$
  \be\label{empriskBCon}
  {1 \over \eta_t}  \sum_{k=1}^{t-1}{1 \over k(k+1)}  \sum_{i=t-k}^{t-1} {\eta_i^2} \leq {1 \over 4\kappa^2}.
  \ee
Then for all $t \in[T],$
\be\label{empiricalBConse}
   \sup_{k\in [t]} \mE_{\J}[\mcE_{\bf z}(\omega_k)]  \leq 8 \mcE_{\bf z}(0) \sup_{k\in [t]} \left\{ {1 \over \eta_k k} \sum_{l=1}^k \eta_l\right\}.
\ee
\end{lemma}
\begin{proof}
By Lemma \ref{lemma:empiricalRelat}, we have \eref{empiricalRelat}. Dividing both sides by $\eta_t$, we can relax the inequality as
\bea
    \mE_{\J}[\mcE_{\bf z}(\omega_t)]  \leq 4\mcE_{\bf z}(0) {1 \over \eta_t t} \sum_{l=1}^t \eta_l + 2\kappa^2 {1 \over \eta_t}  \sum_{k=1}^{t-1}{1 \over k(k+1)}  \sum_{i=t-k}^{t-1} \eta_i^2 \sup_{i \in [t-1]}\mE_{\J} [\mcE_{\bf z}(\omega_i)].
\eea
In Lemma \ref{lemma:induction},
we let  $u_t = \mE_{\J}[\mcE_{\bf z}(\omega_t)]$, $A_t = 4\mcE_{\bf z}(0){1 \over \eta_t t} \sum_{l=1}^t \eta_l$ and $$B_t = 2\kappa^2 {1 \over \eta_t}  \sum_{k=1}^{t-1}{1 \over k(k+1)}  \sum_{i=t-k}^{t-1} \eta_i^2.$$
Condition \eref{empriskBCon} guarantees that $\sup_{t \in [T]} B_t \leq 1/2.$ Thus, \eref{inductionConse} holds, and the desired result follows by plugging with $B=1/2.$
The proof is complete.
\end{proof}

Finally, we need the following lemma to bound $\mcE_{\bf z}(0)$, whose proof follows from applying the Bernstein inequality from Lemma \ref{lem:Bernstein}.
\begin{lemma}\label{lem:sumY}
  Under Assumption \ref{as:noiseExp}, with probability at least $1-\delta_3$ ($\delta_3 \in ]0,1[$), there holds
  \bea
  \mcE_{\bf z}(0) \leq   Mv + 2Mv\left( {1 \over m} + {\sqrt{2} \over \sqrt{m} } \right) \log {2 \over \delta_3}.
  \eea
  In particular, if $m \geq 32 \log^2 {2 \over \delta_3},$ then
  \be\label{sumY}
  \mcE_{\bf z}(0) \leq 2M v.
  \ee
\end{lemma}
\begin{proof}
  Following from \eref{noiseExp},
  \bea
  \int_{Z} y^{2l} d\rho \leq {1 \over 2} l! M^{l-2} \cdot (2M^2 v), \qquad \forall l\in \mN,
  \eea
  and
  \bea
  \int_{Z} y^{2} d\rho \leq   M  v.
  \eea
  Therefore,
  \bea
  \int_{Z} |y^2 - \mE y^2|^l d\rho &\leq& \int_{Z} \max(|y|^{2l}, (\mE y^2)^l) d\rho \\
  &\leq& \int_{Z} (|y|^{2l}+ (\mE y^2)^l) d\rho \\
  &\leq& {1 \over 2} l! M^{l-2} \cdot (2M^2 v) +   (M v)^l \\
  &\leq& {1 \over 2} l! (M v)^{l-2} (2M v)^2,
  \eea
  where for the last inequality we used $v \geq 1.$
  Applying Lemma \ref{lem:Bernstein}, with $\omega_i = y_i^2$ for all $i\in[m]$, $B= M v$ and $\sigma = 2Mv,$ we know that with probability at least $1-\delta_3,$ there holds
  \bea
  {1 \over n}\sum_{i=1}^n y_i^2 - \int_{Z}y^2 d\rho \leq 2Mv\left( {1 \over m} + {2 \over \sqrt{m} } \right) \log {2 \over \delta_3}.
  \eea
The proof is complete.
\end{proof}

\subsection{Bounding $\left\|\TK^{1 \over 2}\Pi^t_{k+1}(\TX)\right\|$}
We bound the weighted factor $\left\|\TK^{1 \over 2}\Pi^t_{k+1}(\TX)\right\|$ as follows.
\begin{lemma}\label{lem:sqrtLKProd}
 Assume \eref{differOperator} holds for some $\lambda>0$
 and $\eta_1 \kappa^2 \leq 1$. Then
  \bea
  \|\TK^{1\over 2}\Pi^t_{k+1}(\TX)\|^2 \leq {1 \over \sum_{i=k+1}^t \eta_i} + 4\lambda.
  \eea
\end{lemma}
\begin{proof}
Note that we have
  \bea
  \|\TK^{1\over 2}\Pi^t_{k+1}(\TX)\| \leq \|\TK^{1\over 2}(\TX + \lambda I)^{-{1\over 2}}\| \|(\TX + \lambda I)^{1\over 2}\Pi^t_{k+1}(\TX)\|.
  \eea
 Using \eref{differOperator},  we can relax the above as
  \bea
  \|\TK^{1\over 2}\Pi^t_{k+1}(\TX)\| \leq 2 \|(\TX + \lambda I)^{1\over 2}\Pi^t_{k+1}(\TX)\|,
  \eea
  which leads to
  \bea
  \|\TK^{1\over 2}\Pi^t_{k+1}(\TX)\|^2 \leq 4 \|(\TX + \lambda I)^{1\over 2}\Pi^t_{k+1}(\TX)\|^2.
  \eea
  Since
  \bea
  \|(\TX + \lambda I)^{1\over 2}\Pi^t_{k+1}(\TX)\|^2 &=& \|(\TX + \lambda I)\Pi^t_{k+1}(\TX)\Pi^t_{k+1}(\TX)\|  \\
  \\ &\leq&  \|\TX \Pi^t_{k+1}(\TX)\Pi^t_{k+1}(\TX)\| + \lambda \\
  &=& \|\TX^{1\over 2} \Pi^t_{k+1}(\TX)\|^2 + \lambda,
  \eea
  and with $\eta_t\kappa^2 \leq 1$, $\|\TX\| \leq \tr(\TX) \leq \kappa^2,$
  by Lemma \ref{lemma:initialerror},
  \bea
  \|\TX^{1\over 2} \Pi^t_{k+1}(\TX)\|^2 \leq {1 \over 2\mathrm{e} \sum_{i=k+1}^t \eta_i} \leq {1 \over 4 \sum_{i=k+1}^t \eta_i},
  \eea
  we thus derive the desired result. The proof is complete.
\end{proof}

\subsection{Deriving Error Bounds}
With Lemmas \ref{lem:cul_err}--\ref{lem:sqrtLKProd}, we are ready to estimate the computational variance , $\mE_{\J}\|\IK \omega_{t+1} - \IK \nu_{t+1}\|_{\rho}^2,$ as follows.
\begin{pro}\label{pro:compErrA}
	Under Assumption \ref{as:noiseExp},
assume \eref{differOperator} holds for some $\lambda>0$, $\eta_1 \kappa^2 \leq 1/2,$  \eref{empriskBCon} and \eref{sumY}.
Then, we have for all  $t\in [T],$
\be\label{compuErrA}
\mE_{\bf J}\|\IK \omega_{t+1} - \IK \nu_{t+1}\|_{\rho}^2
\leq {16 M v \kappa^2 \over b}  \sup_{k\in [t]} \left\{ {1 \over \eta_k k} \sum_{l=1}^k \eta_l\right\} \left(\sum_{k=1}^{t-1}  {\eta_{k}^2 \over \sum_{i=k+1}^t \eta_i} + 4\lambda \sum_{k=1}^{t-1} \eta_{k}^2  + \eta_t^2 \kappa^2 \right).
\ee
\end{pro}

\begin{proof}
According to Lemmas \ref{lem:cul_err} and \ref{lemma:empriskB}, we have \eref{eq:cul_err} and \eref{empiricalBConse}. It thus follows that
\bea
\mE_{\bf J}\|\IK \omega_{t+1} - \IK \nu_{t+1}\|_{\rho}^2
\leq {8 \mcE_{\bf z}(0) \kappa^2 \over b}  \sup_{k\in [t]} \left\{ {1 \over \eta_k k} \sum_{l=1}^k \eta_l\right\} \sum_{k=1}^t \eta_{k}^2 \left\|\TK^{1\over 2}\Pi^t_{k+1}(\TX)\right\|^2.
\eea
Now the proof can be finished by applying Lemma \ref{lem:sqrtLKProd} which tells us that
\bea
\sum_{k=1}^t \eta_{k}^2 \left\|\TK^{1 \over 2}\Pi^t_{k+1}(\TX)\right\|^2 &=& \sum_{k=1}^{t-1} \eta_{k}^2 \left\|\TK^{1\over 2}\Pi^t_{k+1}(\TX)\right\|^2 + \eta_t^2 \left\|\TK^{1 \over 2}\right\|^2 \\
&\leq& \sum_{k=1}^{t-1}  {\eta_{k}^2 \over \sum_{i=k+1}^t \eta_i} + 4\lambda \sum_{k=1}^{t-1} \eta_{k}^2  + \eta_t^2 \kappa^2,
\eea
 and \eref{sumY} to the above inequality. The proof is complete.
\end{proof}
Setting $\eta_t = \eta_1 t^{-\theta}$ for some appropriate $\eta_1$ and $\theta$ in the above proposition, we get the following explicitly upper bounds for $\mE_{\J}\|\IK \omega_{t+1} - \IK \nu_{t+1}\|_{\rho}^2.$
\begin{pro}\label{pro:compErrB}Under Assumption \ref{as:noiseExp},
 assume \eref{differOperator} holds for some $\lambda>0$ and \eref{sumY}.
Let $\eta_t = \eta_1 t^{-\theta} $ for all $t \in [T],$ with $\theta \in [0,1[$ and
  \be\label{etaRestri}
  0<\eta_1 \leq {t^{\min(\theta, 1-\theta)} \over 8 \kappa^2 (\log t + 1)} , \qquad \forall t\in [T].
  \ee
 Then, for all  $t\in [T]$,
\be\label{compuErrB}
\mE_{\J}\|\IK \omega_{t+1} - \IK \nu_{t+1}\|_{\rho}^2
\leq {16 M v \kappa^2 \over b(1-\theta)} \left( 5\eta_1 t^{-\min(\theta,1-\theta)} + 8\lambda \eta_1^2 t^{(1-2\theta)_+}\right) (1 \vee \log t ).
\ee
\end{pro}
\begin{proof}
  We will use Proposition \ref{pro:compErrA} to prove the result.
  Thus, we need to verify the condition \eref{empriskBCon}.
Note that
\bea
\sum_{k=1}^{t-1} {1 \over k(k+1)} \sum_{i=t-k}^{t-1} \eta_i^2 = \sum_{i=1}^{t-1} \eta_i^2 \sum_{k=t-i}^{t-1} {1 \over k(k+1)} = \sum_{i=1}^{t-1} \eta_i^2 \left( {1\over t-i} - {1 \over t} \right) \leq \sum_{i=1}^{t-1} {\eta_i^2 \over t-i}.
\eea
Substituting with $\eta_i = \eta i^{-\theta},$ and by Lemma \ref{lem:estimate2},
\bea
\sum_{k=1}^{t-1} {1 \over k(k+1)} \sum_{i=t-k}^{t-1} \eta_i^2 \leq \eta_1^2 \sum_{i=1}^{t-1} {i^{-2\theta} \over t-i} \leq 2\eta_1^2 t^{-\min(2\theta,1)} (\log t +1).
\eea
Dividing both sides by $\eta_t $ ($= \eta_1 t^{-\theta}$), and then using \eref{etaRestri},
\bea
{1\over \eta_t }\sum_{k=1}^{t-1} {1 \over k(k+1)} \sum_{i=t-k}^{t-1} \eta_i^2 \leq 2\eta_1 t^{-\min(\theta,1-\theta)} (\log t +1) \leq {1 \over 4 \kappa^2}.
\eea
This verifies \eref{empriskBCon}. Note also that by taking $t= 1$ in \eref{etaRestri}, for all $t\in [T]$ ,
\bea
\eta_t \kappa^2 \leq \eta_1 \kappa^2 \leq {1 \over 8 \kappa^2} \leq {1 \over 2}.
\eea
 We thus can apply Proposition \ref{pro:compErrA} to derive \eref{compuErrA}. What remains is to control the right hand side of \eref{compuErrA}.
Since
\bea
\sum_{k=1}^{t-1}{\eta_{k}^2 \over \sum_{i=k+1}^t \eta_i} = \eta_1 \sum_{k=1}^{t-1}{k^{-2\theta} \over \sum_{i=k+1}^t i^{-\theta}} \leq \eta_1 \sum_{k=1}^{t-1}{k^{-2\theta} \over (t-k)t^{-\theta}},
\eea
combining with Lemma \ref{lem:estimate2},
\bea
\sum_{k=1}^{t-1}{\eta_{k}^2 \over \sum_{i=k+1}^t \eta_i} \leq 2  \eta_1 t^{-\min(\theta,1-\theta)} (\log t +1).
\eea
Also, by Lemma \ref{lem:estimate1},
\bea
{1 \over \eta_k k} \sum_{l=1}^k \eta_l = {1 \over k^{1-\theta}} \sum_{l=1}^k l^{-\theta} \leq {1 \over 1-\theta},
\eea
and by Lemma \ref{lem:estimate1a},
\bea
\sum_{k=1}^{t-1} \eta_{k}^2 = \eta_1^2 \sum_{k=1}^{t-1} k^{-2\theta} \leq \eta_1^2 t^{\max(1-2\theta,0)} (\log t +1).
\eea
Introducing the last three estimates into \eref{compuErrA} and using that $\eta_t^2 \kappa^2 \leq \eta_1 t^{-\theta}$ by \eref{etaRestri}, we get the desired result. The proof is complete.
\end{proof}

Collect some of the above analysis, we get the following result for the computational variance.
\begin{thm}
   Under Assumptions \ref{as:noiseExp}, let $\delta_2 \in]0,1[$, ${9\kappa^2 \over m} \log {m \over \delta_2} \leq \lambda \leq \|\TK\|,$ $\delta_3\in ]0,1[$, $m \geq 32 \log^2 {2 \over \delta_3},$ and $\eta_t = \eta_1 t^{-\theta} $ for all $t \in [T],$ with $\theta \in [0,1[$ and $\eta_1$ such that \eref{etaRestri}.
 Then, with probability at least $1-\delta_2 -\delta_3$, \eref{compuErrB} holds for all  $t\in [T]$.
\end{thm}

\section{Deriving Total Error Bounds} \label{sec:deriveing}
The purpose of this section is to derive total error bounds.
\subsection{Attainable Case}
We have the following general theorem for $\zeta \geq 1/2$, with which we  prove our main results stated in Section \ref{sec:main}.
\begin{thm}\label{thm:generalRate}
  Under Assumptions  \ref{as:noiseExp}, \ref{as:regularity} and \ref{as:eigenvalues}, let $\zeta \geq 1/2$, $T \in \mN$ with $T\geq 3,$ $\delta \in]0,1[$, $\eta_t = \eta \kappa^{-2} t^{-\theta} $ for all $t \in [T],$ with $\theta \in [0,1[$ and $\eta$ such that
  \be\label{etaRestriA}
  0<\eta \leq {t^{\min(\theta, 1-\theta)} \over 8(\log t + 1)} , \qquad \forall t\in [T].
  \ee
  If for some $\epsilon \in]0,1],$
    \be\label{sampleN}
 m \geq \left( {18 \kappa^2 \over \epsilon \|\TK\|}  \log \left( {27 \kappa^2 \over \epsilon \|\TK\| \delta}  \right)  \right)^{1/\epsilon},
\ee
     then the following holds with probability at least $1-\delta$: for all $t \in [T],$
        \be\label{totalErrGen}
        \begin{split}
        \mE_{\J}[\mcE(\omega_{t+1})] - \inf_{\HK} \mcE \leq q_1 (\eta t^{1-\theta})^{-2\zeta} + q_2 m^{\gamma(1-\epsilon) -1} (1 \vee  \eta^2 m^{2\epsilon-2} t^{2-2\theta} ) (\log T)^2  \log^2 {12 \over \delta} \\ + q_3 \eta b^{-1} ( t^{-\min(\theta,1-\theta)} \vee m^{\epsilon-1}\eta t^{(1-2\theta)_+} ) \log T.
        \end{split}
        \ee
        Here, $q_1 = 2R^2 \zeta^{2\zeta},$ $q_2= {10^4 (R\kappa^{2\zeta} + \sqrt{M})^2 (\kappa/{\sqrt{\|\TK\|} + \sqrt{2 \sqrt{v} c_{\gamma} /\|\TK\|^{\gamma}} })^2 \over (1-\theta)^2 } ,$ and $q_3 = {208 M v \over 1-\theta}.$
\end{thm}
\begin{proof}
Let $\lambda = \|\TK\| m^{\epsilon - 1}.$  Clearly, $\lambda \leq \|\TK\|.$
 For any $A \geq 0$ and $B\geq 1$, by applying \eref{exppoly} with $\zeta=1, x =(Bm)^{\epsilon}$ and $c = {\epsilon \over 2A B^{\epsilon}},$
 \be\label{exppolyB}
 A \log (Bm) = {A \over \epsilon} \log( (Bm)^{\epsilon}) \leq {A \over \epsilon} \log \left({2 A B^{\epsilon} \over \mathrm{e}\epsilon}\right) + {1 \over 2} m^{\epsilon} \leq {A \over \epsilon} \log \left({A B \over \epsilon} \right) + {1 \over 2} m^{\epsilon}.
 \ee
Using the above inequality with
  $A = {9\kappa^2 \over \|\TK\|}$ and $B={1 \over \delta_2},$
one can prove that the condition \eref{sampleN}
 ensures that ${9\kappa^2 \over m} \log {m \over \delta_2} \leq \lambda$ is satisfied with $\delta_2 = {\delta \over 3},$
  Therefore, by Lemma \ref{lem:differOperator}, \eref{differOperator} holds with probability at least $1-\delta_2.$
 Similarly the condition \eref{sampleN}
 implies that $m \geq 32 \log^2 {2 \over \delta_3}$ is satisfied with $\delta_3 = {\delta \over 3},$ and  thus by Lemma \ref{lem:sumY}, \eref{sumY} holds with probability at least $1-\delta_3.$
 Combining with Lemma \ref{lem:lambdaNk}, by taking the union bound, we know that with probability at least $1- \delta_1 - \delta_2 - \delta_3$, \eref{differOperator}, \eref{sumY} and \eref{lambdaNk} hold for all $k \in [T].$ Now, we can apply Propositions \ref{pro:sampleErrB} and \ref{pro:compErrB} to get \eref{gSampleErrB} and \eref{compuErrB}. Noting that by \eref{etaRestriA}, $\sqrt{2}\eta \leq 1,$ and by a simple calculation, we derive from \eref{gSampleErrB} that
  \bea
  \begin{split}
  &\|\IK \nu_{t+1} - \IK \mu_{t+1}\|_{\rho}^2 \\
  \leq& {4624(R\kappa^{2\zeta} + \sqrt{M})^2 (\kappa/{\sqrt{\|\TK\|} + \sqrt{2\sqrt{v} c_{\gamma} /\|\TK\|^{\gamma}} })^2 \over (1-\theta)^2 } m^{\gamma(1-\epsilon) -1} (1 \vee \lambda^2 \eta^2\kappa^{-4} t^{2-2\theta} \vee \log^2 t) \log^2 {4 \over \delta_1} \\
  \leq& {4624 (R\kappa^{2\zeta} + \sqrt{M})^2 (\kappa/{\sqrt{\|\TK\|} + \sqrt{2 \sqrt{v} c_{\gamma} /\|\TK\|^{\gamma}} })^2 \over (1-\theta)^2 } m^{\gamma(1-\epsilon) -1}  (1 \vee  \eta^2 m^{2\epsilon -2} t^{2-2\theta} ) (\log T)^2  \log^2 {4 \over \delta_1},
  \end{split}\eea
  where for the last inequality, we used $\|\TK\| \leq \kappa^2.$ Similarly, by a simple calculation, we get from \eref{compuErrB} that
  \bea
\mE_{\bf J}\|\IK \omega_{t+1} - \IK \nu_{t+1}\|_{\rho}^2
&\leq& {208 M v \over b(1-\theta)} (\eta t^{-\min(\theta,1-\theta)} \vee \lambda \eta^2 \kappa^{-2} t^{(1-2\theta)_+} ) (1 \vee \log t )\\
&\leq&  {208 M v \over b(1-\theta)} (\eta t^{-\min(\theta,1-\theta)} \vee m^{\epsilon -1}\eta^2 t^{(1-2\theta)_+} ) \log T.
\eea
Letting $\delta_1 = {\delta\over 3}$, and
introducing the above estimates and \eref{initialErrB} into \eref{errorDecompos}, we get \eref{totalErrGen}. The proof is complete.
\end{proof}

\begin{proof}[of Theorem \ref{thm:main}]
By choosing $\epsilon = 1 - {1 \over 2\zeta + \gamma}$ and $\theta=0$ in  Theorem \ref{thm:generalRate}, then the condition \eref{sampleN} reduces to
$m \geq m_{\delta}$, where
\be\label{numberDelta}
 m_{\delta} = \left( {18 \kappa^2 p \over \|\TK\|}  \log \left( {27 \kappa^2p \over \|\TK\| \delta}  \right)  \right)^{p}, \quad p = {2\zeta + \gamma \over 2\zeta + \gamma -1}.
\ee
The desired result thus follows by applying Theorem \ref{thm:generalRate}.
\end{proof}

\subsection{Non Attainable Case}
For the non-attainable case, we have the following general results on generalization errors for SGM.
\begin{thm}\label{thm:generalRateNon}
  Under Assumptions  \ref{as:noiseExp}, \ref{as:regularity} and \ref{as:eigenvalues}, let $\zeta \leq 1/2$, $T \in \mN$ with $T\geq 3,$ $\delta \in]0,1[$, $\eta_t = \eta \kappa^{-2} t^{-\theta} $ for all $t \in [T],$ with $\theta \in [0,1[$ and $\eta$ such that \eref{etaRestriA} and
 for some $\epsilon \in]0,1],$ \eref{sampleN} holds.
     Then the following holds with probability at least $1-\delta$: for all $t \in [T],$
        \be\label{totalErrGenNon}
        \begin{split}
        \mE_{\J}[\mcE(\omega_{t+1})] - \inf_{\HK} \mcE
        \lesssim \left( (\eta t^{1-\theta})^{-2\zeta}
        +  m^{\gamma(1-\epsilon) -1}  \right) (1 \vee  \eta m^{\epsilon-1} t^{1-\theta} )^3 \log^4 T  \log^2 {1 \over \delta} \\
        + \eta b^{-1} ( t^{-\min(\theta,1-\theta)} \vee m^{\epsilon-1}\eta t^{(1-2\theta)_+} ) \log T .
        \end{split}
        \ee
        {Here, the  constant in the upper bounds is positive and  depends only on $\kappa^2, \|\TK\|, M, v, \zeta, R,c_{\gamma}$, $\gamma$ and $\|\FH\|_{\infty}$  } .
\end{thm}
\begin{proof}
The proof is similar to that for Theorem \ref{thm:generalRate}. We include the sketch only and omit the constants appeared.
Similar to the proof of Theorem \ref{thm:generalRate}, with $\lambda= \|\TK\|m^{\epsilon-1},$ one can prove that with probability at least $1- \delta_1 - \delta_2 - \delta_3$, \eref{differOperator}, \eref{sumY} and \eref{lambdaNkNon} hold for all $k \in [T].$ Now, we can apply Propositions \ref{pro:sampleErrB} and \ref{pro:compErrB} to get \eref{gSampleErrBNon} and \eref{compuErrB}. Noting that by \eref{etaRestri}, $\sqrt{2}\eta \leq 1,$ and by a simple calculation, we derive from \eref{gSampleErrBNon} that
  \begin{align*}
  &\|\IK \nu_{t+1} - \IK \mu_{t+1}\|_{\rho}^2\\
  \lesssim& m^{\gamma(1-\epsilon) -1} (1 \vee  \eta^2 m^{2\epsilon-2} t^{2-2\theta} ) \log^4 T  \log^2 {1 \over \delta}
  + (\eta t^{1-\theta} )^{-2\zeta} (1\vee \eta t^{1-\theta} m^{-1})(1 \vee  \eta^2 m^{2\epsilon-2} t^{2-2\theta} ) \log^4 T  \log^2 {1 \over \delta}\\
  \lesssim& m^{\gamma(1-\epsilon) -1} (1 \vee  \eta^2 m^{2\epsilon-2} t^{2-2\theta} ) \log^4 T  \log^2 {1 \over \delta}
  + (\eta t^{1-\theta} )^{-2\zeta} (1 \vee  \eta m^{\epsilon-1} t^{1-\theta} )^3 \log^4 T  \log^2 {1 \over \delta}.
  \end{align*}
The rest of the proof parallelizes to that for Theorem \ref{thm:generalRate}.
\end{proof}

Now, we are in a position to prove Theorem \ref{thm:generalRateNonFix}.
\begin{proof}[of Theorem \ref{thm:generalRateNonFix}]
 The second part of the theorem follows directly from applying Theorem \ref{thm:generalRateNon} with $\theta=0.$  The first part can be proved by applying Theorem \ref{thm:generalRateNon} with $\theta=0$ and $\epsilon= 1- {1 \over 2\zeta+\gamma}$, combining with the same argument from the proof of Theorem \ref{thm:main} to verify the condition \eref{sampleN}. We omit the details.
\end{proof}


\subsection{Batch GM}
Following the proof of Theorems \ref{thm:main} and \ref{thm:generalRateNonFix}, we know that the following results hold for batch GM, from which one can prove Theorem \ref{thm:bgmopt}.

\begin{thm}\label{thm:bgmoptGen}
  Under Assumptions  \ref{as:noiseExp}, \ref{as:regularity} and \ref{as:eigenvalues}, set $\eta_t = \eta \kappa^{-2} $ with $\eta\leq 1$, for all $t \in [m].$ With probability at least $1-\delta$ ($0< \delta < 1$), the following holds for the learning sequence generated by \eref{Alg2B}: \\
     1) if $\zeta>1/2$ and $m \geq m_{\delta}$ with $m_{\delta}$ given by \eref{numberDelta}, then
      \be\label{eq:bgmBound}
        \begin{split}
        \mE_{\J} [\mcE(\omega_{t+1})] - \inf_{ \HK}\mcE \lesssim  (\eta t)^{-2\zeta} +  m^{-{2\zeta \over 2\zeta+\gamma}} (1 +  m^{-{1 \over 2\zeta + \gamma}} \eta t )^2 \log^2 T \log^2 {1 \over \delta};
        \end{split}
        \ee
2) if $\zeta\leq 1/2,$ $2\zeta+\gamma >1$ and $m \geq m_{\delta}$ with $m_{\delta}$ given by \eref{numberDelta}, then
      \bea
        \begin{split}
        \mE_{\J}[\mcE(\omega_{t+1})] - \inf_{ \HK} \mcE
        \lesssim \left( (\eta t)^{-2\zeta}
        +  m^{-{2\zeta \over 2\zeta+\gamma}}  \right) (1 \vee   m^{-{1\over 2\zeta+\gamma}} \eta t )^3 \log^4 T  \log^2 {1 \over \delta};
        \end{split}
        \eea
3) if $2\zeta+\gamma\leq 1$ and for some $\epsilon \in]0,1],$
    \eref{sampleN} hold, then
\bea
\begin{split}
        \mE_{\J}[\mcE(\omega_{t+1})] - \inf_{\HK} \mcE
        \lesssim \left( (\eta t)^{-2\zeta}
        +  m^{\gamma(1-\epsilon) -1}  \right) (1 \vee  \eta m^{\epsilon-1} t )^3 \log^4 T  \log^2 {1 \over \delta}.
        \end{split}
\eea
{Here, all the  constants in the upper bounds are positive and  depend only on $\kappa^2, \|\TK\|, M, v, \zeta, R,c_{\gamma}$ and $\gamma$ (and also on $\|\FH\|_{\infty}$ when $\zeta<1/2$) } .
\end{thm}

\section{Convergence in $\HK$-norm}\label{sec:hnorm}
In this section, we will give convergence results in $\HK$-norm for Algorithm \ref{alg:1} in the attainable case.
For the sake of simplicity, we will only consider a fixed step-size sequence, i.e, $\eta_t = \eta$ for all $t$.

Using a similar  procedure  as that for \eref{errorDecompos}, we can prove the following error decomposition,
\be\label{eq:errDecHn}
\mE_{\bf J}[\|\omega_{t} - \omega^{\dag}\|_{\HK}^2] \lesssim \|\mu_{t} - \omega^{\dag}\|_{\HK}^2 + \|\mu_{t} - \nu_t\|_{\HK}^2 + \mE_{\bf J}[\|\omega_{t} - \nu_t\|_{\HK}^2].
\ee
To estimate the bias term, $\|\mu_{t} - \omega^{\dag}\|_{\HK}^2$, we introduce the following lemma from \citep{yao2007early,rosasco2015learning}. Its proof is similar as that for \eref{initialErrB}
and will be given in Appendix \ref{sec:prov_estimates} for the sake of completeness.

\begin{lemma}\label{lemma:biasHnorm}
 Under Assumption \ref{as:regularity},
 let $\zeta \geq 1/2$ and $\eta_t = \eta $ for all $t\in \mN$, with $\eta \in ]0,\kappa^{-2}]$, then
\be\label{initialErrBHnorm}
  \|\mu_{t+1} - \omega^{\dag}\|_{\HK} \leq R \left( \zeta-1/2 \over \eta t \right)^{\zeta-1/2}.
  \ee
\end{lemma}
To estimate the sample variance term, $\|\mu_{t} - \nu_t\|_{\HK}^2$, we use \eref{eq:interm2} and get that
\bea
&&\|\nu_{t+1} - \mu_{t+1}\|_{\HK} = \left\|\sum_{k=1}^t \eta_k \Pi_{k+1}^t(\TX)  N_k\right\|_{\HK} \\
&\leq& \sum_{k=1}^t \eta_k \left\|\TKL^{-{1\over 2}}\right\|\left\|\TKL^{-{1\over 2}}\Pi_{k+1}^t(\TX)  N_k\right\|_{\HK}\leq  {1\over \sqrt{\lambda}} \sum_{k=1}^t \eta_k \left\|\TKL^{1\over 2}\Pi_{k+1}^t(\TX)  N_k\right\|_{\HK}.
\eea
From the proof of Theorem \ref{thm:sampleErr}, we know that $\sum_{k=1}^t \eta_k \left\|\TKL^{1\over 2}\Pi_{k+1}^t(\TX)  N_k\right\|_{\HK}$ is upper bounded by the right-hand side of \eref{gSampleErrB}. With $\eta_t = \eta$ and $\lambda = \|\TK\| m^{-{1\over 2\zeta+\gamma}}$, we thus have
\be\label{eq:samErrHnorm}
\|\nu_{t+1} - \mu_{t+1}\|_{\HK} \lesssim m^{-{\zeta-1/2 \over 2\zeta+\gamma}}  (1 +  m^{-{1\over 2\zeta+\gamma}}\eta t )\log t\log {1 \over \delta}.
\ee
Finally, for the computational variance term, $\mE_{\bf J}[\|\omega_{t} - \nu_t\|_{\HK}^2]$, we use a same procedure as that for \eref{eq:cul_err} to get
\be\label{eq:compErrHnorm}
\mE_{\bf J}\| \omega_{t+1} -  \nu_{t+1}\|_{\HK}^2
\leq {\kappa^2 \over b} \sum_{k=1}^t \eta^2 \left\|\Pi^t_{k+1}(\TX)\right\|^2  \mE_{\J}[\mcE_{\bf z}(\omega_k)]
\lesssim {\eta^2  t \over b} ,
\ee
where we used \eref{empiricalBConse} and \eref{sumY} in the last inequality. Introducing \eref{initialErrBHnorm}, \eref{eq:samErrHnorm} and \eref{eq:compErrHnorm} into the error decomposition \eref{eq:errDecHn}, we can prove Theorem \ref{thm:hnorm}.

\section{Numerical Simulations}\label{sec:numerical}
\begin{figure}[h]
    \centering
    \begin{subfigure}[]
    {0.3\textwidth} 	
    \includegraphics[width=\textwidth]{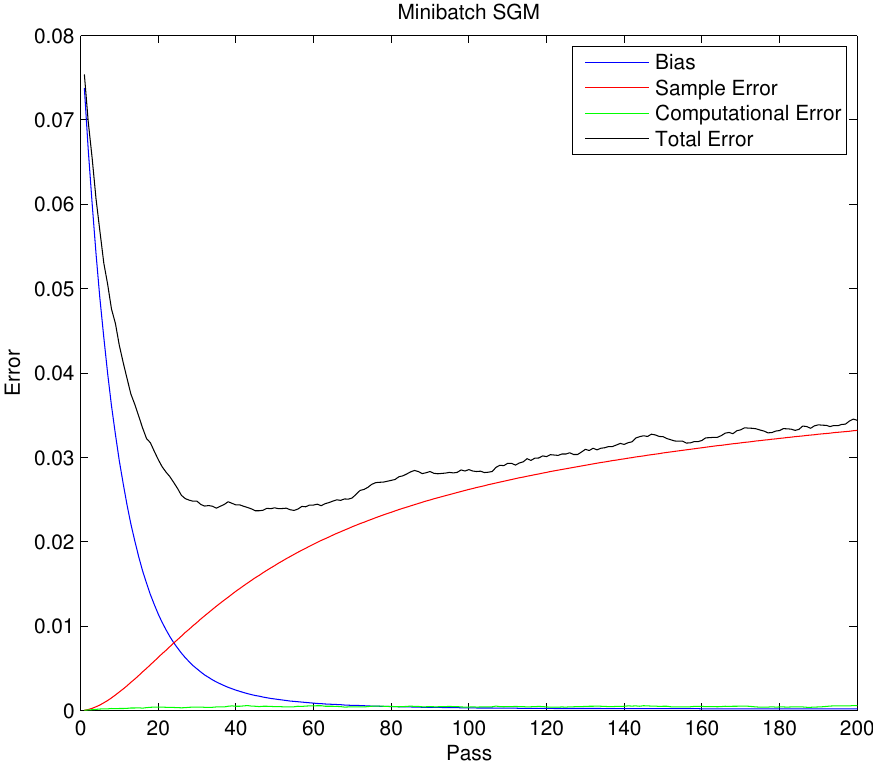}
    \caption{Minibatch SGM}
    \label{fig:minibatch_SGM}
    \end{subfigure}
     ~
    \begin{subfigure}[]
    {0.3\textwidth}
    \includegraphics[width=\textwidth]{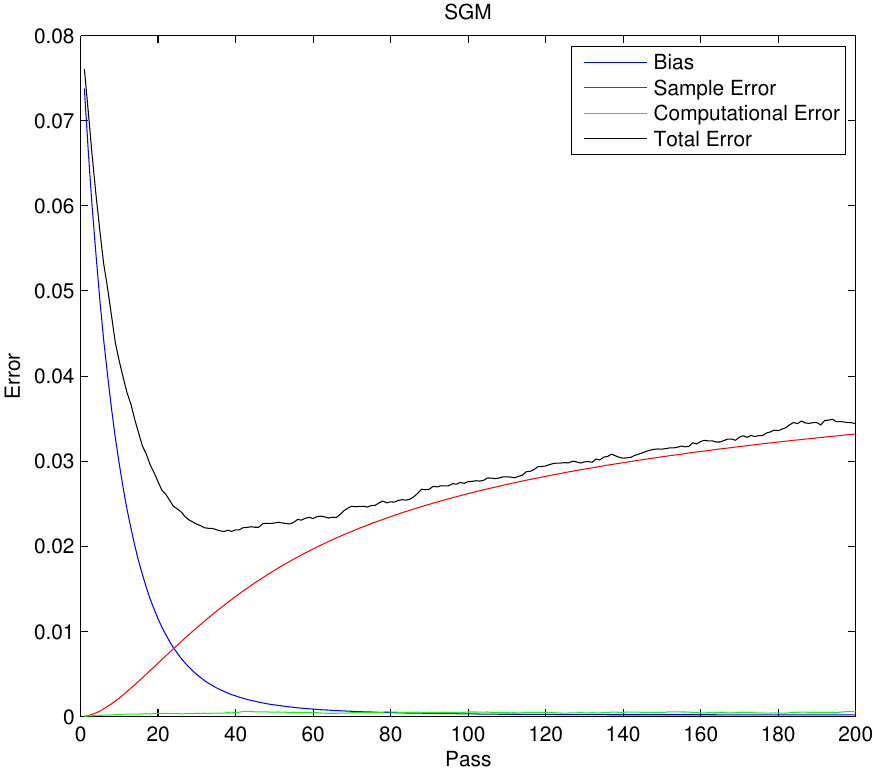} \caption{SGM}
    \label{fig:SGM}
    \end{subfigure}
    ~
    \begin{subfigure}[]
    {0.3\textwidth}
    \includegraphics[width=\textwidth]{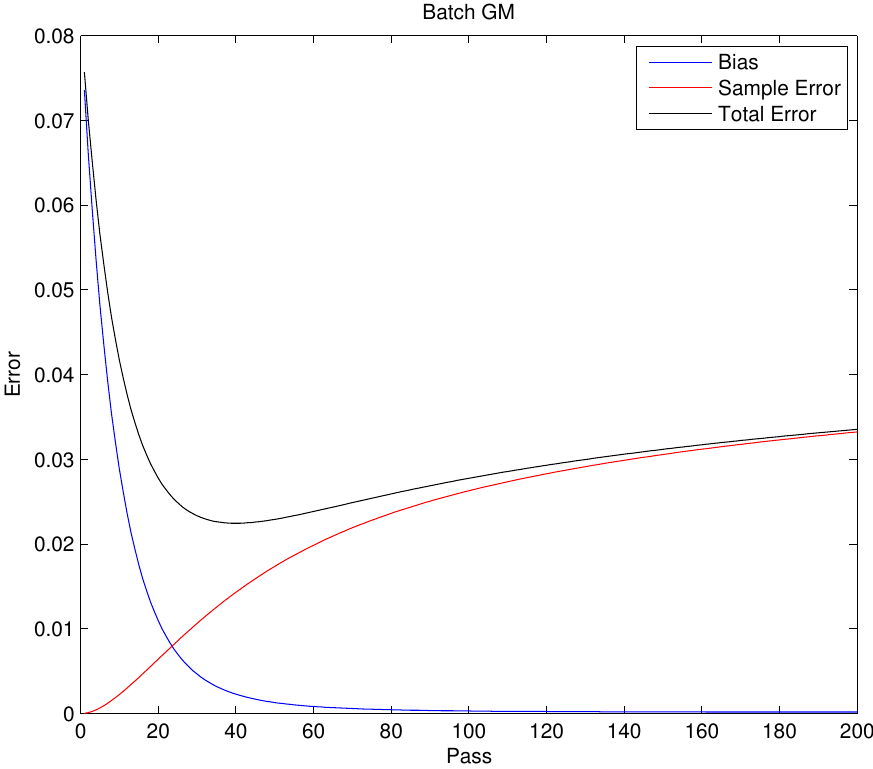}
    \caption{Batch GM}
    \label{fig:batchGM}
    \end{subfigure}
    \caption{Error decompositions for gradient-based learning algorithms on {\em synthesis data}, where {\em m = 100}.}
    \label{fig:comparisions}
\end{figure}

\begin{figure}[]
    \centering
    \begin{subfigure}[]
    {0.3\textwidth} 	
    \includegraphics[width=\textwidth]{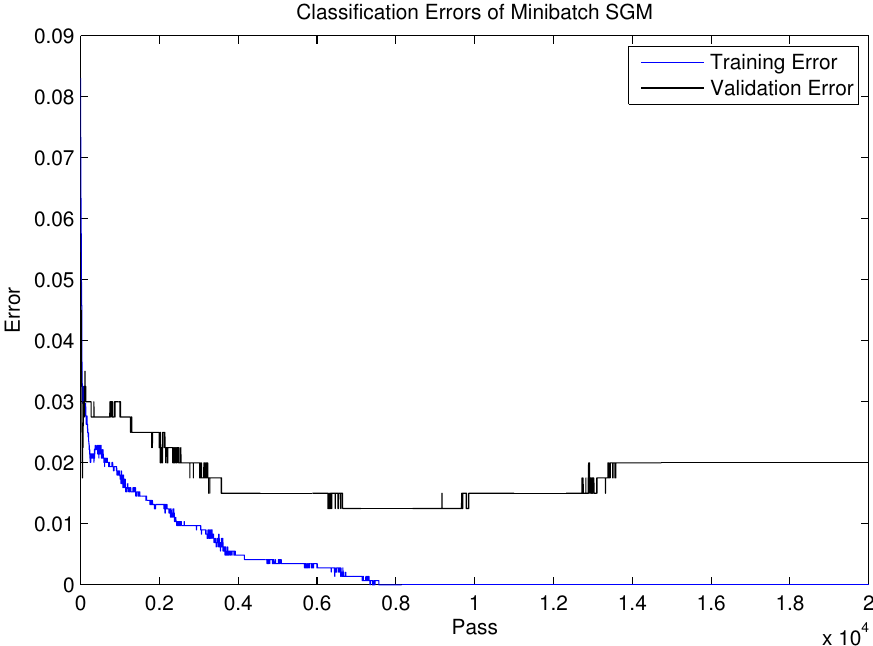}
    \caption{Minibatch SGM}
    \label{fig:minibatch_SGM_bc}
    \end{subfigure}
     ~
    \begin{subfigure}[]
    {0.3\textwidth}
    \includegraphics[width=\textwidth]{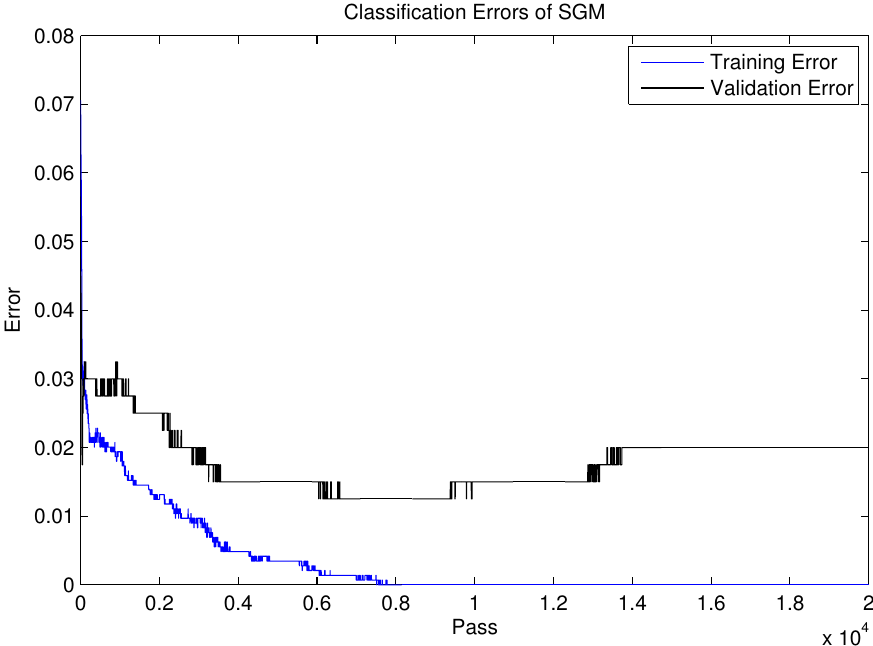}
    \caption{SGM}
    \label{fig:SGM_bc}
    \end{subfigure}
    ~
    \begin{subfigure}[]
    {0.3\textwidth}
    \includegraphics[width=\textwidth]{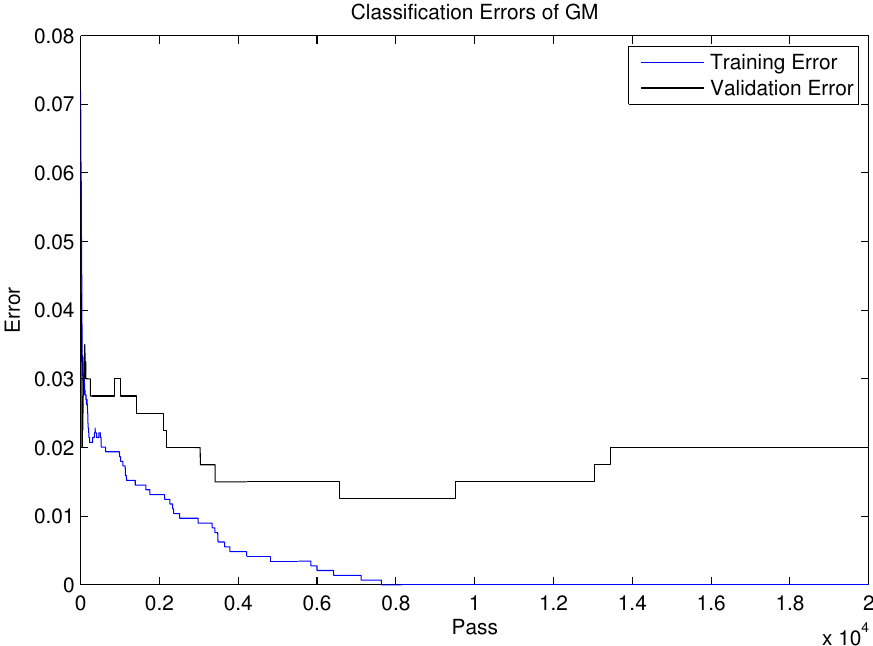}
    \caption{Batch GM}
    \label{fig:batchGM_bc}
    \end{subfigure}
    \caption{Misclassification Errors for gradient-based learning algorithms on {\em BreastCancer} dataset.}
    \label{fig:comparisions_bc}
\end{figure}

 In order to illustrate our theoretical results and the error decomposition, we first performed some simulations on a simple problem. We constructed $m=100$ i.i.d. training examples of the form $y= f_{\rho}(x_i) + \omega_i$.
Here, the regression function is $f_{\rho}(x) = |x - 1/2| - 1/2,$
the input point $x_i$ is uniformly distributed in $[0,1],$ and $\omega_i$ is a Gaussian noise with zero mean and standard deviation $1,$ for each $i\in [m].$
We perform three experiments with the same $\HK$, a RKHS associated
with a Gaussian kernel $K(x,x') = \exp(-(x-x')^2/(2\sigma^2))$ where $\sigma=0.2$. In the first experiment, we run mini-batch SGM, where the mini-batch size $b=\sqrt{m},$ and the step-size $\eta_t = 1/(8\sqrt{m})$.
In the second experiment, we run simple SGM where the step-size is fixed as $\eta_t = 1/(8m)$, while in the third experiment, we run batch GM using the fixed step-size $\eta_t = 1/8.$
For mini-batch SGM and SGM, the total error $\|\IK \omega_{t} - f_{\rho}\|_{\LRH}^2$, the bias $\|\IK \hat{\mu}_{t} - f_{\rho}\|_{\LRH}^2$, the sample variance $\|\IK \nu_{t} - \IK \hat{\mu}_t\|_{\LRH}^2$
and the computational variance $\|\IK \omega_t - \IK \nu_t\|_{\LRH}^2$, averaged over 50 trials, are depicted in Figures \ref{fig:minibatch_SGM} and \ref{fig:SGM}, respectively.
For batch GM, the total error $\|\IK \nu_{t} - f_{\rho}\|_{\LRH}^2$, the bias $\|\IK \hat{\mu}_{t} - f_{\rho}\|_{\LRH}^2$ and the sample variance $\|\IK \nu_{t} - \hat{\mu}_t\|_{\LRH}^2,$
 averaged over 50 trials are depicted in Figure \ref{fig:batchGM}. Here,
 we replace the unknown marginal distribution $\rho_{X}$ by an empirical measure $\hat{\rho} = {1 \over 2000}\sum_{i=1}^{2000} \delta_{\hat{x}_i},$
 where each $\hat{x}_i$ is uniformly distributed in $[0,1].$
From Figure \ref{fig:minibatch_SGM} or \ref{fig:SGM}, we see that as the number of passes increases\footnote{Note that the terminology `running the algorithm with $p$ passes' means  `running the algorithm
with $\lceil mp/b \rceil$ iterations', where $b$ is the mini-batch size.},
 the bias decreases, while the sample error increases.
Furthermore, we see that in comparisons with the bias and the sample error, the computational error is negligible.
In all these experiments,  the minimal total error is achieved when the bias and the sample error are balanced. These empirical results show the effects of the three terms from the error decomposition, and complement the derived bound \eref{mainTotalErr}, as well as the regularization effect of the number of passes over the data.
Finally, we tested the simple SGM, mini-batch SGM, and batch GM, using similar step-sizes as those in the first simulation,
on the {\it BreastCancer} data-set\footnote{\url{https://archive.ics.uci.edu/ml/datasets/}}. The classification errors on the training set and the testing set of these three algorithms are
depicted in Figure \ref{fig:comparisions_bc}. We see that all of these algorithms perform similarly, which complement the bounds in  Corollaries \ref{cor:MPSGMB}, \ref{cor:MbSGMB} and \ref{cor:BSGM}.


\acks{
This material is based upon work supported by the Center for Brains, Minds and Machines (CBMM), funded by NSF STC award CCF-1231216. L. R. acknowledges the financial support of the Italian Ministry of Education, University and Research FIRB project RBFR12M3AC.}

\appendices

\section{Learning with Kernel Methods}\label{app:learning}
Let the input space $\Xi$ be a closed subset of Euclidean space $\mR^d$, the output space $Y \subseteq \mR$. Let $\mu$ be an unknown but fixed Borel probability measure on $\Xi \times Y$. Assume that $\mathbf \{(\xi_i, y_i)\}_{i=1}^m$ are i.i.d. from the distribution  $\mu$. A reproducing kernel $K$ is a symmetric function $K: \Xi
\times \Xi \to \mR$ such that $(K(u_i, u_j))_{i, j=1}^\ell$ is
positive semidefinite for any finite set of points
$\{u_i\}_{i=1}^\ell$ in $\Xi$. The kernel $K$ defines a reproducing
kernel Hilbert space (RKHS) $(\mathcal{H}_K, \|\cdot\|_K)$ as the
completion of the linear span of the set $\{K_{\xi}(\cdot):=K(\xi,\cdot):
\xi\in \Xi\}$ with respect to the inner product $\la K_{\xi},
K_u\ra_{K}:=K(\xi,u).$ For any $f \in \mathcal{H}_K$, the reproducing property holds: $f(\xi) = \la K_{\xi}, f\ra_K.$

\begin{Exa}
  [Sobolev Spaces]
  Let $X=[0,1]$ and the kernel
  $$
   K(x,x') =
   \begin{cases}
   (1-y)x, & x\leq y; \\
   (1-x)y, & x \geq y.
  \end{cases}
  $$
  Then the kernel induces a Sobolev Space $\HK = \{f : X \to \mR | f \mbox{ is absolutely continuous }, f(0) = f(1) =
0, f \in L^2(X)\}. $
\end{Exa}
In learning with kernel methods, one considers the following minimization problem
$$ \inf_{f\in \mathcal{H}_K} \int_{\Xi \times Y} (f(\xi) - y)^2 d\mu(\xi,y).$$
Since $f(\xi) = \la K_{\xi},f\ra $ by the reproducing property, the above can be rewritten as
$$ \inf_{f\in \mathcal{H}_K} \int_{\Xi \times Y} (\la f, K_{\xi} \ra - y)^2 d\mu(\xi,y).$$
Letting $X = \{K_{\xi}: \xi \in \Xi\}$ and defining another probability measure
$\rho(K_{\xi},y) = \mu(\xi,y)$, the above reduces to the learning setting in Section \ref{sec:learning}.

\section{Further Corollaries for SGM in the non-attainable case}\label{app:further}
In this section, we state the convergence results for the SGM with different parameter choices similar as those in Corollaries \ref{cor:MPSGMA}--\ref{cor:BSGM}, in the non-attainable case.
These results are direct consequences of Theorem \ref{thm:generalRateNonFix}.
\begin{corollary}
 Under Assumptions \ref{as:noiseExp}, \ref{as:regularity} and \ref{as:eigenvalues}, let $\zeta \leq 1/2$ , $\delta\in ]0,1[$, $b =1$, and $\eta_t \simeq m^{-{2\zeta \over (2\zeta+\gamma)\vee 1}}$ for all $t \in [m^2]$. With probability at least $1-\delta$, the following holds:\\
1) if $2\zeta+\gamma>1$, $m \geq m_{\delta}$ and $T^* = \lceil m^{{2\zeta+ 1\over 2\zeta+\gamma}} \rceil$, then we have \eref{minimaxBoundNonA};\\
2) if $2\zeta+\gamma \leq 1$, and for some $\epsilon \in ]0,1[$, $m \geq m_{\delta, \epsilon}$, and
 $T^* = \lceil m^{1 + 2\zeta - \epsilon}\rceil,$
 then we have \eref{minimaxBoundNonB}.
\end{corollary}

\begin{corollary}
 Under Assumptions \ref{as:noiseExp}, \ref{as:regularity} and \ref{as:eigenvalues}, let $\zeta \leq 1/2$ , $\delta\in ]0,1[$, $b \simeq m^{2\zeta \over (2\zeta+\gamma) \vee 1}$, and $\eta_t \simeq {1\over \log m}$ for all $t \in [m]$. With probability at least $1-\delta$, the following holds:\\
1) if $2\zeta+\gamma>1$, $m \geq m_{\delta}$ and $T^* = \lceil m^{{1\over 2\zeta+\gamma}} \rceil$, then we have \eref{minimaxBoundNonA};\\
2) if $2\zeta+\gamma \leq 1$, and for some $\epsilon \in ]0,1[$, $m \geq m_{\delta, \epsilon}$, and
 $T^* = \lceil m^{1 - \epsilon}\rceil,$
 then we have \eref{minimaxBoundNonB}.
\end{corollary}

\begin{corollary}
 Under Assumptions \ref{as:noiseExp}, \ref{as:regularity} and \ref{as:eigenvalues}, let $\zeta \leq 1/2$ , $\delta\in ]0,1[$,
  $b = m$ and $\eta_t \simeq {1\over \log m}$ for all $t \in [m]$. With probability at least $1-\delta$, the following holds:\\
1) if $2\zeta+\gamma>1$, $m \geq m_{\delta}$ and $T^* = \lceil m^{{1\over 2\zeta+\gamma} } \rceil$, then we have \eref{minimaxBoundNonA};\\
2) if $2\zeta+\gamma \leq 1$, and for some $\epsilon \in ]0,1[$, $m \geq m_{\delta, \epsilon}$, and
 $T^* = \lceil m^{1 - \epsilon}\rceil,$
 then we have \eref{minimaxBoundNonB}.
\end{corollary}

\section{Proofs for Lemmas}\label{sec:prov_estimates}

\begin{proof}[of Lemma \ref{lem:estimate1}]
Note that
  $$ \sum_{k=1}^{t} k^{-\theta}  \leq 1 + \sum_{k=2}^t \int_{k-1}^k u^{-\theta} d u = 1  + \int_{1}^t u^{-\theta} d u = {t^{1-\theta} - \theta\over 1-\theta} ,
 $$
 which leads to the first part of the desired result.
 Similarly,
  \bea
   \sum_{k=1}^t k^{-\theta} \geq \sum_{k=1}^t \int_{k}^{k+1}u^{-\theta} d u = \int_{1}^{t+1} u^{-\theta} d u = {(t+1)^{1-\theta} - 1\over 1-\theta},
  \eea
  and by mean value theorem, $(t+1)^{1-\theta} - 1 \geq (1-\theta)t (t+1)^{-\theta} \geq (1-\theta)t^{1-\theta}/2. $
  This proves the second part of the desired result. The proof is complete.
\end{proof}

\begin{proof}[of Lemma \ref{lem:estimate1a}]
  Note that
  \bea
  \sum_{k=1}^{t} k^{-\theta} = \sum_{k=1}^{t} k^{-1} k^{1-\theta} \leq t^{\max(1-\theta,0)} \sum_{k=1}^{t} k^{-1},
  \eea
  and
  \bea
  \sum_{k=1}^t k^{-1} \leq 1 + \sum_{k=2}^t \int_{k-1}^k u^{-1} du = 1 + \log t.
  \eea
\end{proof}

\begin{proof}[of Lemma \ref{lem:estimate2}]
  Note that
  \bea
  \sum_{k=1}^{t-1} {1 \over t-k} k^{-q} = \sum_{k=1}^{t-1} {k^{1-q} \over (t-k)k} \leq t^{\max(1-q,0)} \sum_{k=1}^{t-1} {1 \over (t-k)k},
  \eea
 and that by Lemma \ref{lem:estimate1a},
   \bea
  \sum_{k=1}^{t-1} {1 \over (t-k)k} = {1 \over t}\sum_{k=1}^{t-1} \left({1 \over t-k} + {1 \over k}\right) = {2 \over t}\sum_{k=1}^{t-1} {1 \over k} \leq {2\over t} (1+\log t).
  \eea
\end{proof}


\begin{proof}[of Proposition \ref{pro:initialErr}]
Since $\mu_{t+1}$ is given by \eref{Alg3}, introducing with \eref{frFH},
\be\label{Alg3D}
\mu_{t+1} = \mu_{t} - \eta_t (\TK \mu_{t} - \IK^* \FH).
\ee
Thus,
\be\label{Alg3C}
\IK \mu_{t+1} =  \IK \mu_{t} - \eta_t \IK(\TK \mu_{t} - \IK^* \FH) = \IK \mu_{t} - \eta_t \LK (\IK \mu_{t} - \FH) .
\ee
Subtracting both sides by  $\FH$,
\bea
\IK \mu_{t+1} - \FH = (I - \eta_t \LK) (\IK \mu_{t} - \FH) .
\eea
 Using this equality iteratively, with $\mu_1 = 0,$
  \bea
  \IK \mu_{t+1} - \FH =  - \Pi_{1}^t(\LK)  \FH.
  \eea
Taking the $\LR$-norm,  by Assumption \ref{as:regularity},
  \bea
  \|\IK \mu_{t+1} - \FH\|_{\rho} = \|\Pi_{1}^t(\LK)  \FH \|_{\rho} \leq \|\Pi_{1}^t(\LK) \LK^{\zeta} \| R.
  \eea
  By applying Lemma \ref{lemma:initialerror}, we get \eref{initialErrA}. Combining \eref{initialErrA} with Lemma \ref{lem:estimate1}, we get
  \eref{initialErrB}. The proof is complete.
\end{proof}

\begin{proof}[of Lemma \ref{lemma:htInfty}]
From \eref{Alg3D}, we have
\bea
\mu_{t+1} = (I - \eta_t \TK) \mu_{t} + \eta_t \IK^* \FH.
\eea
Applying this relationship iteratively, and using $\mu_1 = 0,$ we get
\bea
\mu_{t+1} = \sum_{k=1}^t \eta_k \Pi_{k+1}^t( \TK ) \IK^* \FH = \sum_{k=1}^t \eta_k  \IK^* \Pi_{k+1}^t( \LK ) \FH.
\eea
Therefore, using Assumption \ref{as:regularity} and spectral theory,
\bea
\|\mu_{t+1}\|_{\HK} \leq \left\|\sum_{k=1}^t \eta_k  \IK^* \Pi_{k+1}^t( \LK )\LK^{\zeta} \right\| R \leq R \max_{\sigma \in ]0,\kappa^2]} \sigma^{1/2+ \zeta} \sum_{k=1}^t \eta_k   \Pi_{k+1}^t( \sigma ) .
\eea
{\bf Case} $\zeta \geq 1/2. $ For any $\sigma \in ]0,\kappa^2],$
\bea
\sigma^{1/2+ \zeta} \sum_{k=1}^t \eta_k   \Pi_{k+1}^t( \sigma ) \leq \kappa^{2\zeta-1} \sigma \sum_{k=1}^t \eta_k   \Pi_{k+1}^t( \sigma ) \leq \kappa^{2\zeta-1},
\eea
where for the last inequality, we used
\be\label{eq:les1}
\sum_{k=1}^t \eta_k\sigma   \Pi_{k+1}^t( \sigma )  = \sum_{k=1}^t (1 - (1- \eta_k\sigma))   \Pi_{k+1}^t( \sigma ) = \sum_{k=1}^t \Pi_{k+1}^t( \sigma )  - \sum_{k=1}^t \Pi_{k}^t( \sigma ) = 1 - \Pi_{1}^t( \sigma ).
\ee
Thus,
\bea
\|\mu_{t+1}\|_{\HK} \leq  R \kappa^{2\zeta -1} .
\eea
{\bf Case} $\zeta < 1/2.$ If $\sum_{k=1}^t \eta_k \leq \kappa^{-2},$ then for any $\sigma \leq \kappa^2,$
\bea
\sigma^{1/2+ \zeta} \sum_{k=1}^t \eta_k   \Pi_{k+1}^t( \sigma ) \leq \sigma^{1/2+ \zeta} \sum_{k=1}^t \eta_k   \leq  \kappa^{2\zeta-1}.
\eea
If $\sum_{k=1}^t \eta_k > \kappa^{-2},$ then for any $\sigma \leq (\sum_{k=1}^t \eta_k)^{-1},$
\bea
\sigma^{1/2+ \zeta} \sum_{k=1}^t \eta_k   \Pi_{k+1}^t( \sigma ) \leq \sigma^{1/2+ \zeta} \sum_{k=1}^t \eta_k
\leq  \left(\sum_{k=1}^t \eta_k\right)^{1/2 - \zeta},
\eea
while for $ \kappa^2 \geq \sigma \geq (\sum_{k=1}^t \eta_k)^{-1},$ by \eref{eq:les1},
\bea
\sigma^{1/2+ \zeta} \sum_{k=1}^t \eta_k   \Pi_{k+1}^t( \sigma ) = \sigma^{\zeta-1/2}  \sum_{k=1}^t \eta_k \sigma \Pi_{k+1}^t( \sigma )
\leq  \sigma^{\zeta-1/2} \leq
\left(\sum_{k=1}^t \eta_k\right)^{1/2 - \zeta}.
\eea
From the above analysis, we get that
\bea
\max_{\sigma \in ]0,\kappa^2]} \sigma^{1/2+ \zeta} \sum_{k=1}^t \eta_k   \Pi_{k+1}^t( \sigma ) \leq \kappa^{2\zeta-1} \vee \left(\sum_{k=1}^t \eta_k\right)^{1/2 - \zeta},
\eea
and thus
\bea
\|\mu_{t+1}\|_{\HK} \leq  R \left\{\kappa^{2\zeta-1} \vee \left(\sum_{k=1}^t \eta_k\right)^{1/2 - \zeta}\right\}.
\eea
The proof is complete.
\end{proof}

\begin{proof}[of Lemma \ref{lem:lambdaNk} (1)]
{\it Bounding $\left\| (\TK + \lambda)^{-{1\over 2}}\left(\IK^* f_{\rho} - \SX^*{\bf y}\right) \right\|_{\HK}$:}\\
 For all $i \in [m],$ let $w_i =  y_i (\TK+ \lambda I)^{-{1\over 2}} x_i.$
Obviously, from the definitions of $f_{\rho}$ (see \eref{regressionfunc}) and $\IK$,
\bea
\mE[w_1] = \mE_{x_1} [ f_{\rho}(x_1) (\TK+ \lambda I)^{-{1\over 2}} x_1] = (\TK+ \lambda I)^{-{1\over 2}}  \IK^* f_{\rho}.
\eea
Thus,
\bea
(\TK + \lambda)^{-{1\over 2}}\left(\IK^* f_{\rho} - \SX^*{\bf y}\right) =  {1 \over m} \sum_{i=1}^m (\mE[w_i] - w_i).
\eea
We next estimate the constants $B$ and $\sigma^2(w_1)$ in \eref{bernsteinCondition}.
Note that for any $l \geq 2,$
\bea
\mE[\|w_1 - \mE[w_1]\|_{\HK}^l] \leq  \mE[(\|w_1\|_{\HK} + \mE[\|w_1\|_{\HK}])^l].
\eea
By using H\"{o}lder's inequality twice,
\bea
\mE[\|w_1 - \mE[w_1]\|_{\HK}^l] \leq 2^{l-1} \mE[\|w_1\|_{\HK}^l + (\mE[\|w_1\|_{\HK}])^l] \leq 2^{l-1} \mE[\|w_1\|_{\HK}^l + \mE[\|w_1\|_{\HK}^l]].
\eea
The right-hand side is exactly $ 2^{l}\mE[\|w_1\|_{\HK}^l]$. Therefore, by recalling the definition of $w_1$ and expanding the integration,
\be\label{estimateMomentInterm}
\mE[\|w_1 - \mE[w_1]\|_{\HK}^l]
 \leq
2^{l} \int_{X} \|(\TK + \lambda I)^{-{1\over 2}} x\|_{\HK}^l \int_{Y} y^l d\rho(y|x)   d\rho_{X}(x).
\ee
Introducing \eref{eq:ymoment} and \eref{eq:xmoment} into the above inequality, we have
\bea
\mE[\|w_1 - \mE[w_1]\|_{\HK}^l]
 \leq  l! (2\sqrt{M})^{l} \sqrt{v} \left({\kappa \over \sqrt{\lambda}}\right)^{l-2}  c_{\gamma} \lambda^{-\gamma}
= {1\over 2}  l! \left({2  \kappa \sqrt{M} \over \sqrt{\lambda}}\right)^{l-2}  8  M\sqrt{v} c_{\gamma} \lambda^{-\gamma}.
\eea
Applying Bernstein inequality with $B= {2  \kappa \sqrt{M} \over \sqrt{\lambda} }$ and $\sigma = \sqrt{ 8 M\sqrt{v} c_{\gamma} \lambda^{-\gamma}},$ we get that with probability at least $1 - {\delta_1 \over 2}$,  there holds
\be\label{concentration_frho}
\left\| (\TK + \lambda)^{-{1\over 2}}\left(\IK^* f_{\rho} - \SX^*{\bf y}\right) \right\|_{\HK} = \left\| {1 \over m} \sum_{i=1}^m (\mE[w_i] - w_i) \right\|_{\HK} \leq 4 \sqrt{M} \left( { \kappa   \over m \sqrt{\lambda}} + {\sqrt{ 2\sqrt{v} c_{\gamma} } \over \sqrt{m \lambda^{\gamma}}} \right) \log{4 \over \delta_1}.
\ee
{\it Bounding $\|(\TK + \lambda)^{-{1\over 2}} (\TK-\TX) \|$:}\\
Let $\xi_i = (\TK + \lambda)^{-{1\over 2}} x_i \otimes x_i, $ for all $i \in [m]$. It is easy to see that $\mE[\xi_i] = (\TK + \lambda)^{-{1\over 2}} \TK,$ and that
$(\TK + \lambda)^{-{1\over 2}} (\TK-\TX)  = {1 \over m}\sum_{i=1}^m (\mE[\xi_i] - \xi_i).$
Denote the Hilbert-Schmidt norm of a bounded operator from $\HK$ to $\HK$ by $\|\cdot\|_{HS}.$
Note that
\bea
\|\xi_1\|_{HS}^2 = \|x_1\|_{\HK}^2 \mbox{Trace}((\TK + \lambda)^{-1/2}x_1 \otimes x_1 (\TK + \lambda)^{-1/2}) =  \|x_1\|_{\HK}^2 \mbox{Trace}((\TK + \lambda)^{-1}x_1 \otimes x_1 ).
\eea
By Assumption \eref{boundedKernel},
\bea
\|\xi_1\|_{HS} \leq \sqrt{ \kappa^2 \mbox{Trace}((\TK + \lambda)^{-1}x_1 \otimes x_1 )} \leq \sqrt{ \kappa^2 \mbox{Trace}(x_1 \otimes x_1 ) /\lambda} \leq {\kappa^2 / \sqrt{\lambda}},
\eea
and furthermore, by Assumption \ref{as:eigenvalues},
\bea
\mE[\|\xi_1\|_{HS}^2 ] \leq  \kappa^2 \mE \mbox{Trace}((\TK + \lambda)^{-1}x_1 \otimes x_1 ) = \kappa^2 \mbox{Trace}((\TK + \lambda)^{-1} \TK) \leq \kappa^2 c_{\gamma} \lambda^{-\gamma}.
\eea
According to Lemma \ref{lem:Bernstein}, we get that with probability at least $1 - {\delta_1 \over 2},$ there holds
\be\label{concentration_operator}
\|(\TK + \lambda)^{-{1\over 2}} (\TK-\TX) \|_{HS} \leq 2\kappa\left( {2\kappa \over m \sqrt{\lambda}} + {\sqrt{c_{\gamma}} \over \sqrt{m \lambda^{\gamma}} } \right) \log {4 \over \delta_1}.
\ee
Finally, using the triangle inequality, we have,
\bea
\|(\TK + \lambda)^{-{1 \over 2}} N_k \|_{\HK} \leq \|(\TK + \lambda)^{-{1\over 2}} (\TK-\TX) \| \|\mu_k\|_{\HK} + \left\| (\TK + \lambda)^{-{1\over 2}}\left(\IK^* f_{\rho} - \SX^*{\bf y}\right) \right\|_{\HK}.
\eea
Applying \eref{htInfty} to the above, introducing with \eref{concentration_frho} and \eref{concentration_operator},
and then noting that $\kappa \geq 1$ and $v \geq 1,$
 one can prove the first part of the lemma.\\
\end{proof}

\begin{proof}[of Lemma \ref{lemma:biasHnorm}]
   Obviously, $\FH = \IK \omega^{\dag}$ and thus
$\TK \omega^{\dag} = \IK^{*} \FH$. Combining with Assumption \ref{as:regularity},
$\TK \omega^{\dag} = \IK^{*} \LK^{\zeta} \LK^{-\zeta} \FH = \TK^{\zeta} \IK^{*} \LK^{-\zeta} \FH $,
and $\omega^{\dag} = \TK^{\dag} \TK^{\zeta} \IK^{*} \LK^{-\zeta} \FH.$ Subtracting $\omega^{\dag}$ from  both sides of \eref{Alg3D}, and using $\IK^{*} \FH = \TK \omega^{\dag}$,
we know that
\bea
\mu_{t+1} - \omega^{\dag} = (I - \eta_t \TK)(\mu_t - \omega^{\dag}).
\eea
Applying this relationship iteratively, with $\mu_1=0$,
\bea
\mu_{t+1} - \omega^{\dag} = -\Pi_1^t(\TK)\omega^{\dag}= -\Pi_1^t(\TK) \TK^{\dag} \TK^{\zeta} \IK^{*} \LK^{-\zeta} \FH.
\eea
Therefore,
\bea
\|\mu_{t+1} - \omega^{\dag}\|_{\HK} \leq \|\Pi_1^t(\TK) \TK^{\dag} \TK^{\zeta} \IK^{*}\| R \leq \|\Pi_1^t(\TK) \TK^{\zeta-1/2}\| R.
\eea
Applying Lemma \ref{lemma:initialerror}, one can get the desired result.
\end{proof}

{\section{List of Some Notations} \label{sec:notations}}

\begin{center}
	\begin{tabular}{ c | l  }
		\hline			
		Notation &  Meaning  \\ \hline
		$H$ & the hypothesis space\\
		
		 $X, Y, Z$& the input space, the output space and the sample space ($Z= X \times Y$)  \\
		 
		 $\rho$ & the fixed  probability measure on $Z$ \\
		 
		  $\rho_X$ &the induced marginal measure of $\rho$ on $X$ \\
		 
		 $\rho(\cdot | x)$ &  the conditional probability measure on $Y$ w.r.t. $x\in X$ and $\rho$ \\

		  $\bf z$ & the sample $\{z_i=(x_i, y_i)\}_{i=1}^m$ of size $m\in\mN$, where each $z_i$ is i.i.d. according to $\rho$.\\

		 $m$ & the sample size of the sample $\bf z$\\

		 $\mcE$ & the expected risk defined by \eref{expectedRisk}\\
		 
		 $\mcE_{\bf z}$ & the empirical risk w.r.t the sample $\bf z$ defined by \eref{eq:emprisk} \\
		 
		 $\kappa^2$ & the constant from the bounded assumption \eref{boundedKernel} on the hypothesis space $H$ \\
		 
		 $\{\omega_t\}_t$ & the sequence generated by the SGM \\
		 
		 $\theta$ & the decaying rate on step-sizes \\
		 
		 $b$ &  the minibatch size of the SGM \\
		 
		 $T$ & the maximal number of iterations for the SGM \\
		 
		 $j_i $ ($j_t$ etc.)& the random index from the uniform distribution on $[m]$ for the SGM \\
		 
		 $\J_t$ & the set of random indices at $t$-th iteration of the SGM \\
		 
		  $\J$ &
		 the set of all random indices for the SGM after $T$ iterations \\

$\mE_{\J}$ & the expectation with respect to the random variables $\J$ (conditional on $\bf z$)\\
		 
		 $\{\eta_t\}_t$ & the sequence of step-sizes \\
		 
		 $M, v$ & the positive constants from the moment (bounded) assumption on the output \\
		 
		 $\LR$ & the Hilbert space of square integral functions from $\HK$ to $\mR$ with respect to $\rho_X$ \\
		 
		 $f_{\rho}$ & the regression function defined  \eref{regressionfunc} \\
		 
		 $\HR$ & $\{f: X \to  \mR| \exists \omega \in \HK \mbox{ with } f(x) = \la \omega, x \ra_{\HK}, \rho_X \mbox{-almost surely}\}$ \\
		 
		 $\zeta,R$& the parameters related to the `regularity' of $\FH$ (see Assumption \ref{as:regularity}) \\
		 
		 $\omega^{\dagger}$ & the solution of Problem \eref{expectedRisk} with the minimal norm in the attainable case\\
		 
		 $\gamma, c_{\gamma}$ & the parameters related to the effective dimension (see Assumption \ref{as:eigenvalues}) \\
		 
		 $\{\sigma_i\}_i$ & the sequence of eigenvalues of $\LK$ \\
		
		$\{\nu_t\}_t$ & the sequence generated by  the batch GM \eref{Alg2B}\\
		
$\{\mu_{k}\}_k$ & the sequence defined by the  population iteration \eref{Alg3B} \\

$\IK$ & the linear map from $\HK \to \LR$ defined by $\IK \omega = \la \omega, \cdot \ra$ \\

$\IK^*$ & the adjoint operator of $\IK$, $\IK^* f = \int_{X} f(x) x d\rho_{X}(x) $ \\

$\LK$ & the operator from $\LR $ to $\LR$, $\LK(f) = \IK \IK^*f =\int_{X}  \la x, \cdot\ra_{\HK} f(x) \rho_{X}(x)$\\

$\TK$ & the covariance operator from $\HK$ to $\HK$, $\TK = \IK^* \IK = \int_{X} \la \cdot, x \ra x d\rho_{X}(x)$ \\

$\SX$ & the sampling operator from $\HK$ to $\mR^{m}$, $(\SX \omega)_i = \la \omega, x_i \ra_{\HK}, i \in [m]$ \\

$\SX^*$ & the adjoint operator of  $\SX$, $\SX^* \mathbf{y} = {1 \over m} \sum_{i=1}^m y_i x_i$\\

$\TX$ & the empirical covariance operator, $\TX = \SX^* \SX = {1 \over m} \sum_{i=1}^m \la \cdot, x_i\ra x_i$\\

$\Pi_{t+1}^T(L)$ & $= \Pi_{k=t+1}(I - \eta_k L)$ when $t\in [T-1]$ and $\Pi_{T+1}^{T} = I$ \\

$\sum_{i=t+1}^t\eta_i $ & $ = 0$  \\

$\lambda$ & a `regularization' parameter, $\lambda>0$\\

$\TKL$, &  $\TKL = \TK + \lambda$\\

$\TXL$, &  $\TXL = \TX  +\lambda$\\

$\{N_k\}_k$ & the sequence defined by \eref{Nt}. \\

$M_{k,i}$ & defined by \eref{Mk} 		 
		\\ \hline  
	\end{tabular}
\end{center}

\bibliography{sigd}

\begin{thebibliography}{36}
\providecommand{\natexlab}[1]{#1}
\providecommand{\url}[1]{\texttt{#1}}
\expandafter\ifx\csname urlstyle\endcsname\relax
  \providecommand{\doi}[1]{doi: #1}\else
  \providecommand{\doi}{doi: \begingroup \urlstyle{rm}\Url}\fi

\bibitem[Bach and Moulines(2013)]{bach2013non}
Francis Bach and Eric Moulines.
\newblock Non-strongly-convex smooth stochastic approximation with convergence
  rate {O}($1/ n$).
\newblock In \emph{Advances in Neural Information Processing Systems}, pages
  773--781, 2013.

\bibitem[Bauer et~al.(2007)Bauer, Pereverzev, and
  Rosasco]{bauer2007regularization}
Frank Bauer, Sergei Pereverzev, and Lorenzo Rosasco.
\newblock On regularization algorithms in learning theory.
\newblock \emph{Journal of Complexity}, 23\penalty0 (1):\penalty0 52--72, 2007.

\bibitem[Blanchard and M{\"u}cke(2016)]{blanchard2016optimal}
Gilles Blanchard and Nicole M{\"u}cke.
\newblock Optimal rates for regularization of statistical inverse learning
  problems.
\newblock \emph{arXiv preprint arXiv:1604.04054}, 2016.

\bibitem[Bousquet and Bottou(2008)]{bousquet2008tradeoffs}
Olivier Bousquet and L{\'e}on Bottou.
\newblock The tradeoffs of large scale learning.
\newblock In \emph{Advances in Neural Information Processing Systems}, pages
  161--168, 2008.

\bibitem[Boyd and Mutapcic(2007)]{boyd2007stochastic}
Stephen Boyd and Almir Mutapcic.
\newblock Stochastic subgradient methods.
\newblock Notes for EE364b, Standford University, Winter 2007.

\bibitem[Caponnetto and De~Vito(2007)]{caponnetto2007optimal}
Andrea Caponnetto and Ernesto De~Vito.
\newblock Optimal rates for the regularized least-squares algorithm.
\newblock \emph{Foundations of Computational Mathematics}, 7\penalty0
  (3):\penalty0 331--368, 2007.

\bibitem[Caponnetto and Yao(2010)]{caponnetto2010cross}
Andrea Caponnetto and Yuan Yao.
\newblock Cross-validation based adaptation for regularization operators in
  learning theory.
\newblock \emph{Analysis and Applications}, 8\penalty0 (02):\penalty0 161--183,
  2010.

\bibitem[Cesa-Bianchi et~al.(2004)Cesa-Bianchi, Conconi, and
  Gentile]{cesa-bianchi2004}
Nicolo Cesa-Bianchi, Alex Conconi, and Claudio Gentile.
\newblock On the generalization ability of on-line learning algorithms.
\newblock \emph{IEEE Transactions on Information Theory}, 50\penalty0
  (9):\penalty0 2050--2057, 2004.
\newblock ISSN 0018-9448.

\bibitem[Cotter et~al.(2011)Cotter, Shamir, Srebro, and
  Sridharan]{cotter2011better}
Andrew Cotter, Ohad Shamir, Nati Srebro, and Karthik Sridharan.
\newblock Better mini-batch algorithms via accelerated gradient methods.
\newblock In \emph{Advances in Neural Information Processing Systems}, pages
  1647--1655, 2011.

\bibitem[Cucker and Zhou(2007)]{cucker2007learning}
Felipe Cucker and Ding-Xuan Zhou.
\newblock \emph{Learning Theory: an Approximation Theory Viewpoint}, volume~24.
\newblock Cambridge University Press, 2007.

\bibitem[Dekel et~al.(2012)Dekel, Gilad-Bachrach, Shamir, and
  Xiao]{dekel2012optimal}
Ofer Dekel, Ran Gilad-Bachrach, Ohad Shamir, and Lin Xiao.
\newblock Optimal distributed online prediction using mini-batches.
\newblock \emph{Journal of Machine Learning Research}, 13\penalty0
  (1):\penalty0 165--202, 2012.

\bibitem[Dicker et~al.(2017)Dicker, Foster, Hsu, et~al.]{dicker2016kernel}
Lee~H Dicker, Dean~P Foster, Daniel Hsu, et~al.
\newblock Kernel ridge vs. principal component regression: Minimax bounds and
  the qualification of regularization operators.
\newblock \emph{Electronic Journal of Statistics}, 11\penalty0 (1):\penalty0
  1022--1047, 2017.

\bibitem[Dieuleveut and Bach(2016)]{dieuleveut2014non}
Aymeric Dieuleveut and Francis Bach.
\newblock Non-parametric stochastic approximation with large step sizes.
\newblock \emph{Annals of Statistics}, 44\penalty0 (4):\penalty0 1363--1399,
  2016.

\bibitem[Hardt et~al.(2016)Hardt, Recht, and Singer]{hardt2015train}
Moritz Hardt, Benjamin Recht, and Yoram Singer.
\newblock Train faster, generalize better: Stability of stochastic gradient
  descent.
\newblock In \emph{International Conference on Machine Learning}, 2016.

\bibitem[Lin and Rosasco(2016)]{lin2016optimal}
Junhong Lin and Lorenzo Rosasco.
\newblock Optimal learning for multi-pass stochastic gradient methods.
\newblock In \emph{Advances In Neural Information Processing Systems 29}, pages
  4556--4564. 2016.

\bibitem[Lin et~al.(2016{\natexlab{a}})Lin, Camoriano, and
  Rosasco]{lin2016generalization}
Junhong Lin, Raffaello Camoriano, and Lorenzo Rosasco.
\newblock Generalization properties and implicit regularization of multiple
  passes {SGM}.
\newblock In \emph{International Conference on Machine Learning},
  2016{\natexlab{a}}.

\bibitem[Lin et~al.(2016{\natexlab{b}})Lin, Rosasco, and
  Zhou]{lin2015iterative}
Junhong Lin, Lorenzo Rosasco, and Ding-Xuan Zhou.
\newblock Iterative regularization for learning with convex loss functions.
\newblock \emph{Journal of Machine Learning Research}, 17\penalty0
  (77):\penalty0 1--38, 2016{\natexlab{b}}.

\bibitem[Minsker(2011)]{minsker2011some}
Stanislav Minsker.
\newblock On some extensions of bernstein's inequality for self-adjoint
  operators.
\newblock \emph{arXiv preprint arXiv:1112.5448}, 2011.

\bibitem[Nemirovski et~al.(2009)Nemirovski, Juditsky, Lan, and
  Shapiro]{nemirovski2009robust}
Arkadi Nemirovski, Anatoli Juditsky, Guanghui Lan, and Alexander Shapiro.
\newblock Robust stochastic approximation approach to stochastic programming.
\newblock \emph{SIAM Journal on Optimization}, 19\penalty0 (4):\penalty0
  1574--1609, 2009.

\bibitem[Ng(2016)]{ng2016machine}
Andrew Ng.
\newblock Machine learning.
\newblock Coursera, Standford University, 2016.

\bibitem[Orabona(2014)]{orab14}
Francesco Orabona.
\newblock Simultaneous model selection and optimization through parameter-free
  stochastic learning.
\newblock In \emph{Advances in Neural Information Processing Systems}, pages
  1116--1124, 2014.

\bibitem[Pinelis and Sakhanenko(1986)]{pinelis1986remarks}
IF~Pinelis and AI~Sakhanenko.
\newblock Remarks on inequalities for large deviation probabilities.
\newblock \emph{Theory of Probability \& Its Applications}, 30\penalty0
  (1):\penalty0 143--148, 1986.

\bibitem[Poljak(1987)]{poljak1987introduction}
Boris~T Poljak.
\newblock \emph{Introduction to Optimization}.
\newblock Optimization Software, 1987.

\bibitem[Rosasco and Villa(2015)]{rosasco2015learning}
Lorenzo Rosasco and Silvia Villa.
\newblock Learning with incremental iterative regularization.
\newblock In \emph{Advances in Neural Information Processing Systems}, pages
  1621--1629, 2015.

\bibitem[Rudi et~al.(2015)Rudi, Camoriano, and Rosasco]{rudi2015less}
Alessandro Rudi, Raffaello Camoriano, and Lorenzo Rosasco.
\newblock Less is more: Nystr{\"o}m computational regularization.
\newblock In \emph{Advances in Neural Information Processing Systems}, pages
  1648--1656, 2015.

\bibitem[Shalev-Shwartz et~al.(2011)Shalev-Shwartz, Singer, Srebro, and
  Cotter]{shalev2011pegasos}
Shai Shalev-Shwartz, Yoram Singer, Nathan Srebro, and Andrew Cotter.
\newblock Pegasos: Primal estimated sub-gradient solver for svm.
\newblock \emph{Mathematical Programming}, 127\penalty0 (1):\penalty0 3--30,
  2011.

\bibitem[Shamir and Zhang(2013)]{shamir2013stochastic}
Ohad Shamir and Tong Zhang.
\newblock Stochastic gradient descent for non-smooth optimization: Convergence
  results and optimal averaging schemes.
\newblock In \emph{International Conference on Machine Learning}, pages 71--79,
  2013.

\bibitem[Smale and Zhou(2007)]{smale2007learning}
Steve Smale and Ding-Xuan Zhou.
\newblock Learning theory estimates via integral operators and their
  approximations.
\newblock \emph{Constructive Approximation}, 26\penalty0 (2):\penalty0
  153--172, 2007.

\bibitem[Sra et~al.(2012)Sra, Nowozin, and Wright]{sra2012optimization}
Suvrit Sra, Sebastian Nowozin, and Stephen~J Wright.
\newblock \emph{Optimization for Machine Learning}.
\newblock {MIT} Press, 2012.

\bibitem[Steinwart and Christmann(2008)]{steinwart2008support}
Ingo Steinwart and Andreas Christmann.
\newblock \emph{Support Vector Machines}.
\newblock Springer Science Business Media, 2008.

\bibitem[Steinwart et~al.(2009)Steinwart, Hush, and
  Scovel]{steinwart2009optimal}
Ingo Steinwart, Don Hush, and Clint Scovel.
\newblock Optimal rates for regularized least squares regression.
\newblock \emph{Conference of Learning Theory}, 2009.

\bibitem[Tarres and Yao(2014)]{tarres2014online}
Pierre Tarres and Yuan Yao.
\newblock Online learning as stochastic approximation of regularization paths:
  Optimality and almost-sure convergence.
\newblock \emph{IEEE Transactions on Information Theory}, 60\penalty0
  (9):\penalty0 5716--5735, 2014.

\bibitem[Tropp(2012)]{tropp2012user}
Joel~A Tropp.
\newblock User-friendly tools for random matrices: An introduction.
\newblock Technical report, DTIC Document, 2012.

\bibitem[Yao et~al.(2007)Yao, Rosasco, and Caponnetto]{yao2007early}
Yuan Yao, Lorenzo Rosasco, and Andrea Caponnetto.
\newblock On early stopping in gradient descent learning.
\newblock \emph{Constructive Approximation}, 26\penalty0 (2):\penalty0
  289--315, 2007.

\bibitem[Ying and Pontil(2008)]{ying2008online}
Yiming Ying and Massimiliano Pontil.
\newblock Online gradient descent learning algorithms.
\newblock \emph{Foundations of Computational Mathematics}, 8\penalty0
  (5):\penalty0 561--596, 2008.

\bibitem[Zhang(2005)]{zhang2005learning}
Tong Zhang.
\newblock Learning bounds for kernel regression using effective data
  dimensionality.
\newblock \emph{Neural Computation}, 17\penalty0 (9):\penalty0 2077--2098,
  2005.

\end{thebibliography}
\bibliographystyle{abbrv}

\end{document}